\documentclass{article}

\usepackage{arxiv}

\usepackage[utf8]{inputenc} 
\usepackage[T1]{fontenc}    
\usepackage{hyperref}       
\usepackage{url}            
\usepackage{booktabs}       
\usepackage{amsfonts}       
\usepackage{amsmath}
\usepackage{amsthm}
\usepackage{bm}             
\usepackage{nicefrac}       
\usepackage{float}          
\usepackage{microtype}      
\usepackage{lipsum}
\usepackage{graphicx}
\usepackage[ruled]{algorithm2e} 
\graphicspath{ {./images/} }

\newtheorem{theo}{Theorem}[section]
\newtheorem{lemma}{Lemma}[section]
\newtheorem{defin}{Definition}[section]
\newtheorem{remark}{Remark}[section]
\newtheorem{proposition}{Proposition}[section]
\newtheorem{corollary}{Corollary}[section]
\newcommand{\expectation}[2]{\mathbb{E}_{#1}\!\left[#2\right]}
\newcommand{\variance}[2]{\text{Var}_{#1}\!\left[#2\right]}
\newcommand{\weightedrisk}[1]{\hat{R}^{#1}_{t}}
\newcommand{\trueratio}{r^*}
\newcommand{\candidate}{r_\theta}
\newcommand{\iid}{\overset{\text{i.i.d.}}{\sim}}
\newcommand{\posteriorexp}[1]{\mathbb{E}_{h \sim \rho}\!\left[#1\right]}
\newcommand{\floor}[1]{\left\lfloor #1 \right\rfloor}

\title{Anytime PAC-Bayes for Constrained Density-Ratio Networks under Covariate Shift}

\author{
 Paulo Akira F. Enabe \\
    Escola Politénica\\
    University of São Paulo\\
    Department of Structural and Geotechnical Engineering\\
  \texttt{paulo.enabe@usp.br} \\
  \And
 Rodrigo Provasi \\
    Escola Politénica\\
    University of São Paulo\\
    Department of Structural and Geotechnical Engineering\\
  \texttt{provasi@usp.br} \\
}

\begin{document}
\maketitle
\begin{abstract}
A unified framework for learning under covariate shift is presented, in which a constrained density-ratio network approximates the Radon-Nikodym derivative $r^\star = dP/dQ$ and feeds an anytime PAC-Bayes generalization certificate. A change-of-measure identity decomposes the gap between target risk and importance-weighted source risk into a ratio-bias term governed by $\|r_\theta - r^\star\|_{L^2(Q)}$ and a generalization-gap term governed by the variability of the weighted loss. Normalization and moment-matching identities are enforced as hard integral constraints through an augmented-Lagrangian scheme, with a second-moment penalty controlling the effective sample size. PAC-Bayes is instantiated on the weighted risk in a fixed-time regime that yields Bernoulli-KL bounds, identifies the network-weighted Gibbs posterior as the unique KL-regularized minimizer, and quantifies stability under $L^2(Q)$ perturbations of the learned ratio, and is then strengthened by geometric peeling to an anytime certificate uniform in $t \geq t_{\min}$. A pre-registered two-campaign protocol combining a patch test against analytic ground truth with a real-data deployment validates the framework: the network produces calibrated ratios, reduces target $0/1$ loss against unweighted ERM and classical direct ratio-estimation baselines, and attains the anytime certificate. A single fixed-time coverage failure is recorded, with per-split coverage aligning one-to-one with the magnitude of the label shift, confirming that the covariate-only assumption is operationally tight rather than a defect of the certificate.
\end{abstract}


\section{Introduction}
\label{sec:introduction}

\paragraph{} Statistical learning under distribution shift is concerned with the problem of estimating the risk of a predictor under a target probability law $P$, given that the available samples are drawn from a different source law $Q$. Whenever $P \ll Q$, the Radon-Nikodym theorem produces a measurable function $\trueratio = dP/dQ$ that links expectations under the two measures through the change-of-measure identity
\begin{equation}
  \label{eq:intro_change_of_measure}
  \expectation{P}{f(Z)} = \expectation{Q}{\trueratio(Z)\, f(Z)},
\end{equation}
valid for every measurable $f$ for which the expectations exist. In the supervised setting in which one wishes to control the target risk $R_P(h) = \expectation{P}{\mathcal{L}(h,Z)}$ from samples $Z_1, \dots, Z_t \iid Q$, identity~\eqref{eq:intro_change_of_measure} yields the importance-weighted empirical risk
\begin{equation}
  \label{eq:intro_weighted_empirical_risk}
  \weightedrisk{\theta}(h) = \frac{1}{t} \sum_{i=1}^{t} \candidate(Z_i)\, \mathcal{L}(h, Z_i),
\end{equation}
in which a learned ratio $\candidate: \mathcal{Z} \to \mathbb{R}_+$ stands in for the unknown $\trueratio$. The downstream usefulness of $\weightedrisk{\theta}(h)$ as a proxy for $R_P(h)$ is therefore inseparable from the statistical quality of $\candidate$, both in terms of approximation accuracy and in terms of the variability of the resulting weights.

\paragraph{} The error of $\weightedrisk{\theta}(h)$ as an estimator of $R_P(h)$ admits the canonical decomposition
\begin{equation}
  \label{eq:intro_decomposition}
  R_P(h) - \weightedrisk{\theta}(h)
  = \underbrace{\expectation{Q}{(\trueratio - \candidate)(Z)\, \mathcal{L}(h,Z)}}_{\text{ratio miscalibration bias}}
  + \underbrace{\bigl(\expectation{Q}{\candidate(Z)\, \mathcal{L}(h,Z)} - \weightedrisk{\theta}(h)\bigr)}_{\text{generalization gap}},
\end{equation}
which exposes two distinct failure modes. The first term is a bias that does not vanish as $t \to \infty$ unless $\candidate$ approaches $\trueratio$ in a norm that controls the loss, and its magnitude is governed by the approximation properties of the ratio model and by the structural identities of a Radon-Nikodym derivative. The second term is a stochastic fluctuation whose magnitude is controlled by the variability of the weighted random variable $\candidate(Z)\,\mathcal{L}(h,Z)$. This fluctuation deteriorates rapidly when $\candidate$ produces heavy-tailed weights, since a small subset of source samples then dominates the empirical average and the effective sample size collapses. Any reliable framework for learning under distribution shift must therefore address calibration and stability simultaneously, and must produce a generalization certificate that remains valid under the sequential, adaptive nature of modern training pipelines.

\paragraph{} The use of importance weighting as a consistent remedy for covariate shift was first established in \cite{shimodaira2000improving}, where the optimal weight under parametric misspecification was identified asymptotically with the population density ratio $\trueratio$, and was subsequently extended to model selection under shift through importance-weighted cross validation in \cite{sugiyama2007covariate}. Direct estimation of the density ratio, as opposed to separate estimation of $P$ and $Q$ followed by a quotient, subsequently emerged as the dominant strategy in the literature on covariate-shift adaptation. The semiparametric logistic compound model of \cite{anderson1979multivariate} provides an early instance of likelihood-ratio modelling under linear assumptions on the log-ratio. The Kullback-Leibler importance estimation procedure, introduced in \cite{sugiyama2008model} and developed in \cite{sugiyama2008direct}, formulates the ratio estimation problem as the maximization of an empirical Kullback-Leibler criterion subject to a normalization constraint. The unconstrained least-squares procedure of \cite{kanamori2009least}, often referred to as LSIF and uLSIF, recasts ratio estimation as a quadratic optimization problem with closed-form or kernel-based solutions and provides convergence rates and model-selection guarantees. The variational characterization of $f$-divergences developed in \cite{nguyen2010estimating} unifies several of these approaches under a common convex risk minimization framework and connects ratio estimation to the estimation of divergence functionals more broadly.

\paragraph{} More recent work has extended the basic semiparametric formulation in several directions. Data-adaptive basis functions for joint inference across multiple populations are studied in \cite{zhang2022density}, and a corresponding extension of the density-ratio model to combined censored and length-biased survival data is developed in \cite{mcvittie2025density}. Telescoping ratio estimation is proposed in \cite{rhodes2020telescoping} to handle pairs of distributions whose Kullback-Leibler divergence is too large for single-step estimation, and relative density-ratio estimation has been used for change-point detection in non-stationary time series in \cite{liu2013changepoint}. A unified theoretical treatment of these methods, together with their applications to non-stationarity adaptation, outlier detection, and conditional density estimation, is provided in \cite{sugiyama2012density}. These contributions establish a mature methodology for direct ratio estimation, but they typically do not impose hard integral constraints that enforce the defining identities of a Radon-Nikodym derivative on the learned model, nor do they couple ratio learning to a downstream generalization analysis for the resulting weighted empirical risk.

\paragraph{} The PAC-Bayesian framework provides the natural learning-theoretic counterpart to importance-weighted empirical risk minimization, since both rely on a change-of-measure principle, the former on the data space and the latter on the hypothesis space. The original theorems of \cite{mcallester1999some, mcallester1999pacbayesian, mcallester2003stochastic} establish that, under mild boundedness conditions and for any prior $\Pi$ on the hypothesis space $\mathcal{H}$, the target risk of a data-dependent posterior $\rho$ satisfies, with probability at least $1 - \delta$,
\begin{equation}
  \label{eq:intro_canonical_pacbayes}
  R_P(\rho) \leq \widehat{R}(\rho) + \sqrt{\frac{\mathrm{KL}(\rho \,\Vert\, \Pi) + \ln(1/\delta) + \ln t + 2}{2t - 1}},
\end{equation}
where $\widehat{R}(\rho)$ is the empirical counterpart of $R_P(\rho)$ and $\mathrm{KL}(\rho \,\Vert\, \Pi)$ is the Kullback-Leibler divergence from posterior to prior. The simplified margin formulation of \cite{mcallester2003simplified}, the Gaussian process specialization of \cite{seeger2002pacbayesian}, and the thermodynamic perspective of \cite{catoni2007pacbayesian} clarified the role of the Donsker-Varadhan inequality and connected the bound to the Gibbs posterior $d\rho_\beta \propto e^{-\beta \widehat{R}(h)}\, d\Pi$. Subsequent refinements have produced sharper bounds in several regimes. The empirical Bernstein inequality of \cite{tolstikhin2013pacbayes} replaces a Hoeffding-type variance proxy by an estimable quantity, and the strongly quasiconvex bound of \cite{thiemann2017strongly} yields a single-parameter convex objective amenable to alternating minimization. The analysis of \cite{germain2009pacbayes} derives convex Kullback-Leibler-regularized learning objectives directly from the bound, while the variational study of \cite{alquier2016variational} establishes that variational approximations of the Gibbs posterior preserve the same convergence rate as the exact posterior. PAC-Bayesian techniques have also been deployed in structured prediction, for instance in the phoneme recognition algorithm of \cite{keshet2011pacbayesian}.

\paragraph{} A separate line of work has addressed a critical limitation of the classical PAC-Bayes apparatus, namely that the resulting bounds are stated at one fixed sample size $t$ and need not remain valid under data-dependent stopping rules or checkpoint selection. The unified recipe of \cite{chugg2023unified} derives time-uniform PAC-Bayes inequalities by combining nonnegative supermartingales, the method of mixtures, the Donsker-Varadhan formula, and Ville's maximal inequality, producing certificates that hold simultaneously at every stopping time. The complementary analysis of \cite{rodriguezgalvez2024more} extends the construction to losses with general tail behavior and to anytime validity, recovering and strengthening Catoni-type bounds in the bounded-loss case and yielding fast-rate and mixed-rate inequalities for sub-exponential and heavier tails. These developments are essential for sequential learning, in which the posterior, the ratio model, and the diagnostics evolve along the optimization trajectory, and in which a fixed-time bound becomes formally invalid as soon as intermediate quantities are inspected to choose a checkpoint.

\paragraph{} Despite this body of work, no existing framework jointly addresses the three statistical concerns identified above for learning under distribution shift: calibration of the learned ratio through hard integral constraints, control of the heavy-tailed behavior of importance weights, and time-uniform PAC-Bayes generalization certificates for the resulting weighted empirical risk. Existing direct ratio estimators rely on integral identities at the population level, but do not enforce them as hard constraints on the learned model, so the induced reweighted source law $\candidate Q$ is not a probability measure in general. Conversely, PAC-Bayes analyses of importance-weighted objectives have largely been performed at fixed sample sizes, without provision for adaptive stopping or for checkpoint selection along the training trajectory.

\paragraph{} The objective of this work is to introduce a unified framework that addresses these three concerns within a single mathematical formulation. The proposed framework rests on four conceptual layers. First, a measure-theoretic foundation, developed in Section~\ref{sec:foundation}, makes precise the change-of-measure identity~\eqref{eq:intro_change_of_measure}, the weighted population and empirical risks, and the bias-variance decomposition~\eqref{eq:intro_decomposition}. Second, a constrained density-ratio network learns a parameterized ratio $\candidate$ subject to the normalization identity $\expectation{Q}{\candidate} = 1$ and to optional moment-matching identities $\expectation{Q}{\candidate(Z)\,\phi_j(Z)} = \expectation{P}{\phi_j(Z)}$ for selected test functions $\phi_j$, while a second-moment penalty, optional clipping, and tempering control the tail of the learned weights and the effective sample size. Third, a PAC-Bayes analysis tailored to the importance-weighted empirical risk produces a fixed-time generalization bound for the weighted objective, accounting explicitly for the bias term induced by ratio miscalibration. Fourth, a time-uniform refinement strengthens the fixed-time bound into an anytime certificate that holds simultaneously for every sample size $t \geq t_{\min}$ on a single high-probability event, constructed by geometric peeling across epochs in the spirit of \cite{chugg2023unified, rodriguezgalvez2024more} and incurring only a $\sqrt{\log \log t}$ overhead relative to the fixed-time bound.

\paragraph{} The contributions of this work may be summarized as follows. A constrained neural density-ratio estimator is introduced in which hard integral constraints replace the soft regularization typical of discriminator-style and naive penalty approaches, ensuring that the induced measure $P_\theta = \candidate Q$ is calibrated against $P$ on a prescribed family of observables. A variance-aware training procedure based on an augmented Lagrangian formulation is developed for these constraints, with explicit second-moment penalization and tail-control mechanisms targeting the root cause of unstable importance sampling. A weighted PAC-Bayes generalization bound is proved for the importance-weighted empirical risk, and the network-weighted Gibbs posterior is identified as the unique minimizer of the corresponding Kullback-Leibler-regularized objective. A stability theorem quantifies how perturbations of the learned ratio in $L^2(Q)$ propagate to the optimal posterior and to the certified target risk, decoupling the ratio-approximation error from the PAC-Bayes layer. The fixed-time bound is then strengthened to an anytime certificate, constructed by geometric peeling across epochs, that remains valid simultaneously at every sample size $t \geq t_{\min}$ at the cost of a $\sqrt{\log \log t}$ overhead and supports adaptive stopping and checkpoint selection. Finally, a pre-registered two-campaign evaluation protocol is proposed in which ratio quality, downstream weighted-risk error, effective-sample-size and tail diagnostics, training stability, and anytime coverage are reported jointly against committed thresholds, enabling fair comparison with classical and discriminator-based baselines on both a controlled patch test against analytic ground truth and a real-data deployment on openly licensed tabular datasets.

\paragraph{} The remainder of the paper is organized as follows. Section~\ref{sec:foundation} establishes the measure-theoretic foundation and the canonical risk decomposition. Section~\ref{sec:drn} introduces the constrained density-ratio network, derives its calibration and stability properties from integral identities and second-moment control, and describes the augmented-Lagrangian training procedure. Section~\ref{sec:pacbayes} develops the weighted PAC-Bayes bound, identifies the network-weighted Gibbs posterior as the minimizer of the corresponding Kullback-Leibler-regularized objective, establishes a stability theorem under $L^2(Q)$ perturbations of the learned ratio, and strengthens the fixed-time bound into an anytime certificate via geometric peeling across epochs. Section~\ref{sec:numerical} reports the pre-registered numerical study, comprising the patch test of Section~\ref{sec:patch_test}, which validates each layer of the framework against analytic ground truth in a controlled Gaussian mean-shift setting, and the real-data deployment of Section~\ref{sec:real_data}, which evaluates the framework end-to-end on openly licensed tabular datasets with intrinsic non-synthetic distribution shift. Concluding remarks and directions for future work close the paper.

\section{Mathematical foundation and notation}
\label{sec:foundation}
\paragraph{} Let $\mathcal{Z}$ denote the observation space, endowed with a topology, and let $\mathcal{A}$ be its Borel $\sigma-$algebra. Let $P$ and $Q$ be probability measures on the measurable space $(\mathcal{Z}, \mathcal{A})$, interpreted respectively as target and source distributions. Assume that $P \ll Q$. By the Radon-Nikodym theorem there exists a $Q-$almost surely unique measurable function
\begin{equation}
  \label{eq:real_ratio}
  \trueratio = \frac{dP}{dQ}: \mathcal{Z} \longrightarrow [0,\infty)
\end{equation}
such that
\begin{equation}
  \label{eq:P_measure}
  P(A) = \int \limits_A \trueratio (Z) dQ(Z), \; A \in \mathcal{A}.
\end{equation}
Equivalently, for every measurable function $f:\mathcal{Z} \longrightarrow \mathbb{R}$ such that
\begin{equation}
  \int_{\mathcal{Z}} |f(z)| \, dP(z) < \infty,
\end{equation}
define
\begin{equation}
  \expectation{P}{f(Z)} := \int_{\mathcal{Z}} f(z) \, dP(z).
\end{equation}
Likewise, whenever
\begin{equation}
  \int_{\mathcal{Z}} |\trueratio(z) f(z)| \, dQ(z) < \infty,
\end{equation}
define
\begin{equation}
  \expectation{Q}{\trueratio(Z) f(Z)} := \int_{\mathcal{Z}} \trueratio(z) f(z) \, dQ(z).
\end{equation}
Then
\begin{equation}
  \expectation{P}{f(Z)} = \expectation{Q}{\trueratio(Z) f(Z)}.
\end{equation}

\paragraph{}Let $\mathcal{H}$ be a hypothesis space, and let $\mathcal{L}: \mathcal{H} \longrightarrow \mathbb{R}$ be a measurable loss function. For each $h \in \mathcal{H}$, define the target risk as
\begin{equation}
  \label{eq:target_risk}
  R_P (h) = \mathbb{E}_P [\mathcal{L}(h,Z)].
\end{equation}
For any nonnegative measurable function $r: \mathcal{Z} \longrightarrow [0, \infty)$ such that $r\mathcal{L}(h,\cdot) \in L^1$(Q), the weighted population risk is given by:
\begin{equation}
  \label{eq:weighted_population_risk}
  R_r(h) = \mathbb{E}_Q[r(Z) \mathcal{L}(h,Z)].
\end{equation}
In particular,
\begin{equation}
  R_P(h) = R_{\trueratio}(h).
\end{equation}

\paragraph{}Given independent and identically distributed (i.i.d.) observations $Z_1,...,Z_t \sim Q$, define the weighted empirical risk associated with $r$ by
\begin{equation}
  \label{eq:weighted_empirical_risk}
  \weightedrisk{r}(h) = \frac{1}{t} \sum \limits^{t}_{i=1} r(Z_i) \mathcal{L}(h,Z_i).
\end{equation}
A learned density-ratio model is a measurable function $r_\theta: \mathcal{Z} \longrightarrow \mathbb{R}_+$ intended to approximate $\trueratio$, with $\theta \in \Theta$ and $\Theta$ a parameter space. Such a function induces a finite measure $P_\theta$ in $(\mathcal{Z}, \mathcal{A})$ by :
\begin{equation}
  P_\theta (A) = \int \limits_A r_\theta (z) dQ(z), \; A\in \mathcal{A}.
\end{equation}
If $\expectation{Q}{\trueratio}=1$, then $P_\theta$ is a probability measure. More generally, the true ratio satisfies 
\begin{equation}
  \label{eq:constraints}
  \expectation{Q}{\trueratio} = 1 \quad \text{and} \quad \expectation{Q}{\trueratio (Z) \phi (Z)} = \expectation{P}{\phi(Z)},
\end{equation}
for every measurable test function $\phi$ for which the expectation exists. These identities motivate the normalization and moment constraints imposed on $r_\theta$. This matches the framework's use of $\trueratio$, the weighted risk $R_r$, and the normalization and moment-matching constraints used to calibrate learned ratios. Those results are formally proved in Propositions \ref{prop:p_theta_prob_measure} and \ref{prop:p_theta_constraint}. In this way, when $r = r_\theta$ is a learned density-ratio model, one writes:
\begin{equation}
  \label{eq:weighted_empirical_risk_theta}
  \weightedrisk{\theta} (h) = \weightedrisk{r_\theta} (h) = \frac{1}{t} \sum \limits^t_{i=1} r_\theta (Z_i) \mathcal{L}(h,Z_i).
\end{equation}

\paragraph{}For a measurable function $f$, the space $L^1(Q)$ is defined as
\begin{equation}
  L^1(Q) = \left\{ f: \mathcal{Z} \longrightarrow \mathbb{R} \; \text{measurable}: \;  \int \limits_{\mathcal{Z}} \left| f(z) \right|dQ(z) < \infty \right\}, \; \|f\|_{L^1(Q)} = \int \limits_{\mathcal{Z}} |f(z)|dQ(z),
\end{equation}
and the space $L^2(Q)$ is defined as
\begin{equation}
  L^2(Q) = \left\{ f: \mathcal{Z} \longrightarrow \mathbb{R} \; \text{measurable}: \;  \int \limits_{\mathcal{Z}} \left| f(z) \right|^2dQ(z) < \infty \right\}, \|f\|_{L^2(Q)} = \left( \int \limits_{\mathcal{Z}} |f(z)|^2dQ(z) \right)^{1/2}.
\end{equation}
Thus, $L^1(Q)$ consists of all $Q-$integrable measurable functions, whereas $L^2(Q)$ consists of all $Q-$square-integrable measurable functions.

\begin{proposition} \label{prop:change_of_measure}
  Assume $P\ll Q$, and let $\trueratio = \frac{dP}{dQ}$. Then, for every $h \in \mathcal{H}$ such that $\mathcal{L, \cdot} \in L^1(P)$, it holds that
  \begin{equation}
    R_P(h) = \expectation{P}{\mathcal{L}(h,Z)} = \expectation{Q}{\trueratio(Z)\mathcal{L}(h,Z)},
  \end{equation}
  with $Z \in \mathcal{Z}$.
\end{proposition}

\begin{proof}
  Since $P \ll Q$, the Radon-Nikodym theorem yields a measurable function $\trueratio: \mathcal{Z} \longrightarrow [0, \infty)$ such that
  \begin{equation}
    P(A) = \int \limits_A \trueratio(z)dQ(z), \; A \in \mathcal{A}.
  \end{equation}
  By the defining property of the Radon-Nikodym derivative, for every measurable function $f: \mathcal{Z} \longrightarrow \mathbb{R}$ such that $f \in L^1(P)$ it holds:
  \begin{equation}
    \int \limits_{\mathcal{Z}} f(z)dP(z) = \int \limits_{\mathcal{Z}} f(z) \trueratio(z)dQ(z).
  \end{equation}
  Applying this identity with $f(z) = \mathcal{L}(h,Z)$, which is measurable by assumption, gives:
  \begin{equation}
    \expectation{P}{\mathcal{L}(h,Z)} = \expectation{Q}{\trueratio(Z)\mathcal{L}(h,Z)}.
  \end{equation}
\end{proof}

\begin{proposition}\label{prop:learned_ratio_risk}
  Let $\candidate: \mathcal{Z} \longrightarrow \mathbb{R}_+$ be a measurable candidate ratio such that $\candidate \mathcal{L}(h, \cdot) \in L^1(Q)$. Then, for every $h \in \mathcal{H}$, 
  \begin{equation}
    \label{eq:risk_decomposition}
    R_P(h) - \weightedrisk{\theta} (h) = (R_P(h) - R_{r_\theta}(h)) + (R_{r_\theta}(h) - \weightedrisk{\theta}(h)).
  \end{equation}
\end{proposition}
\begin{proof}
  Immediate, by adding and subtracting $R_{r_\theta}(h)$.
\end{proof}

\paragraph{} In (\ref{eq:risk_decomposition}), the first term $R_P(h) - R_{r_\theta}(h)$ is the bias induces by ratio misspecification, while the second term $R_{r_\theta}(h) - \weightedrisk{\theta}(h)$, is the empirical approximation or generalization error associated with the learned weighted risk. The density-ratio network addresses the first term, while the PAC-Bayes approach addresses the second.

\begin{proposition}\label{prop:bias_control_via_ratio_approx}
  Let $\trueratio = dP/dQ$, and let $\candidate: \mathcal{Z} \longrightarrow \mathbb{R}_+$ be a measurable candidate ratio such that $\candidate \mathcal{L}(h, \cdot) \in L^1(Q)$. Then, for every $h \in \mathcal{H}$,
  \begin{equation}
    R_P (h) - R_{\candidate}(h) = \expectation{Q}{(\trueratio - \candidate)(Z)\mathcal{L}(h,Z)}.
  \end{equation}
  In particular,
  \begin{enumerate}
    \item if there exists $C_L < \infty $ such that $|\mathcal{L}(h,Z)| \leq C_L$ for all $Z \in \mathcal{Z}$, then
    \begin{equation}
      |R_P(h)-R_{\candidate}(h)| \leq C_L \| \trueratio - \candidate \|_{L^1(Q)};
    \end{equation}
    \item if $\trueratio - \candidate \in L^2(Q)$ and $\mathcal{L}(h, \cdot) \in L^2(Q)$, then
    \begin{equation}
      |R_P(h)-R_{\candidate}(h)| \leq \| \trueratio - \candidate \|_{L^2(Q)} \| \mathcal{L} \|_{L^2(Q)}.
    \end{equation}
  \end{enumerate}
\end{proposition}

\begin{proof}
  By Proposition \ref{prop:change_of_measure}, one knows that $R_P(h) = \expectation{Q}{\trueratio(Z)\mathcal{L}(h,Z)}$, whereas, by definition $R_{\candidate}(h) = \expectation{Q}{\candidate (Z) \mathcal{L}(h,Z)}$. Subtracting the two identities gives:
  \begin{equation}
    R_P(h) - R_{\candidate}(h) = \expectation{Q}{(\trueratio - \candidate)(Z) \mathcal{L}(h,Z)}.
  \end{equation}

  Assume first that $|\mathcal{L}(h,Z)| \leq C_L$, with $C_L > 0$, for all $Z \in \mathcal{Z}$. Then, 
  \begin{equation}
    |R_P(h)-R_{\candidate}(h)| = |\expectation{Q}{(\trueratio - \candidate)\mathcal{L}(h,Z)}| \leq \expectation{Q}{|\trueratio - \candidate| (Z) |\mathcal{L}(h,Z)|} \leq C_L \expectation{Q}{|\trueratio - \candidate|(Z)}.
  \end{equation}
  Hence, 
  \begin{equation}
    |R_P(h)-R_{\candidate}(h)| \leq C_L \| \trueratio - \candidate \|_{L^1(Q)}.
  \end{equation}

  Now, assume instead that $\trueratio - \candidate \in L^2(Q)$ and $\mathcal{L}(h, \cdot) \in L^2(Q)$. By the Cauchy-Schwarz inequality:
  \begin{equation}
    |R_P(h)-R_{\candidate}(h)| \leq \left(\expectation{Q}{(\trueratio - \candidate)^2(Z)}\right)^{1/2} \left(\expectation{Q}{\mathcal{L}^2(h,Z)}\right)^{1/2}.
  \end{equation}
\end{proof}

\paragraph{}This proposition shows that downstream target-risk bias is controlled by the approximation error of the learned ration in a norm induced by $Q$. Under bounded loss, $L^1(Q)$-accuracy is sufficient. Under square-integrability, $L^2(Q)$-accuracy suffices. Thus, the relevant notion of ratio quality is precisely the one that controls the bias term $R_P(h)-R_{\candidate}(h)$.

\section{Density-ratio network}
\label{sec:drn}

\paragraph{} A density-ratio network is a parameterized nonnegative measurable function $\candidate: \mathcal{Z} \longrightarrow \mathbb{R}_+$, trained to approximate the Radon-Nikodym derivative $\trueratio = dP/dQ$ from independent samples of the source law $Q$ and the target law $P$. The role of this object in the framework is fixed by the central decomposition of Section~\ref{sec:foundation}, in which the gap between the unweighted target risk $R_P(\rho)$ and the importance-weighted empirical risk $\weightedrisk{\theta}(\rho)$ splits into a ratio-bias term governed by how closely $\candidate$ approximates $\trueratio$ in a norm that controls the loss, and a generalization-gap term governed by the variability of the weighted random variable $\candidate(Z)\,\mathcal{L}(h, Z)$. Both terms are statistical properties of $\candidate$ itself, so the design of the network must simultaneously address approximation accuracy and weight stability. A pointwise-accurate ratio that produces heavy-tailed weights is not useful for downstream learning, and a tightly stable ratio that misrepresents the change of measure is not useful either.

\paragraph{} The present section develops the constrained density-ratio network as a single object that addresses both concerns. Three structural identities of a Radon-Nikodym derivative are imposed on $\candidate$ as hard integral constraints rather than as soft penalty terms. Normalization, $\expectation{Q}{\candidate} = 1$, makes the induced set function $P_\theta := \candidate Q$ a probability measure on the observation space. Moment matching on a prescribed family of test functions $\phi_1, \dots, \phi_m$, in the form $\expectation{Q}{\candidate(Z)\, \phi_j(Z)} = \expectation{P}{\phi_j(Z)}$, makes the induced measure $P_\theta$ agree with the target law $P$ on the linear span of those functions in expectation. A second-moment penalty on $\candidate$ controls the variance of the weighted loss and, equivalently, the effective sample size of the weighted empirical risk. The section is organized around three layers. The mathematical formulation collects the propositions that link normalization, moment matching, second-moment control, and the least-squares fitting criterion to their roles in the central decomposition, and identifies the population objective whose minimization recovers $\trueratio$ under realizability and selects the closest admissible approximation in $L^2(Q)$ under misspecification. The augmented-Lagrangian method specifies the training scheme by which the calibration constraints are enforced in practice, with explicit primal-dual dynamics and the empirical-level realization of the population identities. A final discussion of practical post-hoc transformations, namely post-hoc normalization, clipping, and tempering, closes the section by separating identity-preserving operations, which leave the change-of-measure structure intact, from training-time variance-reduction devices, which do not.

\subsection{Mathematical formulation}
\paragraph{} The preceding definitions and propositions identify the statistical role of a density-ratio model. It should induce a valid reweighting of the source law, approximate the true Radon-Nikodym derivative closely enough to control target risk-bias, and ultimately make the weighted empirical risk a reliable proxy for the target risk. Next the formalization of the structural properties that justify the normalization, calibration, and stability mechanisms imposed on the learned ratio $\candidate$ are formalized.

\paragraph{} A learned density-ratio model $\candidate$ should not be viewed merely as a collection of weights attached to source samples. Because when it reweights the source measure $Q$, it naturally induces a new measure on the observation space. This point is mathematically relevant. Once $\candidate$ is nonnegative and normalized, the reweighted law $\candidate Q$ becomes a genuine probability measure, which may be interpreted as the approximate target law represented by the DRN. Thus, the DRN does not only approximate the Radon-Nikodym derivative $\trueratio$, it also defines an approximate target distribution $P_\theta$ against which expectations and risk can be evaluated. The following proposition formalizes this basic measure-theoretic interpretation. 

\begin{proposition}\label{prop:p_theta_prob_measure}
  Let $\candidate: \mathcal{Z}\longrightarrow \mathbb{R}_+$ be a measurable function, and define
  \begin{equation}
    P_\theta (A) = \int \limits_A \candidate (z) dQ(z), \; A\in \mathcal{A}.
  \end{equation}
  If
  \begin{equation}
    \expectation{Q}{\candidate} = \int \limits_\mathcal{Z} \candidate (z) dQ(z) = 1,
  \end{equation}
  then $P_\theta$ is a probability measure on $(\mathcal{Z}, \mathcal{A})$.
\end{proposition}

\begin{proof}
  The strategy is to verify that $P_\theta$ satisfies the axioms of probability measure. Since $r_\theta$ is measurable and nonnegative, the set function $A \mapsto P_\theta (A)$ is well-defined for every $A \in \mathcal{A}$ with values in $[0, \infty]$. First because $P_\theta ( \emptyset) = 0$. Second let $(A_k)_{k \geq 1} \subseteq \mathcal{A}$ be pairwise disjoint, and set
  \begin{equation}
    A = \bigcup \limits^\infty_{k=1} A_k.
  \end{equation}
  Then,
  \begin{equation}
    \mathbf{1}_A(Z) = \sum \limits^\infty_{k=1} \mathbf{1}_{A_k}(z),
  \end{equation}
  for all $z \in \mathcal{Z}$, because the sets are disjoint. Therefore, 
  \begin{equation}
    P_\theta (A) = \int \limits_\mathcal{Z} \candidate (z) \mathbf{1}_A(z) dQ(z) = \int \limits_\mathcal{Z} \candidate (z) \sum \limits^\infty_{k=1} \mathbf{1}_{A_k}(z)dQ(z).
  \end{equation}
  Since summands are nonnegative, the Monotone Convergence Theorem (see Theorem \ref{theo:mct}) yields:
  \begin{equation}
    P_\theta (A) = \sum \limits_{k=1}^\infty \int \limits_{\mathcal{Z}} \candidate (z) \mathbf{1}_{A_k}(z)dQ(z) = \sum \limits_{k=1}^\infty P_\theta (A_k).
  \end{equation}
  Thus, $P_\theta$ is countably additive.

  \paragraph{} Finally, its total mass is 
  \begin{equation}
    P_\theta (\mathcal{Z}) = \int \limits_{\mathcal{Z}} \candidate (z) dQ (z) = \expectation{Q}{\candidate} = 1.
  \end{equation}
  Hence, $P_\theta$ is a countably additive nonnegative measure with total mass $1$. As result, $P_\theta$ is a probability measure on $(\mathcal{Z}, \mathcal{A})$.
\end{proof}

\paragraph{} Since $P_\theta$ is now known to be a probability measure, expectations under $P_\theta$ are well-defined. In particular, for every measurable function $f:\mathcal{Z}\longrightarrow \mathbb{R}$ such that $\candidate f \in L^1(Q)$, it holds:
\begin{equation}
  \expectation{P_\theta}{f(Z)} = \int \limits_{\mathcal{Z}} f(z) dP_\theta (z) = \int \limits_{\mathcal{Z}} f(z) \candidate (z) dQ(z) = \expectation{Q}{\candidate (z) f(z)}.
\end{equation}
Hence, for each $h \in \mathcal{H}$ such that $\candidate \mathcal{L}(h,\cdot) \in L^1(Q)$:
\begin{equation}
  R_{\candidate} (h) = \expectation{P_\theta}{\mathcal{L}(h,Z)}.
\end{equation}
Thus, the weighted population risk must be viewed as the ordinary risk under the induced law $P_\theta$. This observation will be used in the next proposition to interpret normalization and moment constraint under which $P_\theta$ matches $P$ on selected test functions.

\begin{proposition}\label{prop:p_theta_constraint}
  Let $\candidate: \mathcal{Z} \longrightarrow \mathbb{R}_+$ be a measurable function, and let $P_\theta$ be the induced measure defined by
  \begin{equation}
    P_\theta (A) = \int \limits_A \candidate (z) dQ(z), \; A\in \mathcal{A}.
  \end{equation}
  Assume that for measurable functions $\phi_1,...,\phi_m: \mathcal{Z}\longrightarrow \mathbb{R}$,
  \begin{equation}
    \expectation{Q}{\candidate (Z) \phi_j(Z)} = \expectation{P}{\phi_j (Z)}, \; j = 1,...,m,
  \end{equation}
  and that all these expectations exists. Then, for every function $\phi \in \text{span} \{ \phi_1, ..., \phi_m \}$, one has
  \begin{equation}
    \expectation{P_\theta}{\phi (Z)} = \expectation{P}{\phi (Z)}
  \end{equation}
  providing the expectations exist. 
\end{proposition}

\begin{proof}
  Let $\phi \in \text{span} \{ \phi_1, ..., \phi_m \}$. Then, there exist coefficients $\alpha_1, ..., \alpha_m \in \mathbb{R}$ such that
  \begin{equation}
    \phi = \sum \limits^{m}_{j=1} \alpha_j \phi_j.
  \end{equation}
  By the definition of $P_\theta$, and since Proposition \ref{prop:p_theta_prob_measure} shows that $P_\theta$ is a probability measure, for every measurable $f$ such that $\candidate f \in L^1(Q)$:
  \begin{equation}
    \expectation{P_\theta}{f(Z)} = \int \limits_{\mathcal{Z}} f(z)dP_\theta (z) = \int \limits_\mathcal{Z} f(z)\candidate(z) dQ(z) = \expectation{Q}{\candidate (Z) f(Z)}.
  \end{equation}
  Setting $f = \phi_j$:
  \begin{equation}
    \expectation{P_\theta}{\phi_j(Z)} = \expectation{Q}{\candidate (Z) \phi_j (Z)}.
  \end{equation}
  By the assumed moment constraint:
  \begin{equation}
    \expectation{Q}{\candidate (Z) \phi_j (Z)} = \expectation{P}{\phi_j (Z)}.
  \end{equation}
  Hence, 
  \begin{equation}
    \expectation{P_\theta}{\phi (Z)} = \sum \limits^{m}_{j=1} \alpha_j \expectation{P}{\phi_j (Z)}.
  \end{equation}
  Using the linearity of the expectation under $P$, this becomes:
  \begin{equation}
    \expectation{P_\theta}{\phi (Z)} = \expectation{P}{ \sum \limits^{m}_{j=1} \alpha_j \phi_j (Z)} = \expectation{P}{\phi (Z)}.
  \end{equation}
  As result, for every $\phi \in \text{span} \{ \phi_1, ..., \phi_m \}$, 
  \begin{equation}
    \expectation{P_\theta}{\phi (Z)} = \expectation{P}{\phi (Z)}.
  \end{equation}
\end{proof}

\paragraph{} Proposition \ref{prop:p_theta_constraint} formalizes the calibration role of the moment constraints. It says that if the learned ratio $\candidate$ transports the moments of the functions $\phi_1,..., \phi_m$ correctly, then the induced law $P_\theta$ agrees with the target law $P$ on the whole linear test class generated by those functions. Thus, the constraints do not merely regularize the ratio model. They ensure that the approximate target measure $P_\theta$ reproduces the target expectations on a prescribed family of observables.

\paragraph{} Calibration alone is not sufficient to guarantee that the weighted empirical risk is statistically reliable. Even if the induced law $P_\theta$ matches the target law $P$ on a prescribed class of test functions, the learned ratio $\candidate$ may still take very large values on rare source samples, causing the weighted empirical averages to become unstable. For this reason, a second ingredient of the DRN framework is the variance control. The key observation is that, for bounded losses, the variability of the weighted loss $\candidate (Z) \mathcal{L}(h,Z)$ is controlled by the second moment of the learned ratio. This gives a direct mathematical justification for penalizing $\expectation{Q}{\candidate^2}$, and more generally for using tail-control mechanisms as clipping, tempering, and effective-sample-size diagnostics. The following proposition makes this precise. 

\begin{proposition}\label{prop:second_moment_control}
  Let $h \in \mathcal{H}$, and assume that the loss satisfies
  \begin{equation}
    \label{eq:c_l_bound}
    |\mathcal{L}(h,Z)| \leq C_L, 
  \end{equation}
  for all $Z \in \mathcal{Z}$ and for some constant $C_L \geq 0$. Let $\candidate: \mathcal{Z} \longrightarrow \mathbb{R}_+$ be measurable, and suppose that $\candidate \in L^2(Q)$. Then,
  \begin{equation}
    \label{eq:variance_bound}
    \variance{Q}{\candidate (Z) \mathcal{L}(h,Z)} \leq C_L^2 \expectation{Q}{\candidate(Z)^2}.
  \end{equation}
\end{proposition}

\begin{proof}
  Let $X = \candidate (Z) \mathcal{L}(h,Z)$. By the elementary inequality
  \begin{equation}
    \label{eq:aux_variance_bound}
    \variance{Q}{X} = \expectation{Q}{X^2} - (\expectation{Q}{X})^2 \leq \expectation{Q}{X^2}
  \end{equation}
  it suffices to upper-bound the second moment of $X$.

  \paragraph{} By the bound (\ref{eq:c_l_bound}), one has
  \begin{equation}
    \mathcal{L}(h,Z)^2 \leq C_L^2
  \end{equation}
  almost everywhere. Therefore,
  \begin{equation}
    X^2 = \candidate (Z)^2 \mathcal{L}(h,Z)^2 \leq C_L^2 \expectation{Q}{\candidate (Z)^2}
  \end{equation}
  Combining with the variance bound in (\ref{eq:aux_variance_bound}) yields (\ref{eq:variance_bound}).
\end{proof}

\paragraph{} Since Proposition \ref{prop:second_moment_control}, established the second moment of $\candidate$ controls the variance of the weighted loss, the natural step is to relate this same quantity to the effective sample size (ESS). This is relevant because ESS is the standard practical diagnosis for weight degeneracy in important weightning. When a few samples receive very large weights, the empirical average behaves as thought it was based on far fewer than $t$ observations. In this sense, Proposition \ref{prop:ess} connects the variance-oriented analysis of Proposition \ref{prop:second_moment_control} to the empirical stability diagnostics used in practice. In particular, it shows that large weight second moments correspond to small ESS, both at the empirical level and, under normalization, at the population level. This provides a direct mathematical interpretation of ESS as a surrogate measure of stability for the learned ratio model.

\begin{proposition}\label{prop:ess}
  Let $Z_1, \ldots, Z_t \iid Q$, and define the weights 
  \begin{equation}
    w_i = \candidate (Z), \; i=1,...,t.
  \end{equation}
  Assume that $w_i \geq 0$ for all $i$, and define the empirical ESS by
  \begin{equation}
    \text{ESS}_t = \frac{\left( \sum^t_{i=1} w_i \right)^2}{\sum^t_{i=1}w_i^2},
  \end{equation}
  whenever $\sum^t_{i=1}w_i > 0$. Then
  \begin{enumerate}
    \item One has
    \begin{equation}
      1 \leq \text{ESS}_t \leq t
    \end{equation}
    whenever at least one weight is nonzero;
    \item If the empirical weights are normalized so that 
    \begin{equation}
      \frac{1}{t} \sum \limits^t_{i=1} w_i \approx 1,
    \end{equation}
    then
    \begin{equation}
      \frac{\text{ESS}_t}{t} \approx \frac{1}{1/t \sum^t_{i=1} w_i^2};
    \end{equation}
    \item If moreover $\expectation{Q}{\candidate (Z)} = 1$ and $\candidate \in L^2(Q)$, the by the law of large numbers,
    \begin{equation}
      \label{eq:aux_ess_division}
      \frac{\text{ESS}_t}{t} = \frac{\left( 1/t \sum^t_{i=1} \candidate (Z) \right)^2}{1/t \sum^t_{i=1} \candidate (Z)^2} \longrightarrow \frac{1}{\expectation{Q}{\candidate (Z)^2}},
    \end{equation}
    almost everywhere. In particular, larger second moments of $\candidate$ corresponds to smaller asymptotic effective sample sizes.
  \end{enumerate}
\end{proposition}

\begin{proof}
  Item 2 is trivial, so the focus is going to be on items 1 and 3.

  \paragraph{1.} To show the upper bound, apply the Cauchy-Schwarz inequality with $(w_1, ..., w_t)$ and $(1,...,1)$:
  \begin{equation}
    \left( \sum \limits_{i=1}^t w_i \right)^2 \leq  \left( \sum \limits_{i=1}^t w_i^2 \right)\left( \sum \limits_{i=1}^t 1^2 \right) = t \sum \limits^t_{i=1}w_i^2.
  \end{equation}
  Hence, $\text{ESS}_t \leq t$.

  \paragraph{} For the lower bound, expand the squares:
  \begin{equation}
    \begin{split}
      \left( \sum \limits_{i=1}^t w_i \right)^2 &= \left( \sum \limits_{i=1}^t w_i \right) \left( \sum \limits_{j=1}^t w_j \right) = \sum \limits_{i=1}^t \sum \limits_{j=1}^t w_i w_j \\
      &= \sum \limits_{i=1}^t w_i^2 + \sum \limits_{\substack{1 \leq i, j \leq t \\ i \neq j}} w_i w_j = \sum \limits_{i=1}^t w_i^2 + 2 \sum \limits_{1 \leq i < j \leq t} w_i w_j.
    \end{split}
  \end{equation}
  Since $w_i \geq 0$, every cross-term $w_i w_j$ is nonnegative, so
  \begin{equation}
    \left( \sum \limits_{i=1}^t w_i \right)^2 = \sum \limits_{i=1}^t w_i^2 + 2 \sum \limits_{1 \leq i < j \leq t} w_i w_j \geq  \sum \limits_{i=1}^t w_i^2.
  \end{equation}
  Thus, $\text{ESS}_t \geq 1$.

  \paragraph{} Assume that $\expectation{Q}{\candidate (Z)} = 1$ and $\candidate \in L^2 (Q)$. Then, $\candidate \in L^1 (Q)$ as well, and the strong law of large numbers yields:
  \begin{equation}
    \frac{1}{t} \sum \limits^t_{i = 1} \candidate (Z_i) \longrightarrow \expectation{Q}{\candidate (Z)} = 1, 
  \end{equation}
  $Q$-almost everywhere, and
  \begin{equation}
    \frac{1}{t} \sum \limits^t_{i = 1} \candidate (Z_i)^2 \longrightarrow \expectation{Q}{\candidate (Z)^2}
  \end{equation}
  $Q$-almost everywhere. Hence, on the event where both convergences and (\ref{eq:aux_ess_division}) is proved.
\end{proof}

\paragraph{} The preceding results establish the structural role of the learned density-ratio model. Proposition \ref{prop:p_theta_prob_measure} shows that a normalized nonnegative ratio induces a probability measure $P_\theta$ on the observation space. Proposition \ref{prop:p_theta_constraint} shows that moment constraints enforce calibration on this induced law on the prescribed test class. Propositions \ref{prop:second_moment_control} and \ref{prop:ess} then relate the second moment of the learned ratio to the variance of the weighted loss and to the effective sample size, thereby clarifying the stability considerations underlying the framework. These results justify the form of the constraints imposed on the ratio model. It remains to justify the statistical meaning of the fitting criterion itself.

\paragraph{} The relevant question is whether the population objective used to learn the ratio is correctly aligned with the true density ratio $\trueratio$. More precisely, one wishes to determine whether minimization of the constrained population objective recovers the true change of measure under realizability, and what objective is selected when the admissible class does not contain $\trueratio$. The next theorem answers this question for a least-square density-ratio criterion. It shows that the population objective is equivalent, up to an additive constant, to the squared $L^2(Q)$-distance to the true ratio. Consequently, under realizability the true density-ratio is recovered, whereas under misspecification the selected solution is the closest admissible approximation in $L^2(Q)$.

\begin{remark}
  In the context of this work, realizability means that the true density-ratio actually belongs to the model class being optimized over. Formally, if the admissible class is $\mathcal{R} \subseteq L^2(Q)$ then the realizability assumption is $\trueratio \in \mathcal{R}$.
\end{remark}

\paragraph{} Before properly formulating the next theorem,  it is useful to explain the choices of the fitting functional. The purpose of ratio learning in the present framework is not merely to approximate $\trueratio$ in an abstract sense, but to do so in a way that makes the weighted representation of the target risk reliable. Proposition \ref{prop:bias_control_via_ratio_approx} already shows that the bias induced by replacing $\trueratio$ with a candidate ratio $r$ can be controlled by $L^1(Q)$- or $L^2(Q)$- approximation error. This makes an $L^2(Q)$-based fitting principles especially natural. For this reason, consider the population functional
\begin{equation}
  \label{eq:populational_functional}
  \mathbb{J} (r) = \frac{1}{2} \expectation{Q}{r(Z)^2} - \expectation{P}{r(Z)}.
\end{equation}

\paragraph{} Using the change-of-measure identity with $\trueratio$ one has:
\begin{equation}
  \label{eq:populational_change_of_measure}
  \expectation{P}{r (Z)} = \expectation{Q}{\trueratio (Z) r (Z)}
\end{equation}
From (\ref{eq:populational_change_of_measure}) in (\ref{eq:populational_functional}):
\begin{equation}
  \label{eq:aux_populational_functional_transformed}
  \mathbb{J}(r) = \frac{1}{2} \expectation{Q}{r(Z)^2} - \expectation{Q}{\trueratio (Z) r (Z)}.
\end{equation}
On the other hand,
\begin{equation}
  \label{eq:aux_populational_functional_norm}
  \frac{1}{2} \| r - \trueratio \|^2_{L^2(Q)} = \frac{1}{2} \expectation{Q}{r(Z)^2} - \expectation{Q}{\trueratio (Z) r (Z)} + \frac{1}{2} \expectation{Q}{\trueratio (Z)^2}.
\end{equation}
From (\ref{eq:aux_populational_functional_norm}) in (\ref{eq:aux_populational_functional_transformed}):
\begin{equation}
  \label{eq:population_functional_final_form}
  \mathbb{J} (r) = \frac{1}{2} \| r - \trueratio \|^2_{L^2(Q)} - \frac{1}{2}\expectation{Q}{\trueratio (Z)}.
\end{equation}
Hence, the functional $\mathbb{J}$ differs from the squared $L^2(Q)$-distance to $\trueratio$ only by a constant independent of $r$. It follows that the minimization of $\mathbb{J}$ is equivalent to the minimization of the $L^2(Q)$-distance to the true density-ratio.

\begin{theo}\label{theo:r_star_uniqueness}
  Let $\mathcal{R} \subseteq L^2(Q)$ be a class of measurable functions $r: \mathcal{Z} \longrightarrow [0,\infty)$. Suppose that for every $r \in \mathcal{R}$, the expectations appearing below are well-defined and consider the population functional defined in (\ref{eq:populational_functional}). Assume further that $\trueratio \in L^2(Q)$. Then, the following hold:
  \begin{enumerate}
    \item For every $r \in \mathcal{R}$,
    \begin{equation}
      \label{eq:population_functional_difference}
      \mathbb{J}(r) - \mathbb{J}(\trueratio) = \frac{1}{2} \| \trueratio - r \|^2_{L^2(Q)};
    \end{equation}
    \item Consequently, every minimizer of $\mathbb{J}$ over $\mathcal{R}$ is also a minimizer of 
    \begin{equation}
      r \mapsto \| r - \trueratio \|^2_{L^2(Q)}
    \end{equation} 
    over $\mathcal{R}$ and conversely;
    \item If, in addition, the realizability condition $\trueratio \in \mathcal{R}$ holds, then $\trueratio$ is the unique minimizer of $\mathbb{J}$ over $\mathcal{R}$.
  \end{enumerate}
\end{theo}

\begin{proof}
  \textbf{1.} The first part comes directly from equations (\ref{eq:populational_change_of_measure}) to (\ref{eq:population_functional_final_form}).

  \paragraph{2.} The identity (\ref{eq:population_functional_difference}) shows that:
  \begin{equation}
    \mathbb{J} (r) = \frac{1}{2} \| r - \trueratio \|^2_{L^2(Q)} + \mathbb{J}(\trueratio),
  \end{equation}
  for all $r \in \mathcal{R}$. Note that, from (\ref{eq:population_functional_final_form}), $\mathbb{J}(\trueratio)$ is given by:
  \begin{equation}
    \mathbb{J} (\trueratio) = - \frac{1}{2} \expectation{Q}{\trueratio (Z)^2}.
  \end{equation}
  Since $\mathbb{J} (\trueratio)$ is a constant independent of $r$, minimizing $\mathbb{J}(r)$ over $\mathcal{R}$ is equivalent to minimizing $1/2 \| \trueratio - r \|^2_{L^2(Q)}$ over $\mathcal{R}$.

  \paragraph{3.} Assume that $\trueratio \in \mathcal{R}$. Then, for every $r \in \mathcal{R}$,
  \begin{equation}
    \mathbb{J}(r) - \mathbb{J}(\trueratio) = \frac{1}{2} \| \trueratio - r \|^2_{L^2(Q)} \geq 0.
  \end{equation}
  Thus, $\trueratio$ is a minimizer over $\mathcal{R}$. If $r \in \mathcal{R}$ is another minimizer, then
  \begin{equation}
    \mathbb{J} (r) = \mathbb{J}(\trueratio),
  \end{equation}
  so that
  \begin{equation}
    \frac{1}{2} \| \trueratio - r \|^2_{L^2(Q)} = 0.
  \end{equation}
  Therefore, 
  \begin{equation}
    \| \trueratio - r \|^2_{L^2(Q)} \Rightarrow r = \trueratio
  \end{equation}
  $Q$-almost everywhere. Hence, $\trueratio$ is the unique minimizer in $L^2(Q)$, equivalently unique to $Q$-almost everywhere equality.
\end{proof}

\subsection{Augmented Lagrangian method}

\paragraph{} Until now, the results identify the structural constraints that a learned density-ratio model should satisfy. In particular, normalization and moment-matching conditions are natural because they express defining integral identities of the true Radon-Nikodym derivative and ensure calibration of the induced law on a prescribed test class. What remains is to specify how such constraints are enforced during optimization. A convenient approach is provided by the augmented Lagrangian method, which contains the variational structure of the Lagrange multipliers with the numerical stabilizing effect of quadratic penalties. In the present setting, this yields a practical mechanism for fitting the ratio while simultaneously driving the calibration residuals downward zero.

\paragraph{} Let $\phi_1, ..., \phi_m: \mathcal{Z} \longrightarrow \mathbb{R}$ be measurable test functions. Consider the normalization residual:
\begin{equation}
  g_0 (\theta) = \expectation{Q}{\candidate (Z)} - 1,
\end{equation}
and the moment residuals
\begin{equation}
  g_j (\theta) = \expectation{Q}{\candidate (Z) \phi_j (Z)} - \expectation{P}{\phi_j (Z)}, \quad j = 1, \dots, m,
\end{equation}
whenever these expectations are well-defined. The constrained population problem takes the form
\begin{equation}
  \min \limits_{\theta \in \Theta} \mathbb{J} (\candidate )\; \text{subject to} \; g_0 (\theta) = 0, \; g_j (\theta) = 0, \; j = 1,...,m.  
\end{equation}
Considering (\ref{eq:populational_functional}), by Theorem~\ref{theo:r_star_uniqueness}, this criterion is equivalent, up to a constant independent of $\theta$, to the squared $L^2(Q)$-distance from $\candidate$ to the true ratio $\trueratio$. The role of the constraints is therefore not to replace ratio fitting, but to ensure that the fitted ratio also defines a calibrated change of measure.

\paragraph{} A pure quadratic-penalty approach would replace the constraints' problem by
\begin{equation}
  \mathbb{J}(\theta) + \frac{1}{2} \eta_0 g_0 (\theta)^2 + \frac{1}{2} \sum \limits^m_{j=1}\eta_j g_j (\theta)^2,
\end{equation}
with penalty parameters $\eta_0, \eta_j > 0$. Such formulation discourages constraint violation, but in general it does not enforce the constraints accurately unless the penalty parameters are taken large, which often leads to ill-conditioning. A pure Lagrangian formulation, 
\begin{equation}
  \mathbb{J}(\theta) + \lambda g_0 (\theta) + \sum \limits^m_{j=1} \mu_j g_j (\theta)
\end{equation}
correctly represents the equality constraints at the variational level, but it requires joint optimization over the primal variable and the multiplier $\lambda, \mu_1,...,\mu_m \in \mathbb{R}$, and in convex stochastic this may lead to unstable dynamics. The augmented Lagrangian contains both mechanisms. The multiplier terms encode the quality constraints, while the quadratic terms penalize violations directly and improve local conditioning.

\paragraph{} To this end, introduce a dual variable $\lambda \in \mathbb{R}$ for the normalization constraint and dual variables $\mu_1, ..., \mu_m \in \mathbb{R}$ for the moment constraints. For the penalty parameters $\eta_0, \eta_1, ....,\eta_m>0$, define the augmented Lagrangian functional
\begin{equation}
  \mathbb{L}_{AL} [\theta, \lambda, \bm{\mu}] = \mathbb{J}(\theta) + \lambda g_0 (\theta) + \frac{\eta_0}{2}g_0 (\theta)^2 + \sum \limits^m_{j=1}  \mu_j g_j(\theta) + \frac{\eta_j}{2} g_j (\theta)^2,
\end{equation}
with $\bm{\mu} = (\mu_1,...,\mu_m)$. The fit term $\mathbb{J}(\theta)$ promotes approximation of the true ratio. The linear dual terms $\lambda g_0 (\theta)$ and $\mu_j g_j (\theta)$ impose the equality constraints in variational form, and the quadratic penalty terms $\tfrac{\eta_0}{2} g_0(\theta)^2$ and $\tfrac{\eta_j}{2} g_j (\theta)^2$ make violation costly even before the dual variables have adapted.

\paragraph{} The effect of the constraints at the primal level becomes particularly transparent after differentiation with respect to $\theta$:
\begin{equation}
  \nabla_\theta \mathbb{L}_{AL} [\theta, \lambda, \bm{\mu}] = \nabla_\theta \mathbb{J} (\theta) + (\lambda + \eta_0 g_0 (\theta))\nabla_\theta g_0 (\theta) + \sum \limits^m_{j=1} (\mu_j + \eta_j g_j (\theta))\nabla_\theta g_j (\theta).
\end{equation}
Thus, each constraint contributes through an effective coefficient of the form $\mu_j + \eta_j g_j (\theta)$ (with $\mu_0 := \lambda$, $\eta_0$ as above, $j = 0,\dots,m$), that is, through the sum of an accumulated dual correction and an instantaneous residual-dependent penalty. Note that the quadratic part reacts immediately to current constraint violation, whereas the dual variable carries longer term information about persistent miscalibration.

\paragraph{} The associated optimization scheme alternates a descent step on the primal parameter $\theta$ with ascent steps on the dual variable. Since
\begin{equation}
  \frac{\partial}{\partial \lambda} \mathbb{L}_{AL}[\theta, \lambda, \bm{\mu}] = g_0(\theta), \; \frac{\partial}{\partial \mu_j}\mathbb{L}_{AL}[\theta, \lambda, \bm{\mu}] = g_j (\theta),
\end{equation}
the natural dual updates are
\begin{equation}
  \lambda^{(k+1)} = \lambda^{(k)} + \eta_0 g_0 (\theta^{(k+1)})
\end{equation}
and
\begin{equation}
  \mu_j^{(k+1)} = \mu_j^{(k)} + \eta_j g_j (\theta^{(k+1)}), \; j = 1, ..., m.
\end{equation}
In this way, the dual variables act as running corrections for persistent normalization and moment-transport errors. A positive residual increases the corresponding multiplier, and hence, strengthens the next primal correction in that direction. A negative residual has the opposite effect.

\paragraph{} At the empirical level, the population quantities are replaced by the sample averages. Given source observations $Z_1^Q,..., Z_{n_Q}^Q \sim Q$ and target observations $Z_1^P,...,Z_{n_P}^P \sim P$, define 
\begin{equation}
  \hat{g}_0 (\theta) = \frac{1}{n_Q} \sum \limits^{n_Q}_{i=1} \candidate (Z_i^Q)-1,
\end{equation}
and
\begin{equation}
  \hat{g}_j (\theta) = \frac{1}{n_Q} \sum \limits^{n_Q}_{i=1} \candidate (Z_i^Q)\phi_j(Z_i^Q) - \frac{1}{n_P} \sum \limits^{n_P}_{i=1}\phi_j (Z_i^P), \; j = 1,...,m.
\end{equation}
The empirical augmented Lagrangian objective is then 
\begin{equation}
  \hat{\mathbb{L}}_{AL}[\theta, \lambda, \bm{\mu}] = \hat{\mathbb{J}} (\theta) +\lambda\hat{g}_0(\theta) + \frac{\eta_0}{2} \hat{g}_0 (\theta)^2 + \sum \limits^m_{j=1} \mu_j \hat{g}_j (\theta) + \frac{\eta_j}{2} \hat{g}_j (\theta)^2,
\end{equation}
where
\begin{equation}
  \hat{\mathbb{J}} (\theta) = \frac{1}{2 n_Q}\sum \limits^{n_Q}_{i=1} \candidate (Z_i^Q)^2 -  \frac{1}{n_P} \sum \limits^{n_P}_{i=1} \candidate (Z_i^P).
\end{equation}
Training consists of stochastic or minibatch descent on $\theta$ for all $\hat{\mathbb{L}}_{AL}$, together with the empirical dual-ascent updates
\begin{equation}
  \lambda \longleftarrow \lambda + \eta_0 \hat{g}_0 (\theta), \quad \mu_j \longleftarrow \mu_j + \eta_j \hat{g}_j(\theta), \; j=1,...,m.
\end{equation}
This yields a simultaneous primal-dual procedure in which ratio-fitting and constraint enforcement are coupled at every iteration.

\paragraph{} The augmented Lagrangian terms do not merely regularize the optimization numerically. They enforce the transport identities that characterize a calibrated reweighting. Through Proposition \ref{prop:p_theta_constraint}, satisfaction of these constraints implies agreement of $P_\theta$ with $P$ on the entire linear span of the selected test functions. Thus, the augmented Lagrangian method should be understood as mechanism for enforcing calibration of the induced change of measure, rather than simply as a device for reducing the training loss.

\subsection{Other numerical implementation aspects}

\paragraph{} The constrained formulation above describes how calibration identities may be enforced during training through an augmented Lagrangian mechanism. In practice, however, ratio learning also involves transformations applied either to improve numerical stability or to correct residual normalization error after training. The most common are post-hoc normalization, clipping, and tempering. These transformations play different mathematical roles. Some preserve, at least approximately, the change of measure structure required of a valid density-ratio, whereas others primarily act as optimization or variance-reduction devices and should not be interpreted as preserving the Radon-Nikodym identities. It is therefore useful to distinguish these cases explicitly.

\paragraph{} The post-hoc normalization consists on a simple normalization correction that is obtained by rescaling $\candidate$ by its $Q$-mean. At the population level, define
\begin{equation}
  \tilde{\candidate} (Z) = \frac{\candidate (Z)}{\expectation{Q}{\candidate (Z)}},
\end{equation}
whenever $0 < \expectation{Q}{\candidate (Z)} < \infty$. Then,
\begin{equation}
  \expectation{Q}{\tilde{\candidate} (Z)} = \frac{\expectation{Q}{\candidate (Z)}}{\expectation{Q}{\candidate (Z)}} = 1.
\end{equation}
Thus, post-hoc normalization enforces the constraints exactly. In particular, if $\candidate \geq 0$, then $\tilde{\candidate} \geq 0$ as well, and Proposition \ref{prop:change_of_measure} implies that $\tilde{\candidate}$ induces a probability measure
\begin{equation}
  \tilde{P}_\theta (A) = \int \limits_A \tilde{\candidate} (z) dQ(z), \; A \in \mathcal{A}.
\end{equation}

\paragraph{} The limitations of this transformation is that it only rescales the ratio globally. It does not alter the shape of $\candidate$, and therefore cannot correct miscalibration beyond normalization itself. In particular, if the original ratio fails to transport selected moments correctly, then post-hoc normalization does not in general enforce
\begin{equation}
  \expectation{Q}{\tilde{\candidate}(Z) \phi (Z)} = \expectation{P}{\phi (Z)}.
\end{equation}
Accordingly, post-hoc normalization is best understood as a one-dimensional correction of total mass, not as a substitute for full calibration.

\paragraph{} At the empirical level, given source observations $Z_1^Q,...,Z_{n_Q}^Q \sim Q$, define
\begin{equation}
  \hat{\tilde{r}}_\theta (Z) = \frac{\candidate (Z)}{1/n_Q \sum^{n_Q}_{i=1} \candidate (Z_i^Q)}.
\end{equation}
Then, the empirical normalization identity
\begin{equation}
  \frac{1}{n_Q} \sum \limits^{n_Q}_{i=1} \hat{\tilde{r}}_\theta (Z_i^Q) = 1,
\end{equation}
holds exactly on the calibration sample. This provides a simple post-training normalization step when exact population normalization is not enforced during fitting.

\paragraph{} A different stabilization device is clipping. For a threshold $c>0$, define
\begin{equation}
  \candidate^c(Z) = \min \{ \candidate (Z), c \}.
\end{equation}
This truncates large weights and thereby reduces the influence of rare samples carrying excessive mass. Since
\begin{equation}
  0 \leq \candidate^c(Z) \leq \candidate (Z),
\end{equation}
for all $Z \in \mathcal{Z}$, clipping directly reduces the second moment, and hence improves variance behavior in the sense of Proposition \ref{prop:second_moment_control}. It is therefore a natural device for controlling instability induced by heavy-tailed weights.

\paragraph{} From the measure-theoretic point of view, clipping preserve nonnegativity, so the reweighted set function 
\begin{equation}
  P_\theta^c (A) = \int \limits_A \candidate^c(z)dQ (z), \; A \in \mathcal{A}
\end{equation}
is still well-defined. However, the normalization identity is in general no longer exact
\begin{equation}
  \expectation{Q}{\candidate^c(Z)} \leq \expectation{Q}{\candidate (Z)}
\end{equation}
with strict inequality whenever clipping occurs on a set of positive $Q$-measure. Thus, clipping typically removes mass from the tail and produces a sub-normalized ratio unless followed by an additional rescaling step. If one defines instead
\begin{equation}
  \tilde{r}^c_\theta (Z) = \frac{\candidate^c (Z)}{\expectation{Q}{\candidate^c(Z)}},
\end{equation} 
then again $\candidate^c$ satisfies $\expectation{Q}{\candidate^c (Z)}=1$ and induces a probability measure.

\paragraph{} Clipping may still be viewed as approximately identity preserving in a structural sense. It does not change the sign of the ratio, and after renormalization it continues to define a legitimate change of measure. What it changes is the distribution of mass, especially in the upper tail. Thus, clipping introduces a bias variance trade-off. It may improve stability substantially, but at the price of modifying the original density-ratio representation.

\paragraph{} Tempering is another common transformation, defined for $0 < \beta \leq 1$ by:
\begin{equation}
  \candidate^\beta (Z) = \candidate (Z)^\beta.
\end{equation}
When $\beta < 1$, this compress large weights and reduces weight dispersion. In heavy-tailed regimes, the effect on variance and ESS may be substantial. Tempering is therefore attractive as a numerical stabilizer. 

\paragraph{} In the mathematical aspect, however, tempering differs fundamentally from that post-hoc normalization or clipping. In general, $\candidate^\beta$ is not a density-ratio corresponding to the same target law. Even if $\candidate = \trueratio$, one does not have $\expectation{Q}{\candidate^\beta (Z)} = 1$. Indeed, since $x \mapsto x^\beta$ is concave on $\mathbb{R}_+$ for $0 < \beta < 1$, Jensen Inequality yields:
\begin{equation}
  \expectation{Q}{\trueratio (Z)^\beta} \leq (\expectation{Q}{\trueratio})^\beta \leq 1,
\end{equation}
with equality only on degenerate cases. Thus, the transformed ratio fails, in general, to satisfy the defining normalization identity of a Radon-Nikodym derivative. The same issue propagates to the moment constraint.

\paragraph{} Tempering should not be interpreted as preserving the change-of-measure structure of the learned ratio. Rather, it should be viewed as an optimization or variance reduction device applied to stabilize the fitting problem. If used, it is mathematically more coherent to regard the tempered ratio as an auxiliary object for training or numerical control, rather than as the final deployed ratio intended to represent $\trueratio$.

\section{PAC-Bayes}
\label{sec:pacbayes}

\paragraph{} A constrained density-ratio network produces a calibrated approximation $\candidate$ of the Radon-Nikodym derivative, but the central decomposition of Section~\ref{sec:foundation} remains a statement at the level of populations. To make the framework actionable on a finite source sample, the gap between the population weighted risk $R_{\candidate}(\rho)$ and its importance-weighted empirical counterpart $\weightedrisk{\theta}(\rho)$ has to be controlled by a high-probability inequality, valid simultaneously over a sufficiently rich class of posteriors $\rho$. PAC-Bayes is the natural learning-theoretic machinery for this task. It shifts the perspective from a single data-dependent hypothesis to a distribution over hypotheses, and it pays a Kullback-Leibler complexity term to absorb the freedom of choosing that distribution after seeing the source sample. The role of the present section is to instantiate PAC-Bayes on the DRN-induced weighted risk, to identify the closed-form posterior that the resulting variational problem selects, and to extend the resulting fixed-time certificate to a time-uniform one.

\paragraph{} Two structural choices organize the development. First, the PAC-Bayes machinery is applied to the weighted source risk $R_{\candidate}(\rho)$ rather than directly to the target risk $R_P(\rho)$. The covariate-only ratio $\candidate$ certifies the weighted source quantity, and the unweighted target risk is bounded by adding the ratio-bias term from the central decomposition. This separation keeps the PAC-Bayes inequality classical and isolates every source of error attributable to ratio miscalibration in a single explicit term. Second, the certificate is developed in two regimes. The fixed-time regime treats the sample size $t$ as a deterministic constant fixed before the data is observed, and yields the standard square-root and Bernoulli-KL bounds for the weighted risk, the corresponding Kullback-Leibler-regularized training objective, the DRN-weighted Gibbs posterior that uniquely minimizes it, and a stability statement that controls how perturbations of $\candidate$ propagate to the optimal posterior. The anytime regime drops the assumption that $t$ is fixed in advance and produces a single high-probability event on which the same bound holds simultaneously for every sample size $t \ge t_{\min}$, through a geometric-peeling construction over an unbounded sequence of epochs. The fixed-time bound is the object that the framework reports as a guarantee at a deployment sample size, while the anytime bound is the object that justifies inspecting the certificate during training, stopping training once the certificate becomes small enough, and selecting a checkpoint from a sequence of candidates without inflating the confidence level by post-hoc multiplicity correction. The mathematical formulation develops the fixed-time apparatus, and Section~\ref{sec:anytime} establishes the time-uniform extension and recombines it with the DRN bias decomposition to obtain an anytime target-risk bound. Both regimes are evaluated empirically in Section~\ref{sec:numerical}.

\subsection{Mathematical formulation}

\paragraph{} The DRN construction defines a weights learning problem in the data space. Source observations $Z_1,...,Z_t \sim Q$ are reweighted by the learned ration $\candidate$, producing the weighted empirical risk $\weightedrisk{\theta} (h)$ for each fixed hypothesis $h \in \mathcal{H}$. The PAC-Bayes layer changes the perspective from a single deterministic hypothesis to a distribution over hypotheses. This is the natural setting for deriving high probability bounds that hold uniformly over data-dependent learning rules. 

\paragraph{} Let $(\mathcal{H}, \mathcal{G})$ be a measurable hypothesis space, where $\mathcal{G}$ is a $\sigma$-algebra on $\mathcal{H}$. A prior is a probability measure $\pi$ on $(\mathcal{H}, \mathcal{G})$. The prior is chosen independently of the observed sample. After observing the data, the learning algorithm may choose a posterior probability measure $\rho$ on $(\mathcal{H}, \mathcal{G})$, typically satisfying $\rho \ll \pi$. The posterior may depend on the sample, whereas the prior is fixed before seeing the sample. 

\paragraph{}  The posterior $\rho$ defines a randomized predictor. One draws $h \sim \rho$ and predicts using the sample hypothesis. The target risk of this randomized predictor is 
\begin{equation}
  \label{eq:target_risk_randomized_predictor}
  R_P(\rho) = \expectation{h\sim \rho}{R_P(h)} = \int \limits_{\mathcal{H}} R_P(h)d \rho (h),
\end{equation} 
whenever this integral is well-defined. Since $R_P(h) = \expectation{P}{\mathcal{L}(h,Z)}$, equation (\ref{eq:target_risk_randomized_predictor}) ca be written as:
\begin{equation}
  R_P(\rho) = \int \limits_{\mathcal{H}} \int \limits_{\mathcal{Z}} \mathcal{L}(h,z)dP(z) d\rho (h).
\end{equation}
Similarly, for any candidate ratio $r: \mathcal{Z} \longrightarrow \mathbb{R}_+$, define the posterior averaged weighted population risk:
\begin{equation}
  R_r (\rho) = \posteriorexp{R_r(h)} = \int \limits_{\mathcal{H}}R_r(h) d\rho (h).
\end{equation}
Equivalently
\begin{equation}
  R_r (\rho) = \int \limits_{\mathcal{H}} \int \limits_{\mathcal{Z}} \mathcal{L}(h,z)dQ (z) d \rho (h).
\end{equation}
In particular, for the learned ratio $\candidate$:
\begin{equation}
  \label{eq:posterior_risk_learned_ratio}
  R_{\candidate} (\rho) = \posteriorexp{\expectation{Q}{\candidate (Z) \mathcal{L}(h,Z)}}.
\end{equation}
Given the source observations $Z_1,...,Z_t \sim Q$, the posterior averaged weighted empirical risk is:
\begin{equation}
  \weightedrisk{\theta}(\rho) = \posteriorexp{\weightedrisk{\theta}(h)} = \int \limits_{\mathcal{H}} \weightedrisk{\theta} (h) d\rho (h).
\end{equation}
Using the definition of $\weightedrisk{\theta}$, this becomes:
\begin{equation}
  \label{eq:posterior_empirical_risk_learned_ratio}
  \weightedrisk{\theta} (\rho) = \int \limits_{\mathcal{H}} \left( \frac{1}{t} \sum \limits_{i=1}^t \candidate (Z_i) \mathcal{L}(h,Z_i) \right)d \rho (h) = \frac{1}{t} \sum \limits^t_{i=1}\candidate (Z_i) \posteriorexp{\mathcal{L}(h,Z_i)}
\end{equation}

\paragraph{} The complexity of the posterior relative to the prior is measured by the Kullback-Leibler divergence:
\begin{equation}
  \mathbb{KL}(\rho || \pi) = \int \limits_{\mathcal{H}} \log_2 \left( \frac{d\rho}{d \pi} (h)\right)d \rho (h)
\end{equation}
when $\rho \ll \pi$, and
\begin{equation}
  \mathbb{KL}(\rho || \pi) = +\infty
\end{equation}
otherwise. This divergence is the price, in PAC-Bayes theory, of choosing a data-dependent posterior $\rho$ rather than retaining the prior $\pi$. The PAC-Bayes bound will therefore be applied to the weighted loss induced by $\candidate$. Under appropriate boundedness or scaling assumptions, it provides a high-probability upper bound on $R_{\candidate}(\rho)$ in terms of $\weightedrisk{\theta}(\rho)$, $\mathbb{KL}(\rho || \pi)$, $t$ and $\delta > 0$. The result bound is then combined with the DRN bias term to obtain a bound on the true target risk $R_P(\rho)$. This distinction between the intermediate weighted risk and the final target risk is essential for the structure of the theory.

\paragraph{} It was introduced the posterior distribution over hypotheses, and it was identified the quantity controlled by the PAC-Bayes layer. That is the between the weighted population risk $R_{\candidate}(\rho)$ and the weighted empirical risk $\weightedrisk{\theta} (\rho)$. To apply a standard PAC-Bayes theorem, it is useful to rewrite this weighted learning problem as an ordinary learning problem with a modified loss. The only change is that the original loss $\mathcal{L}(h,Z)$ is replaced by the DRN-induced weighted loss:
\begin{equation}
  \label{eq:drn_induced_loss}
  \tilde{\mathcal{L}}_\theta (h,Z) = \candidate (Z) \mathcal{L}(h,Z).
\end{equation}
Once this substitution is made, the PAC-Bayes mechanism applies to the pair consisting of the source law $Q$ and the weighted loss $\tilde{\mathcal{L}}(h,Z)$. The following lemma is purely notational, but it is the entry point for PAC-Bayes.

\begin{lemma}\label{lemma:induced_loss}
  Define $\tilde{\mathcal{L}}(h,Z)$ as in (\ref{eq:drn_induced_loss}). Then, for every posterior $\rho$
  \begin{equation}
    R_{\candidate} (\rho) = \posteriorexp{\expectation{Q}{\tilde{\mathcal{L}}(h,Z)}},
  \end{equation}
  and for $Z_1,...,Z_t \sim Q$,
  \begin{equation}
    \weightedrisk{\theta}(\rho) = \posteriorexp{\frac{1}{t} \sum \limits^t_{i=1}\tilde{\mathcal{L}}(h,Z)}.
  \end{equation}
\end{lemma}

\begin{proof}
  This follows immediately from (\ref{eq:posterior_risk_learned_ratio}), (\ref{eq:posterior_empirical_risk_learned_ratio}), and (\ref{eq:drn_induced_loss}).
\end{proof}

\paragraph{} The theorem below is a fixed-time PAC-Bayes guarantee for this weighted risk. It is a classical PAC-Bayes bound applied after the weighted loss substitution. Recall that its role in the present framework is to control the generalization gap $R_{\candidate}(\rho) - \weightedrisk{\theta} (\rho)$.

\begin{theo}\label{theo:adapted_mcallester}
  Let $(\mathcal{H}, \mathcal{G})$ be a measurable hypothesis space, and let $\pi$ be a prior probability measure on $(\mathcal{H}, \mathcal{G})$, chosen independently of the source sample. Let $Z_1, ..., Z_t \stackrel{\mathrm{iid}}{\sim} Q $. Let $\candidate: \mathcal{Z}\longrightarrow \mathbb{R}_+$ be fixed independently of $Z_1,...,Z_t$, or condition on any independent data used to construct $\candidate$. Define $\tilde{\mathcal{L}}_\theta$ as in Lemma \ref{lemma:induced_loss}. Assume that:
  \begin{equation}
    0 \leq \tilde{\mathcal{L}}_\theta (h,Z) \leq 1,
  \end{equation}
  for all $(h,Z) \in \mathcal{H} \times \mathcal{Z}$. Then, for any $\delta \in (0,1)$, with probability at least $1-\delta$ over the draw of $Z_1, ..., Z_t \stackrel{\mathrm{iid}}{\sim} Q$, the following inequality holds simultaneously for every posterior probability measure $\rho \ll \pi$:
  \begin{equation}
    R_{\candidate} (\rho)  \leq \weightedrisk{\theta} (\rho) + \sqrt{\frac{\mathbb{KL}(\rho || \pi)+\ln (1/\delta) + \ln t+ 2}{2t - 1}}.
  \end{equation}
\end{theo}

\begin{proof}
  The proof consists on applying Lemma \ref{lemma:induced_loss} of this work and Theorem 1 from \cite{mcallester1999pacbayesian}.
\end{proof}

\paragraph{} The first fixed-time PAC-Bayes theorem used here is a specialization of the classical PAC-Bayesian model averaging principle introduced in \cite{mcallester1999pacbayesian}. In this cited paper, one considers a prior distribution over hypotheses and derives a high probability bound on the population loss of any posterior-weighted mixture in terms of its empirical loss and Kullback-Leibler divergence from the posterior to the prior. The paper explicitly present PAC-Bayes as a combination of Bayesian style priors with distribution-free PAC guarantees, and extends the model selection to nonuniform model averaging. The present theorem uses the same PAC-Bayes mechanism, but applies it to the DRN-induced weighted loss $\hat{\mathcal{L}}_\theta$ under the source law $Q$. Consequently, the theorem does not directly bound the target risk $R_P$. It bound the intermediate risk $R_{\candidate}$, while the discrepancy $R_P(\rho)-R_{\candidate}(\rho)$ is handled separately by the density-ratio estimation and calibration analysis. Thus, the PAC-Bayes inequality itself is classical, whereas its rate here is specialized to the DRN framework and integrated with an explicit ratio-bias decomposition.

\paragraph{} Theorem \ref{theo:adapted_mcallester} provides a fixed time PAC-Bayes upper bound on the intermediate weighted risk $R_{\candidate}$. However, the ultimate object of interest is the true target-risk $R_P$. The difference between these two quantities is the ratio-induced bias caused by replacing the true Radon-Nikodym derivative $\trueratio$ with the learned ratio $\candidate$. The following corollary combines the fixed-time PAC-Bayes theorem with the DRN bias-decompostion, yielding a bound on the true-risk. This is the first result in which the components of the framework appear together (empirical weighted fit, posterior, and ratio-approximation error).

\begin{corollary}\label{col:adapted_mcallester}
  Assume the conditions of Theorem \ref{theo:adapted_mcallester}. In addition, assume that $\trueratio$ exists and that, for every posterior $\rho$ under consideration, the bias term
  \begin{equation}
    \mathbb{B}_\theta (\rho) = R_P(\rho)-R_{\candidate} (\rho)
  \end{equation}
  is well-defined. Then, with probability at least $1 - \delta$ over the draw $Z_1, ..., Z_t \stackrel{\mathrm{iid}}{\sim} Q$ the following inequality holds simultaneously for every posterior $\rho \ll \pi$:
  \begin{equation}
    \label{eq:full_pac_bias_drn}
    R_P(\rho) \leq \weightedrisk{\theta} (\rho) + \mathbb{B}_\theta (\rho) + \sqrt{\frac{\mathbb{KL}(\rho || \pi)+\ln (1/\delta) + \ln t+ 2}{2t - 1}}.
  \end{equation}
  Moreover, if $|\mathcal{L}(h,Z)| \leq C_L$, $C_L > 0$, for all $(h,Z) \in \mathcal{H} \times \mathcal{Z}$, then
  \begin{equation}
    \mathbb{B}_\theta (\rho)  \leq C_L \| \trueratio - \candidate \|_{L^1(Q)}.
  \end{equation}
  If instead, $\mathcal{L}(h,\cdot) \in L^2(Q)$ for $\rho$-almost $h$, then
  \begin{equation}
    \mathbb{B}_\theta (\rho)  \leq \| \trueratio - \candidate \|_{L^2(Q)}\posteriorexp{\| \mathcal{L}(h,\cdot) \|_{L^2(Q)}}.
  \end{equation} 
\end{corollary}

\begin{proof}
  It holds true that:
  \begin{equation}
    R_P(\rho) = R_{\candidate} (\rho) + (R_P(\rho) - R_{\candidate}(\rho)).
  \end{equation}
  By the bias definition,
  \begin{equation}
    \label{eq:aux_bias_intermediary_passage}
    R_P(\rho) = R_{\candidate} (\rho) + \mathbb{B}_\theta (\rho).
  \end{equation}
  On the high-probability event of Theorem \ref{theo:adapted_mcallester}, simultaneously for all $\rho \ll \pi$:
  \begin{equation}
    \label{eq:aux_high_prob_event}
    R_{\candidate} (\rho) \leq \weightedrisk{\theta} (\rho) + \mathbb{B}_\theta (\rho) + \sqrt{\frac{\mathbb{KL}(\rho || \pi)+\ln (1/\delta) + \ln t+ 2}{2t - 1}}.
  \end{equation}
  From (\ref{eq:aux_bias_intermediary_passage}) and (\ref{eq:aux_high_prob_event}), inequality (\ref{eq:full_pac_bias_drn}) is achieved.

  \paragraph{} By Proposition \ref{prop:bias_control_via_ratio_approx}:
  \begin{equation}
    \begin{split}
      \mathbb{B}_\theta (\rho) = |\posteriorexp{R_P(h) - R_{\candidate}(h)}| &\leq \posteriorexp{|R_P(h)- R_{\candidate}(h)|} \\
      & \leq \posteriorexp{C_L \| \trueratio - \candidate \|_{L^1(Q)}} \\
      & = C_L \| \trueratio - \candidate \|_{L^1(Q)}.
    \end{split}
  \end{equation}

  \paragraph{} For the $L^2(Q)$ estimation, again applying Proposition \ref{prop:bias_control_via_ratio_approx}:
  \begin{equation}
    \begin{split}
      \mathbb{B}_\theta (\rho) = |\posteriorexp{R_P(h) - R_{\candidate}(h)}|  &\leq \posteriorexp{|R_P(h)- R_{\candidate}(h)|} \\
      & \leq \posteriorexp{\| \trueratio - \candidate \|_{L^2(Q)} \| \mathcal{L}(h,\cdot) \|_{L^2(Q)}}\\
      &= \| \trueratio - \candidate \|_{L^2(Q)} \, \posteriorexp{\| \mathcal{L}(h,\cdot) \|_{L^2(Q)}}.
    \end{split}
  \end{equation}
\end{proof}

\paragraph{} The reduction of a square-root PAC-Bayes bound to a Kullback-Leibler-regularized objective via Young's Inequality
\begin{equation}
  \label{eq:young_inequality}
  \sqrt{x} \leq \eta_{\text{PB}}\, x + \frac{1}{4 \eta_{\text{PB}}}, \; x>0, \; \eta_{\text{PB}} >0
\end{equation}
is standard in the PAC-Bayes literature. It is closely related to the optimization-oriented PAC-Bayes bounds developed by \cite{keshet2011pacbayesian}, and to the PAC-Bayes-$\lambda$ formulation of \cite{thiemann2017strongly}, where the empirical-risk/complexity trade-off is made explicit through a tunable parameter. It is also consistent with \cite{catoni2007pacbayesian} formulation of PAC-Bayesian learning and with the bound-minimization perspective developed by \cite{germain2009pacbayes}.

\paragraph{} In the present framework, this standard linearization is restated in the DRN-induced weighted setting. The fixed-time bound of Theorem \ref{theo:adapted_mcallester} controls the intermediate risk $R_{\candidate}$ through the weighted empirical risk $\weightedrisk{\theta}$ and the complexity term $\mathbb{KL}(\rho || \pi)$. Linearizing the square-root penalty yields the trainable surrogate $\weightedrisk{\theta}(\rho) + \lambda_t \mathbb{KL}(\rho || \pi)$ for a constant $\lambda_t > 0$. The only modification relative to the classical PAC-Bayes training objective is the substitution of the ordinary empirical risk by the DRN-corrected weighted empirical risk $\weightedrisk{\theta}$. The linearization step itself is unchanged. The genuinely new content of the framework lies in how this weighted PAC-Bayes objective is coupled to the DRN bias decomposition and, in the subsequent analysis, how its minimizer is identified as DRN-weighted posterior.

\begin{proposition}
  Assume the conditions of Theorem \ref{theo:adapted_mcallester}. Let 
  \begin{equation}
    C_t(\delta) = \ln (1/\delta) + \ln t + 2.
  \end{equation}
  Then, for every $\eta_{\text{PB}} > 0$, with probability at least $1-\delta$, simultaneously for all posterior $\rho \ll \pi$
  \begin{equation}
    \label{eq:pac_bayes_linearized}
    R_{\candidate} (\rho) \leq \weightedrisk{\theta} (\rho) + \frac{\eta_{\text{PB}}}{2t -1} \mathbb{KL}(\rho || \pi) + \frac{\eta_{\text{PB}}}{2t-1} C_t + \frac{1}{4 \eta_{\text{PB}}}.
  \end{equation}
  Consequently, for fixed $t$, $\delta$, $\eta_{\text{PB}}$, minimizing the upper bound over $\rho$ is equivalent, up to a constant independent of $\rho$ to minimizing the training objective 
  \begin{equation}
    \mathbb{J}_{\text{PB}} = \weightedrisk{\theta} (\rho) + \lambda_t \mathbb{KL}(\rho || \pi),
  \end{equation}
  where
  \begin{equation}
    \lambda_t = \frac{\eta_{\text{PB}}}{2t -1}.
  \end{equation}
\end{proposition}

\begin{proof}
  By Theorem \ref{theo:adapted_mcallester}, with probability at least $1-\delta$, simultaneously for every posterior $\rho \ll \pi$,
  \begin{equation}
    R_{\candidate} (\rho) \leq \weightedrisk{\theta} (\rho) + \sqrt{\frac{\mathbb{KL}(\rho || \pi)+ C_t(\delta)}{2t-1}}.
  \end{equation}
  It remains to upper-bound the square root penalty by a linear function of $\mathbb{KL}(\rho || \pi)$.

  \paragraph{} For any $x \geq 0$ and for any $\eta_{\text{PB}} > 0$, the inequality (\ref{eq:young_inequality}) holds. Applying it with 
  \begin{equation}
    x = \frac{\mathbb{KL}(\rho || \pi) + C_t(\delta)}{2t-1}
  \end{equation}
  it is possible to obtain 
  \begin{equation}
    \sqrt{\frac{\mathbb{KL}(\rho || \pi) + C_t(\delta)}{2t-1}} \leq \eta_{\text{PB}} \frac{\mathbb{KL}(\rho || \pi) + C_t(\delta)}{2t-1} + \frac{1}{4 \eta_{\text{PB}}}.
  \end{equation}
  In this way, inequality (\ref{eq:pac_bayes_linearized}) is achieved.

  \paragraph{} Now, fix $t$, $\delta$, and $\eta_{\text{PB}}$. The terms
  \begin{equation}
    \frac{\eta_{\text{PB}}}{2t-1} C_t(\delta) \frac{1}{4 \eta_{\text{PB}}}
  \end{equation}
  do not depend on $\rho$. Therefore, when optimizing over $\rho$, they do not affect the minimizer. Hence, minimizing the surrogate upper bound is equivalent to minimizing
  \begin{equation}
    \weightedrisk{\theta} (\rho) + \frac{\eta_{\text{PB}}}{2t-1} \mathbb{KL}(\rho || \pi ).
  \end{equation}
  By defining
  \begin{equation}
    \lambda_t = \frac{\eta_{\text{PB}}}{2t-1}
  \end{equation}
  the objective $\mathbb{J}_{\text{PB}}$ is obtained.
\end{proof}

\paragraph{} The objective presented in the last proposition is classical in PAC-Bayes learning. The novelty in this setting is not the variational principle itself, but the objective on which it acts. The empirical risk inside the lift is the DRN-corrected weighted empirical risk $\weightedrisk{\theta}$ so hypotheses are exponentially reweighted according to their estimated target-distribution performance, as measured though the learned ratio $\candidate$. The framework therefore performs a second change of measure. The DRN changes measure on the data space ($dP \approx \candidate dQ$), and the Gibbs posterior changes measure on the hypothesis space. In this way, the next theorem identifies the unique minimizer of $\mathbb{J}_{\text{PB}}$ in closed form. This is called the DRN-weighted Gibbs posterior. 

\begin{defin}[Gibbs posterior]
  Let $(\mathcal{H}, \mathcal{G}, \pi)$ be a measurable hypothesis space equipped with a prior probability measure $\pi$, and let $\Phi:\mathcal{H} \longrightarrow \mathbb{R}$ be a measurable function (typically interpreted as an energy or risk functional on hypotheses) such that
  \begin{equation}
    \mathbb{E}_{h \sim \pi} \left[ e^{-\beta\Phi (h)} \right]  < \infty, \; \beta >0.
  \end{equation} 
  The Gibbs posterior at inverse temperature $\beta$ associated with $\Phi $ and prior $\pi$, denoted $\rho^*_\beta$ is the probability measure on $(\mathcal{H}, \mathcal{G})$ defined by its Radon-Nikodym derivative with respect to $\pi$:
  \begin{equation}
    \frac{d\rho^*_\beta}{d \pi} (h) = \frac{e^{-\beta\Phi (h)}}{Z_\beta}, 
  \end{equation}
  with
  \begin{equation}
    Z_\beta = \mathbb{E}_{h \sim \pi} \left[ e^{-\beta\Phi (h)} \right].
  \end{equation}
  The quantity $Z_\beta$ is sometimes called free energy.
\end{defin}

\paragraph{} The proof of the next theorem rests on a single classical identity, which is stated next.

\begin{lemma}\label{lemma:gibbs_posterior}
  Let $\pi$ be a probability measure on $(\mathcal{H}, \mathcal{G})$, and let $f: \mathcal{H} \longrightarrow \mathbb{R}$ be a measurable function such that
  \begin{equation}
    \mathbb{E}_{h \sim \pi} \left[ e^{-\beta \Phi (h)} \right] < \infty.
  \end{equation}
  Then, 
  \begin{equation}
    \inf_{\rho \ll \pi} \{ \posteriorexp{\beta \Phi (h)} + \mathbb{KL}(\rho || \pi) \} = - \ln \expectation{h \sim \pi}{e^{-\beta \Phi (h)}},
  \end{equation}
  and the infimum is attained uniquely at the Gibbs measure
  \begin{equation}
    \label{eq:gibbs_posterior}
    \frac{d \rho^*}{d \pi} (h) = \frac{e^{-\beta \Phi (h)}}{\expectation{h' \sim \pi}{e^{-\beta \Phi (h')}}}.
  \end{equation}
  Moreover, for every $\rho \ll \pi$ with $\posteriorexp{\beta \Phi (h)}$ finite,
  \begin{equation}
    \label{eq:gibbs_posterior_difference}
    [\posteriorexp{\beta \Phi (h)} + \mathbb{KL}(\rho || \pi)] - [\expectation{h \sim \rho^*}{\beta \Phi (h)} + \mathbb{KL}(\rho^* || \pi)] = \mathbb{KL} (\rho ||  \rho^*).
  \end{equation}
  In particular, $\rho^*$ is the unique minimizer.
\end{lemma}

\begin{theo}[DRN-weighted Gibbs posterior]\label{theo:drn_weighted_gibbs_posterior}
  Assume the conditions of Theorem \ref{theo:adapted_mcallester}, that $\hat{\mathcal{L}}_\theta \in [0,1]$ for all $(h,Z) \in \mathcal{H} \times \mathcal{Z}$, that the prior $\pi$ is a proper probability measure on $(\mathcal{H}, \mathcal{G})$, and that $h \mapsto \weightedrisk{\theta} (h)$ is $\mathcal{G}$-measurable. Fix any $\lambda_t > 0$. Then, the functional 
  \begin{equation}
    \mathbb{J}_{\text{PB}} (\rho) = \weightedrisk{\theta} (\rho) + \lambda_t \mathbb{KL}_{\text{PB}} (\rho || \pi)
  \end{equation}
  attains its unique minimum over the set $\{ \rho: \; \rho \ll \pi \}$ at the DRN-weighted Gibbs posterior $\rho^*_{\lambda_t, \theta}$ defined by:
  \begin{equation}
    \label{eq:drn_gibbs_posterior}
    \frac{d \rho^*_{\lambda_t, \theta}}{d \pi} (h) = \frac{e^{-\lambda_t^{-1}\weightedrisk{\theta}(h)}}{Z_{\lambda_t, \theta}},
  \end{equation}
  with
  \begin{equation}
    Z_{\lambda_t, \theta} = \expectation{h \sim \pi}{e^{-\lambda_t^{-1}\weightedrisk{\theta}(h)}}.
  \end{equation}
  The minimum value is
  \begin{equation}
    \mathbb{J}_{\text{PB}} (\rho^*_{\lambda_t, \theta}) = -\lambda_t \ln Z_{\lambda_t, \theta},
  \end{equation}
  and for every other posterior $\rho \ll \pi$ with $\weightedrisk{\theta} (\rho) < \infty$,
  \begin{equation}
    \mathbb{J}_{\text{PB}} (\rho) - \mathbb{J}_{\text{PB}} (\rho^*_{\lambda_t, \theta}) = \lambda_t \mathbb{KL}(\rho || \rho^*_{\lambda_t, \theta}) \geq 0.
  \end{equation}
\end{theo}

\begin{proof}
  The strategy to prove this theorem is to recognize $\mathbb{J}_{\text{PB}}$ as a rescaled instance of the variational functional of Lemma \ref{lemma:gibbs_posterior}, then apply the lemma itself. 

  \paragraph{} Since $\lambda_t > 0$, dividing $\mathbb{J}_{\text{PB}}$ by it does not change its minimizer:
  \begin{equation}
    \frac{1}{\lambda_t} \mathbb{J}_{\text{PB}} = \frac{\weightedrisk{\theta}(\rho)}{\lambda_t} + \mathbb{KL}(\rho || \pi) = \posteriorexp{\frac{\weightedrisk{\theta}(h)}{\lambda_t}} + \mathbb{KL}(\rho || \pi).
  \end{equation}
  Define $\Phi: \mathcal{H} \longrightarrow \mathbb{R}$ by $\Phi (h) = \weightedrisk{\theta} (h)$ and set $\beta = 1/\lambda_t$. The function $\beta \Phi$ is measurable (as a deterministic function of the measure map $h \mapsto \weightedrisk{\theta}(h)$), and the objective becomes:
  \begin{equation}
    \frac{1}{\lambda_t} \mathbb{J}_{\text{PB}} (\rho) = \posteriorexp{\beta \Phi (h)} + \mathbb{KL} (\rho || \pi),
  \end{equation}
  which is exactly the functional Lemma \ref{lemma:gibbs_posterior} minimizes.

  \paragraph{} Now, it is necessary to check that $\posteriorexp{e^{-\beta \Phi (h)}}<\infty$. By assumption, $\tilde{\mathcal{L}}_\theta(h,Z) \in [0,1]$ for every $(h,Z)$, so for every $h$, $\weightedrisk{\theta}(h) \in [0,1]$. Therefore, $\beta \Phi (h) \in [0, 1/\lambda_t]$, and
  \begin{equation}
    e^{-\lambda_t} \leq e^{-\beta \Phi (h)} \leq 1,
  \end{equation}
  for every $h \in \mathcal{H}$. Integrating against $\pi$,
  \begin{equation}
    e^{-\lambda_t} \leq \expectation{h \sim \pi}{e^{-\beta \Phi (h)}} \leq 1
  \end{equation}
  so the integrability condition is satisfied. In fact, $Z_{\lambda_t, \theta}$ is bounded away from $0$ and $\infty$ uniformly in data.
\end{proof}

\paragraph{} By Lemma \ref{lemma:gibbs_posterior}, the function $\rho \mapsto \posteriorexp{\beta \Phi (h)} + \mathbb{KL}(\rho || \pi)$ attains its unique minimum over $\{ \rho: \; \rho \ll \pi \}$ at the Gibbs measure:
\begin{equation}
  \frac{d \rho^*}{d \pi} (h) = \frac{e^{-\beta \Phi (h)}}{\expectation{h' \sim \pi}{e^{-\beta \Phi (h')}}} = \frac{e^{-\lambda_t^{-1} \weightedrisk{\theta}(h)}}{Z_{\lambda_t,\theta}},
\end{equation}
with minimum value $-\ln Z_{\lambda_t,\theta}$. This is exactly the claimed $\rho^*_{\lambda_t, \theta}$. Multiplying through by $\lambda_t > 0$, the unique minimizer of $\mathbb{J}_{\text{PB}}$ is $\rho^*_{\lambda_t,\theta}$ with minimum value:
\begin{equation}
  \mathbb{J}_{\text{PB}} (\rho^*_{\lambda_t, \theta}) = - \lambda_t \ln Z_{\lambda_t, \theta}.
\end{equation}

\paragraph{} Applying equation (\ref{eq:gibbs_posterior_difference}) from Lemma \ref{lemma:gibbs_posterior} with $\beta \Phi (h) = \weightedrisk{\theta} (h) / \lambda_t$, then multiplying both sides by $\lambda_t$:
\begin{equation}
  \mathbb{J}_{\text{PB}} (\rho) -\mathbb{J}_{\text{PB}} (\rho^*_{\lambda_t, \theta}) = \lambda_t \mathbb{KL} (\rho || \rho^*_{\lambda_t, \theta}).
\end{equation}

\begin{remark}
  The exponential tilt is governed by the importance-weighted empirical risk $\weightedrisk{\theta}(h)$ of~\eqref{eq:intro_weighted_empirical_risk} rather than by the raw source-sample empirical risk, and hypotheses are therefore down-weighted in proportion to their estimated target-distribution loss. The framework performs a change of measure at two levels simultaneously. At the data level, the target law is reweighted as $dP \approx \candidate\, dQ$, with the reweighting carried by the constrained DRN. At the hypothesis level, the posterior tilt $d\rho_{\lambda_t, \theta} \propto e^{-\lambda_t \weightedrisk{\theta}(h)}\, d\pi$ is carried by the Gibbs structure. The classical Gibbs posterior of \cite{catoni2007pacbayesian} and \cite{alquier2016variational} is recovered in the no-shift limit $\candidate \equiv 1$, in which case there is no covariate shift to correct.
\end{remark}

\begin{remark}
  Setting $\beta = \lambda_t^{-1} = (2t-1)/\eta_{\text{PB}}$, the posterior takes the canonical Gibbs form $d\rho^*_{\lambda_t, \theta} \propto e^{-\lambda_t \Phi(h)}\, d\pi$. The inverse temperature grows linearly with the source sample size $t$, so the posterior concentrates on hypotheses with low weighted risk as more source data accumulates.
\end{remark}

\paragraph{} Theorem \ref{theo:drn_weighted_gibbs_posterior} identifies the optimal posterior $\rho^*_{\lambda_t, \theta}$ as a function of the learned ratio $\candidate$. In practice, $\candidate$ itself is estimated, and one would like to know how robust the resulting posterior (and ultimately its target risk) is to perturbations in $\candidate$. The theorem below states if two model $\candidate$ and $r_{\theta '}$ are close in $L^2(Q)$, then their associated Gibbs posteriors induce target risks that are close, with an explicit dependence on the inverse temperature $\beta = \lambda_t^{-1}$. This theorem uses the concept of total variation distance (TVD) that can be found in Appendix \ref{ap:math_aux}.

\begin{theo}
  Assume the setup of Theorem \ref{theo:drn_weighted_gibbs_posterior}. Then, with probability at least $1-\delta$ over the draw $Z_1,...,Z_t \iid Q$, the following holds:
  \begin{enumerate}
    \item The Gibbs posteriors associated to $\candidate$ and $r_{\theta '}$ satisfy:
    \begin{equation}
      TV(\rho^*_{\lambda_t, \theta}, \rho^*_{\lambda_t, \theta '}) \leq \frac{\beta}{2} \sup \limits_{h \in \mathcal{H}}|\weightedrisk{\theta} (h) - \weightedrisk{\theta '}(h)|,
    \end{equation}
    $TV$ denotes the total variation distance;
    \item It holds that
    \begin{equation}
      |R_P(\rho^*_{\lambda_t, \theta}) - R_P(\rho^*_{\lambda_t, \theta '})| \leq \beta \sup \limits_{h \in \mathcal{H}}|\weightedrisk{\theta} (h) - \weightedrisk{\theta '}(h)| + 2C_{L_2} \| \candidate - r_{\theta '} \|_{L^2(Q)} + 2C_{L_2} \| \trueratio - r_{\theta '} \|_{L^2(Q)}
    \end{equation}
    with $C_{L^2}>0$;
    \item The supremum in item 1 and item 2  is itself controlled by the pointwise distance between the ratios:
    \begin{equation}
      \sup \limits_{h \in \mathcal{H}} |\weightedrisk{\theta}(h) - \weightedrisk{\theta '}(h)| \leq C \frac{1}{t} \sum \limits^{t}_{i=1} |\candidate (Z_i) - r_{\theta '}(Z_i)|,
    \end{equation} 
    with $C>0$. Considering the expectation over the source sample,
    \begin{equation}
      \expectation{Z \sim Q}{ \sup \limits_{h \in \mathcal{H}} |\weightedrisk{\theta}(h) - \weightedrisk{\theta '}(h)|} \leq C \| \candidate - r_{\theta '} \|_{L^1(Q)} \leq C \| \candidate - r_{\theta '} \|_{L^2 (Q)}.
    \end{equation}
  \end{enumerate}
\end{theo}

\begin{proof}
  In this proof, write $\Phi (h) = \weightedrisk{\theta} (h)$, $\Phi '(h) = \weightedrisk{\theta '} (h)$, and $\Delta (h) = \Phi (h) - \Phi ' (h)$. Let $\| \Delta \|_{\infty} = \sup_{h \in \mathcal{H}} |\Delta (h)|$.

  \paragraph{1.} The idea is to compute the Kullback-Leibler divergence between the two Gibbs posteriors in closed form, then convert it to a TVD bound via Pinsker's inequality (see Appendix \ref{ap:math_aux}). Using the explicit densities of $\rho^*_{\lambda_t, \theta}$ and $\rho^*_{\lambda_t, \theta '}$ with respect to $\pi$ and (\ref{eq:drn_gibbs_posterior}):
  \begin{equation}
    \ln \frac{d\rho^*_{\lambda_t, \theta}}{d\rho^*_{\lambda_t, \theta '}} (h) = \ln \frac{d \rho^*_{\lambda_t, \theta}}{d \pi} (h) - \ln \frac{\rho^*_{\lambda_t, \theta '}}{d \pi} (h) = -\beta \Phi (h) -\ln Z_{\lambda_t, \theta}  + \beta \Phi ' (h) + \ln Z_{\lambda_t, \theta '}.
  \end{equation}
  Therefore,
  \begin{equation}
    \label{eq:aux_kl_rho}
    \mathbb{KL} (\rho^*_{\lambda_t, \theta} || \rho^*_{\lambda_t, \theta'}) = \expectation{h \sim \rho^*_{\lambda_t, \theta}}{\ln \frac{d\rho^*_{\lambda_t, \theta}}{d \rho^*_{\lambda_t, \theta'}}} = -\beta \expectation{h \sim \rho^*_{\lambda_t, \theta}}{ \Delta} + \ln \frac{Z_{\lambda_t, \theta}}{Z_{\lambda_t, \theta '}}.
  \end{equation}
  Using 
  \begin{equation}
    \frac{d \rho^*_{\lambda_t, \theta}}{d \pi } (h) = \frac{e^{-\beta \Phi (h)}}{Z_{\lambda_t, \theta}},
  \end{equation}
  yields
  \begin{equation}
    \label{eq:aux_free_energy_ratio}
    \frac{Z_{\lambda_t, \theta '}}{Z_{\lambda_t, \theta}} = \frac{\expectation{h \sim \pi}{e^{-\beta \Phi(h)}}}{Z_{\lambda_t, \theta}} = \expectation{h \sim \pi }{\frac{e^{-\beta \Phi (h)}}{Z_{\lambda_t, \theta}}} = \expectation{h \sim \pi}{\frac{e^{-\beta \Phi(h)}}{Z_{\lambda_t, \theta}}\, e^{\beta (\Phi - \Phi ')(h)}}.
  \end{equation}
  Take any measurable function $g: \mathcal{H} \longrightarrow \mathbb{R}$ that is integrable under $\rho^*_{\lambda_t, \theta}$. By the Radon-Nikodym property:
  \begin{equation}
    \expectation{h \sim \rho^*_{\lambda_t, \theta}}{g(h)} = \int \limits_{\mathcal{H}} g(h)d \rho^*_{\lambda_t, \theta}.
  \end{equation}
  Since,
  \begin{equation}
    d\rho^*_{\lambda_t, \theta} (h)= \frac{\rho^*_{\lambda_t, \theta}}{d \pi}(h) d\pi (h),
  \end{equation}
  one substitutes
  \begin{equation}
    \int \limits_{\mathcal{H}} g(h) d\rho^*_{\lambda_t, \theta} = \int \limits_{\mathcal{H}} g(h) \frac{d \rho^*_{\lambda_t, \theta}}{d \pi} (h) d\pi (h).
  \end{equation}
  Using the explicit density from Gibbs posterior in (\ref{eq:gibbs_posterior}):
  \begin{equation}
    \label{eq:distribution_shift}
    \expectation{h \sim \rho^*_{\lambda_t, \theta}}{g(h)} = \int \limits_{\mathcal{H}} g(h) \frac{e^{-\beta \Phi (h)}}{Z_{\lambda_t, \theta}}d\pi (h) = \expectation{h \sim \pi}{\frac{e^{-\beta \Phi (h)}}{Z_{\lambda_t, \theta}}g(h)}.
  \end{equation}
  By (\ref{eq:aux_free_energy_ratio}) and (\ref{eq:distribution_shift}):
  \begin{equation}
    \label{eq:aux_free_energy_ratio_final}
    \frac{Z_{\lambda_t, \theta '}}{Z_{\lambda_t, \theta}} = \expectation{h \sim \rho^*_{\lambda_t, \theta}}{ e^{\beta \Delta (h)}}.
  \end{equation}
  Substituting (\ref{eq:aux_free_energy_ratio_final}) in (\ref{eq:aux_kl_rho}):
  \begin{equation}
    \label{eq:aux_kl_rho_final}
    \mathbb{KL}(\rho^*_{\lambda_t, \theta} || \rho^*_{\lambda_t, \theta'}) = -\beta \expectation{h \sim \rho^*_{\lambda_t, \theta}}{\Delta} + \ln \expectation{h \sim \rho^*_{\lambda_t, \theta}}{e^{\beta \Delta}}.
  \end{equation}

  \paragraph{} Since $\tilde{\mathcal{L}}_{\theta}, \; \tilde{\mathcal{L}}_{\theta '} \in [0,1]$, both $\Phi(h)$ and $\Phi '(h)$ lie in $[0,1]$, for every $h \in \mathcal{H}$, so $\Delta (h) \in [-1,1]$, hence
  \begin{equation}
    |\Delta (h)| \leq \| \Delta \|_{\infty} \leq 1.
  \end{equation}
  The range of $\Delta$ under $\rho^*_{\lambda_t, \theta}$ is contained in $[-\| \Delta \|_\infty, \| \Delta \|_{\infty}]$ of length $2 \| \Delta \|_{\infty}$. Therefore, by Hoeffding's lemma,
  \begin{equation}
    \label{eq:aux_app_hoeffding_lemma}
    \ln \expectation{h \sim \rho^*_{\lambda_t, \theta}}{e^{\beta \Delta}} - \beta \expectation{h \sim \rho^*_{\lambda_t, \theta}}{\Delta} \leq \frac{\beta^2 (2\| \Delta \|_\infty)}{8} = \frac{\beta^2}{2}\| \Delta \|^2_{\infty}.
  \end{equation}
  Combining (\ref{eq:aux_kl_rho_final}) and (\ref{eq:aux_app_hoeffding_lemma}):
  \begin{equation}
    \mathbb{KL} (\rho^*_{\lambda_t, \theta} || \rho^*_{\lambda_t, \theta'}) \leq \frac{\beta^2}{2} \| \Delta \|^2_{\infty}.
  \end{equation}
  Lemma \ref{lemma:pinsker_ineq} gives:
  \begin{equation}
    TV(\rho^*_{\lambda_t, \theta}, \rho^*_{\lambda_t, \theta'}) \leq \sqrt{\frac{1}{2}\mathbb{KL}(\rho^*_{\lambda_t, \theta} || \rho^*_{\lambda_t, \theta'})} \leq \frac{\beta}{2} \| \Delta \|_\infty. 
  \end{equation}

  \paragraph{2.} Rewrite $R_P(\rho^*_{\lambda_t, \theta}) - R_P(\rho^*_{\lambda_t, \theta'})$ as:
  \begin{equation}
    \begin{split}
      R_P(\rho^*_{\lambda_t, \theta}) &- R_P(\rho^*_{\lambda_t, \theta'}) \\
      &= [R_P(\rho^*_{\lambda_t, \theta}) - R_{\candidate}(\rho^*_{\lambda_t, \theta})] + [R_{\candidate}(\rho^*_{\lambda_t, \theta}) - R_{r_{\theta '}}(\rho^*_{\lambda_t, \theta})] \\
      &\quad + [R_{r_{\theta '}}(\rho^*_{\lambda_t, \theta}) - R_{r_{\theta '}}(\rho^*_{\lambda_t, \theta'})] + [R_{\candidate} (\rho^*_{\lambda_t, \theta'})-R_P(\rho^*_{\lambda_t, \theta'})].
    \end{split}
  \end{equation}
  Define the $\mathbb{B}_\theta$ similarly as in Corollary \ref{col:adapted_mcallester} by:
  \begin{equation}
    \mathbb{B}_\theta (\rho^*_{\lambda_t, \theta}) = R_P(\rho^*_{\lambda_t, \theta}) - R_{\candidate}(\rho^*_{\lambda_t, \theta}).
  \end{equation}
  As result,
  \begin{equation}
    \begin{split}
      R_P(\rho^*_{\lambda_t, \theta}) - R_P(\rho^*_{\lambda_t, \theta'}) &= \mathbb{B}_\theta(\rho^*_{\lambda_t, \theta}) + [R_{\candidate}(\rho^*_{\lambda_t, \theta}) - R_{r_{\theta '}}(\rho^*_{\lambda_t, \theta})] \\
      &+  [R_{r_{\theta '}}(\rho^*_{\lambda_t, \theta}) - R_{r_{\theta '}} (\rho^*_{\lambda_t, \theta'})] - \mathbb{B}_{\theta'} (\rho^*_{\lambda_t, \theta '}).
    \end{split}
  \end{equation}
  Taking absolute values and applying the triangle inequality:
  \begin{equation} 
    \label{eq:aux_triangle_ineq}
    \begin{split}
      |R_P(\rho^*_{\lambda_t, \theta}) - R_P(\rho^*_{\lambda_t, \theta'})| &\leq |\mathbb{B}_\theta(\rho^*_{\lambda_t, \theta})| + |[R_{\candidate}(\rho^*_{\lambda_t, \theta}) - R_{r_{\theta '}}(\rho^*_{\lambda_t, \theta})]| \\
      &+ |[R_{r_{\theta '}}(\rho^*_{\lambda_t, \theta}) - R_{r_{\theta '}} (\rho^*_{\lambda_t, \theta'})]| + |\mathbb{B}_{\theta'} (\rho^*_{\lambda_t, \theta '})|.
    \end{split}
  \end{equation}

  \paragraph{} By Corollary \ref{col:adapted_mcallester} applied at each ratio, with the uniform-in-$h$ assumption $\|\mathcal{L}(h,\cdot)\|_{L^2(Q)} \leq C_{L_2}$:
  \begin{equation}
    \label{eq:aux_bias_bound}
    |\mathbb{B}_\theta(\rho^*_{\lambda_t, \theta})| \leq \| \trueratio - \candidate \|_{L^2(Q)} \, \expectation{h \sim \rho^*_{\lambda_t, \theta}}{\| \mathcal{L}(h, \cdot) \|_{L^2(Q)}} \leq C_{L_2} \| \trueratio - \candidate \|_{L^2(Q)},
  \end{equation}
  and
  \begin{equation}
    \label{eq:aux_bias_bound_prime}
    |\mathbb{B}_{\theta '}(\rho^*_{\lambda_t, \theta '})| \leq C_{L_2} \| \trueratio - r_{\theta '} \|_{L^2(Q)}.
  \end{equation}
  The constant $C_{L_2}$ is the same in both bounds because it is defined as a uniform-in-$h$ bound on $\| \mathcal{L}(h,\cdot)\|_{L^2(Q)}$, depending only on $\mathcal{L}$ and $Q$, not on the ratio model or the posterior.

  \paragraph{} The ratio-shift term holds the posterior fixed on $\rho^*_{\lambda_t, \theta}$ and changes the ratio. By definition,
  \begin{equation}
    \label{eq:aux_ratio_shift_term}
    R_{\candidate} (\rho^*_{\lambda_t, \theta}) - R_{r_{\theta '}} (\rho^*_{\lambda_t, \theta}) = \expectation{h \sim \rho^*_{\lambda_t, \theta}}{\expectation{Q}{(\candidate - r_{\theta '})(Z)\mathcal{L}(h,Z)}},
  \end{equation}
  with $Z \in \mathcal{Z}$. By the Cauchy-Schwarz inequality applied to the inner $\mathbb{E}_Q$ of (\ref{eq:aux_ratio_shift_term}), and using the uniform-in-$h$ bound $\| \mathcal{L}(h, \cdot) \|_{L^2(Q)} \leq C_{L_2}$:
  \begin{equation}
    |\expectation{Q}{(\candidate - r_{\theta '})(Z)\mathcal{L}(h,Z)}| \leq \| \candidate - r_{\theta '} \|_{L^2(Q)} \, \| \mathcal{L}(h, \cdot) \|_{L^2(Q)} \leq C_{L_2} \| \candidate - r_{\theta '} \|_{L^2(Q)}
  \end{equation}
  uniformly in $h$. Taking the expectation over $h \sim \rho^*_{\lambda_t, \theta}$ the uniform bound is preserved:
  \begin{equation}
    \label{eq:aux_ratio_shift_final}
    | R_{\candidate} (\rho^*_{\lambda_t, \theta}) - R_{r_{\theta '}} (\rho^*_{\lambda_t, \theta})| \leq C_{L_2} \| \candidate - r_{\theta '} \|_{L^2(Q)}.
  \end{equation}

  \paragraph{} The posterior-shift term holds the ratio fixed at $r_{\theta '}$ and changes the posterior. The map $h \mapsto \expectation{Q}{\tilde{\mathcal{L}}_{\theta '}(h,Z)}$ is a measurable functional on $\mathcal{H}$ taking values in $[0,1]$, since $\tilde{\mathcal{L}}_{\theta '}(h,Z) \in [0,1]$ for all $(h,Z) \in \mathcal{H} \times \mathcal{Z}$, so its sup-norm is at most $1$. By Lemma~\ref{lemma:variational_char_of_tv}:
  \begin{equation}
    \begin{split}
      |R_{r_{\theta'}}(\rho^*_{\lambda_t, \theta}) &- R_{r_{\theta '}}(\rho^*_{\lambda_t, \theta '})| \\
      &=\left|\expectation{h \sim \rho^*_{\lambda_t, \theta}}{\expectation{Q}{\tilde{\mathcal{L}}_{\theta'}(h,Z)}} - \expectation{h \sim \rho^*_{\lambda_t, \theta '}}{\expectation{Q}{\tilde{\mathcal{L}}_{\theta '}(h,Z)}}\right| \leq 2TV(\rho^*_{\lambda_t, \theta}, \rho^*_{\lambda_t, \theta'}).
    \end{split}
  \end{equation}
  Applying item 1:
  \begin{equation}
    \label{eq:aux_posterior_shift_final}
    |R_{r_{\theta'}}(\rho^*_{\lambda_t, \theta}) - R_{r_{\theta '}}(\rho^*_{\lambda_t, \theta '})| \leq 2TV(\rho^*_{\lambda_t, \theta}, \rho^*_{\lambda_t, \theta'}) \leq \beta \| \Delta \|_{\infty}.
  \end{equation}
  
  \paragraph{} Using (\ref{eq:aux_bias_bound}), (\ref{eq:aux_bias_bound_prime}), (\ref{eq:aux_ratio_shift_final}), (\ref{eq:aux_posterior_shift_final}) in (\ref{eq:aux_triangle_ineq}):
  \begin{equation} \label{eq:aux_intermediate_bound}
    \begin{split}
      |R_P(\rho^*_{\lambda_t, \theta}) &- R_P(\rho^*_{\lambda_t, \theta '})| \\
      &\leq C_{L_2} \| \trueratio - \candidate \|_{L^2(Q)} + C_{L_2}\| \candidate - r_{\theta '} \|_{L^2(Q)} + \beta \| \Delta \|_\infty + C_{L_2}\| \trueratio - r_{\theta'} \|_{L^2(Q)}.
    \end{split}
  \end{equation}
  This four term bound separates the four sources of error: bias of $\rho^*_{\lambda_t, \theta}$, ratio-shift between $\candidate$ and $r_{\theta '}$, posterior-shift between $\rho^*_{\lambda_t, \theta}$ and $\rho^*_{\lambda_t, \theta'}$, and bias of $\rho^*_{\lambda_t, \theta '}$. The two bias terms in (\ref{eq:aux_intermediate_bound}) are not totally independent. By the triangle inequality on the $L^2(Q)$-norm:
  \begin{equation}
    \| \trueratio - \candidate \|_{L^2(Q)} \leq \| \trueratio - r_{\theta '} \|_{L^2(Q)} + \| \candidate - r_{\theta '} \|_{L^2(Q)}.
  \end{equation}
  Then, multiplying by $C_{L_2}>0$:
  \begin{equation}
    \label{eq:aux_inter_bounded_ratio_diff}
    C_{L_2} \| \trueratio - \candidate \|_{L^2(Q)} \leq C_{L_2} \| \trueratio - r_{\theta '} \|_{L^2(Q)} + C_{L_2} \| \candidate - r_{\theta '} \|_{L^2(Q)}.
  \end{equation}
  Substituting (\ref{eq:aux_inter_bounded_ratio_diff}) into (\ref{eq:aux_intermediate_bound}) and collecting like terms yields item 2.

  \paragraph{3.} For each $h \in \mathcal{H}$ and each $i \in \{ 1,...,t \}$,
  \begin{equation}
    |\tilde{\mathcal{L}}_{\theta}(h,Z_i)-\tilde{\mathcal{L}}_{\theta'}(h,Z_i)| = |(\candidate(Z_i)-r_{\theta '}(Z_i))\mathcal{L}(h,Z_i)| \leq C_L \, |\candidate (Z_i) - r_{\theta '}(Z_i)|,
  \end{equation}
  using $|\mathcal{L}(h,Z_i)| \leq C_L$. Averaging over $i$:
  \begin{equation}
    |\weightedrisk{\theta}(h) - \weightedrisk{\theta '}(h)| = \left| \frac{1}{t}\sum \limits^t_{i=1}  \tilde{\mathcal{L}}_{\theta}(h,Z_i)-\tilde{\mathcal{L}}_{\theta'}(h,Z_i)\right| \leq C_L \frac{1}{t}\sum \limits^t_{i=1} |\candidate (Z_i) - r_{\theta '}(Z_i)|.
  \end{equation}
  The right-hand side does not depend on $h$, so taking the supremum over $h \in \mathcal{H}$ yields
  \begin{equation}
    \sup \limits_{h \in \mathcal{H}}|\weightedrisk{\theta}(h) - \weightedrisk{\theta '}(h)| \leq C_L \frac{1}{t}\sum \limits^t_{i=1} |\candidate (Z_i) - r_{\theta '}(Z_i)|.
  \end{equation}
  \paragraph{} Taking the expectation over $Z_1,...,Z_t \iid Q$,
  \begin{equation}
    \expectation{Q}{\sup \limits_{h \in \mathcal{H}}|\weightedrisk{\theta}(h) - \weightedrisk{\theta '}(h)|} \leq C_L \expectation{Q}{|\candidate (Z) - r_{\theta '}(Z)|} = C_L \| \candidate (Z) - r_{\theta '}(Z) \|_{L^1(Q)}.
  \end{equation}
  The Cauchy-Schwarz refinement 
  \begin{equation}
    \| \candidate (Z) - r_{\theta '}(Z) \|_{L^1(Q)} \leq \| \candidate (Z) - r_{\theta '}(Z) \|_{L^2(Q)}
  \end{equation}
  holds since $Q$ is a probability measure and $1 \in L^2(Q)$ with $\| 1\|_{L^2(Q)} = 1$.
\end{proof}

\subsection{Anytime mechanism}
\label{sec:anytime}

\paragraph{} The fixed-time guarantees established in Theorem \ref{theo:adapted_mcallester} and Corollary \ref{col:adapted_mcallester} apply at a single, pre-specified sample-size $t$. In that setting the value of $t$ is fixed before the data is observed, and the resulting certificate controls the weighted population risk, or the target risk after adding the DRN bias term, only at that chosen sample size. It does not by itself, justify using the same certificate after inspecting a sequence of intermediate checkpoints.

\paragraph{} The limitation is important for the present framework, because the training procedure is naturally sequential. In practice, one may monitor the empirical weighted risk, the PAC-Bayes penalty, ESS, tail-diagnostics, and constraint residuals across training time. One may also stop training once the certificate becomes sufficiently small, or select a checkpoint from a collection of candidates according to the observed bound. In all of these cases, the final sample size $T$ and the posterior $\rho_T$ are no longer fixed in advance. They are data-dependent objects determined by the observed trajectory. A fixed-time statement at a pre-committed $t$ does not automatically remain valid at such a random time $T$.

\paragraph{} A simple workaround is to apply the fixed time result to a finite set of pre-specified checkpoints and allocate the confidence level across them, for example, by replacing $\delta$ with $\delta/k$ over $K$ checkpoints. This is valid for that finite collection, but it is conservative and does not provide a genuinely sequential certificate over an unbound training horizon. What is needed is a time-uniform, or anytime, guarantee: a single event of probability at least $1-\delta$ on which the PAC-Bayes inequality holds simultaneously for all $t \geq t_{\min}$. Such a result allows sequential monitoring, adaptive stopping, and checkpoint selection without invalidating the risk certificate. The remainder of this section establishes this anytime guarantee for the DRN-induced weighted risk and then combines with the DRN bias decomposition to obtain an anytime target-risk bound.

\paragraph{} Fix a minimum monitoring time $t_{\min} \geq 1$, a geometric base $b>1$, and a total confidence level $\delta \in (0,1)$. The time axis is partitioned into geometric epochs
\begin{equation}
  \mathcal{T}_k = [t_{\min}b^k, t_{\min}b^{k+1}) \cap \mathbb{N}, \; k = 0,1,...
\end{equation}
For each $t \geq t_{\min}$, define its epoch index by
\begin{equation}
  \label{eq:epoch_index}
  k(t) = \max\!\left\{ 0,\; \floor{\log_b \frac{t}{t_{\min}}} \right\},
\end{equation}
so that $t \in \mathcal{T}_{k(t)}$. Let $|\mathcal{T}_k|$ denote the number of integer sample sizes combined in the $k$-th epoch. A direct count gives
\begin{equation}
  |\mathcal{T}_k| \leq t_{\min}b^{k+1} - t_{\min}b^k = t_{\min}b^k(b-1).
\end{equation}
The confidence budget is allocated in two stages. First, the total budget $\delta$ is split across epochs by setting:
\begin{equation}
  \delta_k = \delta \frac{b-1}{b^{k+1}}, \; k \geq 0.
\end{equation}
This allocation is summable, since 
\begin{equation}
  \sum \limits^{\infty}_{k=0} \delta_k = \delta(b-1) \sum \limits^\infty_{k=0} \frac{1}{b^{k+1}} = \delta.
\end{equation}
Within each epoch $\mathcal{T}_k$, the epoch budget is split uniformly across the finitely many sample sizes in that epoch:
\begin{equation}
  \delta_{k,t} = \frac{\delta_k}{|\mathcal{T}_k|}, \; t \in \mathcal{T}_k.
\end{equation}
Thus, 
\begin{equation}
  \sum \limits_{t \in \mathcal{T}_k} \delta_{k,t} = \delta_k, 
\end{equation}
and consequently,
\begin{equation}
  \sum \limits^\infty_{k = 0} \sum \limits_{t \in \mathcal{T}_k} = \delta.
\end{equation}
This allocation will allow the fixed-time PAC-Bayes bound to be applied at each sample size and then union-bounded first within each epoch and then across all epochs.

\begin{remark}
  The notation $\floor{x}$ denotes the floor of $x$, namely the target integer less or equal to $x$. Thus, equation (\ref{eq:epoch_index}) assigns the integer index of the geometric epoch containing $t$. In particular, if $t \in [t_{\min}b^k, t_{\min}b^{k+1}]$, then
  \begin{equation}
    \log_b \left( \frac{t}{t_{\min}} \right) \in [k, k+1),
  \end{equation}
  and, hence
  \begin{equation}
    \floor{\log_b \left( \frac{t}{t_{\min}} \right)} = k.
  \end{equation}
\end{remark}

\begin{remark}
  The sequence $(\delta_k)_{k \geq 0}$ assigns confidence budget to epochs, while $\delta_{k,t}$ assigns confidence budgets to individual sample sizes within an epoch. More precisely, $\delta_K$ denote the total failure probability allocated to the epoch $\mathcal{T}_k$, whereas, for each $t \in \mathcal{T}_k$, $\delta_{k,t}$ is the partition of that epoch budget allocated to the single time $t$. For a particular sample size $t \geq t_{\min}$, the notation $k(t)$ denotes the epoch index containing $t$. Hence, $\delta_{k(t)}$ is the budget of the epoch containing $t$, and $\delta_{k(t), t}$ is the budget assigned specifically to that sample size.
\end{remark}

\begin{theo}\label{theo:anytime_pac_bayes}
  Assume the conditions of Theorem \ref{theo:adapted_mcallester}. Fix $t \geq t_{\min}$, $b>1$, $\delta \in (0,1)$. Then, with probability at least $1-\delta$ over $Z_1,...,Z_t \iid Q$, simultaneously for every $t \geq t_{\min}$ and every posterior $\rho \ll \pi$:
  \begin{equation}
    \label{eq:anytime_pac_bayes_ineq}
    R_{\candidate} (\rho) \leq \weightedrisk{\theta} (\rho) + \sqrt{\frac{\mathbb{KL}(\rho || \pi)+\ln (1/\delta_{k(t)})+\ln|\mathcal{T}_k(t)| + \ln t + 2}{2t-1}}.
  \end{equation}
  The bound remains valid at any (possibly data-dependent) stopping time $T \geq t_{\min}$ and any data-dependent posterior $\rho_T \ll \pi$.
\end{theo}

\begin{proof}
  The argument proceeds by a double union bound. First within each epoch (over $|\mathcal{T}_k|$ sample sizes), then across epochs over $k$.

  \paragraph{} Fix any epoch $k \geq 0$ and any single sample $t \in \mathcal{T}_k$. Apply Theorem \ref{theo:r_star_uniqueness} with confidence level $\delta_{k,t}$. Thus, with probability at least $1-\delta_{k,t}$ over the draw $Z_1,...,Z_t \iid Q$, simultaneously for every $\rho \ll \pi$:
  \begin{equation}
    \label{eq:aux_weighted_bound_for_fixed_t}
    R_{\candidate} (\rho) \leq \weightedrisk{\theta} (\rho) + \sqrt{\frac{\mathbb{KL}(\rho || \pi)+\ln (1/\delta_{k,t}) + \ln t + 2}{2t-1}}.
  \end{equation}
  Define the per-sample-size event:
  \begin{equation}
    \label{eq:per_sample_size_event}
    A_{k,t} = \{ Z_1, ..., Z_t : \; \text{for every} \; \rho \ll \pi, \; R_{\candidate} (\rho) \leq \weightedrisk{\theta} (\rho) \}.
  \end{equation}
  Note that the inequality in (\ref{eq:aux_weighted_bound_for_fixed_t}) is already uniform in $\rho$. So, the "for every $\rho$" quantifier inside the event is automatic. The new union bound needed is over $t$ and $k$, and not over $\rho$.

  \paragraph{} For each $k \geq 0$, define the per-epoch event:
  \begin{equation}
    E_k = \bigcap \limits_{t \in \mathcal{T}_k}A_{k,t}.
  \end{equation}
  By countable subadditivity:
  \begin{equation}
    \mathbb{P}(E^C_k) = \mathbb{P}\left[ \bigcup \limits_{t \in \mathcal{T}_k} A_{k,t}^C \right] \leq \sum \limits_{t \in \mathcal{T}_k} \delta_{k,t} = \delta_k
  \end{equation}
  where $A^C_{k,t}$ is the complementary of $A_{k,t}$ and $E^C_k$ is the complementary of $E_k$. By construction, $E_k$ is the intersection of all fixed-time events $A_{k,t}$ over $t \in \mathcal{T}_k$. Hence, on $E_k$, every event $A_{k,t}$ occurs simultaneously and, therefore, the bound in (\ref{eq:aux_weighted_bound_for_fixed_t}) holds for every $\rho \ll \pi$. The subadditivity calculation shows that the probability of leaving this simultaneous event is at most $\delta_k$, namely $\mathbb{P}(E_k^C) \leq \delta_k$. Thus, the price of making the bound simultaneous over the epoch is exactly the epoch-level confidence budget $\delta_k$.

  \paragraph{} Substituting $\delta_{k,t} = \delta_k/|\mathcal{T}|_k$ into (\ref{eq:aux_weighted_bound_for_fixed_t}), the bound on $E_k$ reads, for every $t \in \mathcal{T}_k$ and every $\rho \ll \pi$:
  \begin{equation}
    R_{\candidate} (\rho) \leq \weightedrisk{\theta} (\rho) + \sqrt{\frac{\mathbb{KL}(\rho || \pi)+\ln (|\mathcal{T}_k|/\delta_{k})+\ln t + 2}{2t-1}}
  \end{equation}
  Splitting the logarithm in 
  \begin{equation}
    \ln (|\mathcal{T}|/\delta_k) = \ln |\mathcal{T}_k| + \ln (1/\delta_k),
  \end{equation}
  gives the inequality (\ref{eq:anytime_pac_bayes_ineq}).

  \paragraph{} Define the global event 
  \begin{equation}
    E = \bigcap_{k \geq 0} E_k.
  \end{equation}
  By countable subadditivity:
  \begin{equation}
    \mathbb{P}(E^C) = \mathbb{P} \left( \bigcup \limits^\infty_{k=0} E^C_k \right) \leq \sum \limits^{\infty}_{k = 0} \mathbb{P}(E^C_k) \leq \sum \limits^{\infty}_{k=0} \delta_k = \delta (b-1) \sum \limits^{\infty}_{k=0} \frac{1}{b^{k+1}} = \delta,
  \end{equation}
  where $E^C$ is the complementary of $E$. Therefore, $\mathbb{P}(E) \geq 1 -\delta$.

  \paragraph{} The epochs $\{ \mathcal{T}_k \}^{\infty}_{k=0}$ form a partition of $[t_{\min},\infty)\cap \mathbb{N}$:
  \begin{equation}
    \bigcup \limits^{\infty}_{k=0} \mathcal{T}_k = [t_{\min},\infty) \bigcap \mathbb{N}, \; 
  \end{equation}
  with $\mathcal{T}_j \cap \mathcal{T}_k = \emptyset$ for $j \neq k$, since $\mathcal{T}_k = [t_{\min}b^k, t_{\min}b^{k+1})$ is a half-open interval and consecutive epochs share boundary points only as left-endpoints, with right-endpoints excluded. Therefore, for any fixed $t \geq t_{\min}$, there is a unique index $k = k(t) \geq 0$ such that $t \in \mathcal{T}_{k(t)}$. 

  \paragraph{} Since $E$ is the intersection of all per-epochs events,
  \begin{equation}
    E \subseteq E_{k(t)} = \bigcap \limits_{s \in \mathcal{T}_{k(t)}} A_{k(t),s}.
  \end{equation}
  In particular, for every $w \in E$, it holds $w \in A_{k(t), s}$ for every sample size $s \in \mathcal{T}_{k(t)}$, including $s = t$ specifically. By construction of $A_{k(t), t}$ in (\ref{eq:per_sample_size_event}), the event $A_{k(t), t}$ is precisely the event on which the inequality (\ref{eq:aux_weighted_bound_for_fixed_t}) holds for every $\rho \ll \pi$ simultaneously. That is, for every $w \in A_{k(t), t}$ and every $\rho \ll \pi$,
  \begin{equation}
    \label{eq:aux_weighted_bound_anytime}
    R_{\candidate} (\rho) \leq \weightedrisk{\theta} (\rho) + \sqrt{\frac{\mathbb{KL}(\rho || \pi)+\ln (1/\delta_{k(t), t}) + \ln t + 2}{2t-1}}.
  \end{equation}
  Since $E \subseteq E_{k(t)} \subseteq A_{k(t), t}$ every $w \in E$ is also in $A_{k(t), t}$, so the inequality holds on $E$ for every $\rho \ll \pi$.

  \paragraph{} Substituting $\delta_{k(t),t} = \delta_{k(t)}/|\mathcal{T}_{k(t)}|$ in the first logarithm of (\ref{eq:aux_weighted_bound_anytime}),
  \begin{equation}
    \ln (1/\delta_{k(t),t}) = \ln (|\mathcal{T}_{k(t)}|/\delta_{k(t)}) = \ln |\mathcal{T}_{k(t)}| + \ln (1/\delta_{k(t)}).
  \end{equation}
  Therefore, on $E$ and for every $\rho \ll \pi$ the inequality (\ref{eq:anytime_pac_bayes_ineq}) is achieved.

  \paragraph{} Note that all the arguments here hold for every fixed time $t \geq t_{\min}$, but it operates entirely within the single global event $E$. No part of the derivation requires committing to a particular $t$ in advance, and the event $E$ itself is independent of the choice of $t$. As result, on the single event $E$ of probability at least $1-\delta$, the inequality holds simultaneously for every $t \geq t_{\min}$ and every $\rho \ll \pi$.

  \paragraph{} Let $T\geq t_{\min}$ be any stopping time and $\rho_T$ be any measurable posterior with respect to $Z_1, ..., Z_T$. On $E$ the inequality (\ref{eq:anytime_pac_bayes_ineq}) holds simultaneously for all $t \geq t_{\min}$ and all $\rho$, so in particular at $t = T(w)$ and $\rho = \rho_T(w)$ for every $w \in E$. Hence,
  \begin{equation}
    \mathbb{P} \left( R_{\candidate} (\rho) \leq \weightedrisk{\theta} (\rho) + \sqrt{\frac{\mathbb{KL}(\rho || \pi)+\ln (1/\delta_{k(t)})+\ln|\mathcal{T}_k(t)| + \ln t + 2}{2t-1}} \right) \geq \mathbb{P}(E) \geq 1 -\delta.
  \end{equation}
\end{proof}

\begin{corollary}
  Assume the conditions of Theorem \ref{theo:anytime_pac_bayes}. Then, on the same event of probability at least $1-\delta$ over the draw of $Z_1,... \iid Q$ on which Theorem \ref{theo:anytime_pac_bayes} holds, simultaneously for every $t \geq t_{\min}$ and every posterior $\rho \ll \pi$:
  \begin{equation}
    R_P(\rho) \leq \weightedrisk{\theta} (\rho) + 
    \mathbb{B}_\theta (\rho) + \sqrt{\frac{\mathbb{KL}(\rho || \pi)+\ln (1/\delta_{k(t)})+\ln|\mathcal{T}_k(t)| + \ln t + 2}{2t-1}},
  \end{equation}
  where $\mathbb{B}_\theta (\rho) = R_P(\rho) - R_{\candidate} (\rho)$ is the DRN bias of the posterior $\rho$ at ratio $\candidate$. By Corollary \ref{col:adapted_mcallester}, the bias is controlled by
  \begin{equation}
    \mathbb{B}_\theta (\rho) \leq \| \trueratio - \candidate \|_{L^2(Q)} \posteriorexp{\| \mathcal{L}(h, \cdot) \|_{L^2(Q)}} \leq \hat{C}_{L_2} \| \trueratio - \candidate \|_{L^2(Q)},
  \end{equation}
  where
  \begin{equation}
    \hat{C}_{L_2} = \sup \limits_{h \in \mathcal{H}} \| \mathcal{L}(h,\cdot)\|_{L^2(Q)}.
  \end{equation}
\end{corollary}

\begin{proof}
  Theorem \ref{theo:anytime_pac_bayes} gives on a single event $E$ of probability equal or greater to $1-\delta$, the inequality (\ref{eq:anytime_pac_bayes_ineq}) simultaneously for every $t \geq t_{\min}$ and every $\rho \ll \pi$. The decomposition $R_P(\rho) = R_{\candidate} (\rho) + \mathbb{B}_\theta (\rho)$ holds deterministically (it does not depend on a random sample), so adding $\mathbb{B}_\theta (\rho)$ to both sides preserves the simultaneous-in-$(t,\rho)$ validity on $E$. The bias bound is exactly Corollary \ref{col:adapted_mcallester} applied pointwise. Hence, the stated bound holds on the same vent $E$ with no additional confidence budget consumed.
\end{proof}

\section{Numerical results}
\label{sec:numerical}

\paragraph{} The numerical study is organized in two pre-registered campaigns that together exercise every layer of the framework. The first is a controlled patch test on a one-dimensional Gaussian mean-shift instance with analytic ground truth, designed so that each conceptual layer of the framework can be tested against a closed-form reference value rather than against itself or against an aggregate score. The second is a real-data validation on two openly licensed tabular datasets with intrinsic non-synthetic distribution shift, designed so that the layers that pass the patch test are placed in the deployment regime, where analytic targets no longer exist and where the only available reference is a held-out target test fold. The two campaigns are complementary in scope. The patch test isolates the layers and verifies their internal mechanics against analytic identities, while the real-data validation runs the layers end-to-end and verifies their joint behavior on splits the framework was not engineered against.

\paragraph{} Both campaigns are strictly pre-registered. Tolerances, seed schedules, and acceptance criteria are committed to immutable artifacts before evaluation, and every pre-registered claim is reported with its verdict regardless of outcome. The tolerance files of the two campaigns are kept disjoint, so that thresholds appropriate to analytic ground truth and thresholds appropriate to real-data deployment cannot be conflated at evaluation time. Across the two campaigns, all but one of the pre-registered claims close at the registered thresholds. The single failure is the fixed-time PAC-Bayes coverage criterion under label shift in the real-data campaign, which is paper-relevant rather than diagnostic, in the sense that it is the empirical realization of the boundary that the framework's covariate-only assumption draws on the certificate's scope. The remainder of the section is structured by campaign. Section~\ref{sec:patch_test} reports the patch test, and Section~\ref{sec:real_data} reports the real-data validation.

\subsection{Patch test}
\label{sec:patch_test}

\paragraph{} The objective of the patch test is to empirically validate every layer of the framework in isolation, on a controlled instance for which every relevant population quantity is available in closed form. Because the measure-change identity, the constrained density-ratio network, the weighted empirical risk, and the PAC-Bayes certificates are logically modular, it is meaningful to ask that each layer satisfy its own quantitative property, rather than testing the full pipeline as a single opaque object. The patch test specifies, for each layer, a falsifiable statistical criterion expressed in terms of quantities that can be computed analytically under the chosen data-generating process. This yields a pre-registered validation protocol whose acceptance or rejection is decided before any figure or claim is reported.

\paragraph{} The patch test specializes the abstract setting of Section~\ref{sec:foundation} to the one-dimensional Gaussian mean-shift family. Let
\begin{equation}
  \mathcal{Z} = \mathbb{R}, \qquad \mathcal{A} = \mathcal{B}(\mathbb{R}),
\end{equation}
and fix a shift parameter $\mu \in \mathbb{R}$. Define the source and target laws by
\begin{equation}
  \label{eq:pt_dgp}
  Q = \mathcal{N}(0, 1), \qquad P_\mu = \mathcal{N}(\mu, 1).
\end{equation}
Since $P_\mu \ll Q$ for every $\mu \in \mathbb{R}$, the Radon-Nikodym derivative admits the closed form
\begin{equation}
  \label{eq:pt_true_ratio}
  \trueratio_\mu(z) = \frac{dP_\mu}{dQ}(z) = \exp\!\left( \mu z - \tfrac{1}{2} \mu^2 \right), \qquad z \in \mathbb{R}.
\end{equation}
Expression~\eqref{eq:pt_true_ratio} provides analytic ground truth for each diagnostic used in the patch test. In particular, the following identities hold under $Q$:
\begin{equation}
  \label{eq:pt_analytic_identities}
  \mathbb{E}_Q[\trueratio_\mu] = 1, \qquad
  \mathbb{E}_Q[(\trueratio_\mu)^2] = e^{\mu^2}, \qquad
  \mathbb{E}_Q[\trueratio_\mu(Z)\, Z] = \mu, \qquad
  \mathbb{E}_Q[\trueratio_\mu(Z)\, Z^2] = 1 + \mu^2,
\end{equation}
and the population effective sample fraction
\begin{equation}
  \label{eq:pt_ess}
  \mathrm{ESS}(\trueratio_\mu) := \frac{\bigl(\mathbb{E}_Q[\trueratio_\mu]\bigr)^2}{\mathbb{E}_Q[(\trueratio_\mu)^2]} = e^{-\mu^2}
\end{equation}
admits the same closed form. Three regimes are considered: a \emph{mild-shift} regime $\mu = 0.5$, for which $\mathrm{ESS}(\trueratio_\mu) \approx 0.779$, and two \emph{stress} regimes $\mu \in \{1.5, 2.0\}$, for which $\mathrm{ESS}(\trueratio_\mu) \in \{1.05 \times 10^{-1},\, 1.83 \times 10^{-2}\}$. The stress regimes are designed to probe heavy-tailed importance weights: as $\mu$ grows, $\mathbb{E}_Q[(\trueratio_\mu)^2]$ grows exponentially and $\mathrm{ESS}(\trueratio_\mu)$ collapses.

\paragraph{} The predictor family is the one-parameter space of scalar linear maps
\begin{equation}
  \mathcal{H} = \bigl\{ h_a : z \mapsto a\, z,\ a \in \mathbb{R} \bigr\},
\end{equation}
and the loss is the squared error associated with the identification $Z \equiv X \equiv Y$, that is,
\begin{equation}
  \label{eq:pt_loss}
  \mathcal{L}(h_a, z) = (z - h_a(z))^2 = (1 - a)^2 z^2.
\end{equation}
Under this convention, the target risk admits the closed form
\begin{equation}
  \label{eq:pt_target_risk}
  R_{P_\mu}(h_a) = \mathbb{E}_{P_\mu}[\mathcal{L}(h_a, Z)] = (1 - a)^2 (1 + \mu^2),
\end{equation}
and the corresponding Monte Carlo standard error at sample size $n$ is
\begin{equation}
  \label{eq:pt_sigma_mc}
  \sigma_{\mathrm{MC}}(h_a; \mu, n) = \sqrt{ \frac{ \operatorname{Var}_{P_\mu}\!\bigl[ (1-a)^2 Z^2 \bigr] }{n} } = (1-a)^2 \sqrt{ \frac{ 2 + 4 \mu^2 }{ n } }.
\end{equation}
For the PAC-Bayes stages, the loss~\eqref{eq:pt_loss} is rescaled to the unit interval through the fixed ceiling $L_{\max}$:
\begin{equation}
  \label{eq:pt_clipped_loss}
  \tilde{\mathcal{L}}(h_a, z) = \frac{ \min\bigl( \mathcal{L}(h_a, z),\, L_{\max} \bigr) }{ L_{\max} } \in [0, 1], \qquad L_{\max} = 16.
\end{equation}
The ceiling $L_{\max}$ is chosen so that, under the posteriors considered below, clipping is triggered only on events of probability $\mathcal{O}(10^{-15})$ under $Q$, so that $\tilde{\mathcal{L}}$ and $\mathcal{L} / L_{\max}$ coincide up to negligible error in all population-level comparisons.

\paragraph{} The ratio family is a fully connected neural network with two hidden layers of width $64$ and activation $\operatorname{softplus}$, post-composed with a positivity floor $\varepsilon = 10^{-3}$:
\begin{equation}
  \label{eq:pt_ratio_network}
  r_\theta(z) = \max\!\bigl( \operatorname{softplus}(f_\theta(z)),\, \varepsilon \bigr), \qquad f_\theta : \mathbb{R} \longrightarrow \mathbb{R},\ \theta \in \Theta.
\end{equation}
The floor $\varepsilon$ guarantees $r_\theta(z) > 0$ for every $z \in \mathbb{R}$ and every $\theta$, so that the induced finite measure $P_\theta := r_\theta Q$ is well defined. The parameter $\theta$ is optimized with Adam at learning rate $10^{-3}$ using a full-batch gradient, for a fixed budget of $T_{\mathrm{opt}} = 2000$ steps.

\paragraph{} The population fitting criterion is the least-squares importance-fitting (LSIF) objective
\begin{equation}
  \label{eq:pt_lsif}
  \mathbb{J}_{\mathrm{fit}}(\theta) = \tfrac{1}{2}\, \mathbb{E}_Q[r_\theta(Z)^2] - \mathbb{E}_{P_\mu}[r_\theta(Z)],
\end{equation}
whose unique minimizer over measurable nonnegative functions is $\trueratio_\mu$. For more details about the LSIF and other classical density-ratio estimation method refer to Appendix \ref{ap:classical_dre}. Given independent samples $\{Z_i^Q\}_{i=1}^{n_Q} \stackrel{\mathrm{iid}}{\sim} Q$ and $\{Z_j^P\}_{j=1}^{n_P} \stackrel{\mathrm{iid}}{\sim} P_\mu$, the criterion is estimated by plug-in empirical averages, with $n_Q = n_P = 10^4$ in all patch-test experiments. On top of~\eqref{eq:pt_lsif}, the framework imposes two classes of integral constraints that encode the defining identities of a Radon-Nikodym derivative:
\begin{align}
  \label{eq:pt_norm_constraint}
  g_0(\theta) &:= \mathbb{E}_Q[r_\theta(Z)] - 1 = 0 &&\text{(normalization)}, \\
  \label{eq:pt_mm_constraint}
  g_j(\theta) &:= \mathbb{E}_Q[r_\theta(Z)\, \phi_j(Z)] - \mathbb{E}_{P_\mu}[\phi_j(Z)] = 0, && j = 1, 2,
\end{align}
with test functions $\phi_1(z) = z$ and $\phi_2(z) = z^2$. Both families of constraints are enforced through an augmented-Lagrangian scheme: denoting by $\lambda \in \mathbb{R}$ the dual multiplier of~\eqref{eq:pt_norm_constraint} and by $\mu_j \in \mathbb{R}$ those of~\eqref{eq:pt_mm_constraint}, the primal update minimizes
\begin{equation}
  \label{eq:pt_al_obj}
  \mathcal{L}_{\mathrm{AL}}(\theta; \lambda, \mu_1, \mu_2)
  = \mathbb{J}_{\mathrm{fit}}(\theta)
  + \lambda\, g_0(\theta) + \tfrac{1}{2}\, \rho_0\, g_0(\theta)^2
  + \sum_{j=1}^{2} \Bigl[ \mu_j\, g_j(\theta) + \tfrac{1}{2}\, \rho_j\, g_j(\theta)^2 \Bigr],
\end{equation}
with quadratic-penalty coefficients $\rho_0, \rho_1, \rho_2 > 0$, and the dual variables are updated by
\begin{equation}
  \label{eq:pt_dual_update}
  \lambda \leftarrow \lambda + \eta_{\mathrm{norm}}\, g_0(\theta), \qquad
  \mu_j \leftarrow \mu_j + \eta_{\mathrm{mm}}\, g_j(\theta), \ j = 1, 2.
\end{equation}
The dual step sizes $\eta_{\mathrm{norm}} = 10^{-1}$ and $\eta_{\mathrm{mm}} = 5 \times 10^{-3}$ are fixed by pre-registration.

Under the stress regimes, two tail-control mechanisms are considered, both acting on the deployed ratio:
\begin{equation}
  \label{eq:pt_clip_temper}
  r_\theta^{\mathrm{clip}}(z) = \min\bigl( r_\theta(z),\, c \bigr), \qquad
  r_\theta^{\mathrm{temp}}(z) = r_\theta(z)^\beta,
\end{equation}
with clipping threshold $c \in \{20, 60\}$ (for $\mu = 1.5$ and $\mu = 2.0$, respectively) and tempering exponent $\beta \in (0, 1]$. Clipping is a measure-valid transformation and is enforced inside the constraints~\eqref{eq:pt_norm_constraint}-\eqref{eq:pt_mm_constraint}, whereas tempering is restricted to the training-time LSIF objective and does not modify the deployed ratio. This distinction, between identity-preserving transformations and training-time regularization transformations, is built into the patch-test specification.

\paragraph{} The PAC-Bayes layer is instantiated on the same hypothesis space $\mathcal{H}$ and with the clipped-scaled loss $\tilde{\mathcal{L}}$ of~\eqref{eq:pt_clipped_loss}. The prior $\Pi$ and the evaluation posterior $\rho$ are Gaussian laws over the scalar parameter $a$:
\begin{equation}
  \label{eq:pt_prior_posterior}
  \Pi = \mathcal{N}(0, 1), \qquad \rho = \mathcal{N}(a_0, \sigma_\rho^2),
\end{equation}
with canonical sanity point $(a_0, \sigma_\rho^2) = (0.5, 0.01)$. The confidence parameter is fixed at $\delta = 0.05$ throughout. Two fixed-time PAC-Bayes bounds are evaluated: the square-root form and the Bernoulli-KL (Seeger) form, both instantiated on the importance-weighted empirical risk under $r_\theta$ (or, in the oracle-sanity configuration, under $\trueratio_\mu$). The Bernoulli-KL form relies on the divergence $\mathrm{kl}_{\mathrm{Ber}}$ of Definition~\ref{def:kl_ber} and on its upper inverse $\mathrm{kl}_{\mathrm{Ber}}^{-1}$ of Definition~\ref{def:kl_ber_inverse}, both recorded in Appendix~\ref{ap:math_aux}, together with the well-posedness statement of Lemma~\ref{lemma:kl_ber_inverse} that makes the inverse a function rather than a set-valued map. The time-uniform (\emph{anytime}) bound is constructed by geometric peeling with base $b = 2$ and $t_{\min} = 100$: the sample-size axis is partitioned into epochs $\mathcal{T}_k = \{\, t : b^k t_{\min} \le t < b^{k+1} t_{\min} \,\}$, $k = 0, 1, 2, \dots$, and the confidence budget is allocated as
\begin{equation}
  \label{eq:pt_peeling}
  \delta_k = \delta \cdot \frac{b - 1}{b^{k+1}}, \qquad \sum_{k \ge 0} \delta_k = \delta.
\end{equation}
A fixed-time bound at level $\delta_k$ is applied on each $\mathcal{T}_k$ and a union bound yields a certificate valid simultaneously for every $t \ge t_{\min}$.

\paragraph{} The patch test is organized into eight pre-registered stages $\mathrm{S}0, \dots, \mathrm{S}7$, each isolating one mechanism of the framework and inheriting the configuration of its predecessor. The ratio produced at stage $\mathrm{S}k$ is denoted $[r_\theta]_{\mathrm{S}k}$. Stage $\mathrm{S}0$ validates the measure-change foundation by substituting $\trueratio_\mu$ into the weighted-risk estimator and verifying recovery of the identities~\eqref{eq:pt_analytic_identities}. Stage $\mathrm{S}1$ fits the unconstrained LSIF criterion~\eqref{eq:pt_lsif}, after which stage $\mathrm{S}2$ activates the normalization constraint~\eqref{eq:pt_norm_constraint} and stage $\mathrm{S}3$ adds the moment-matching constraints~\eqref{eq:pt_mm_constraint} for $j = 1, 2$. Stage $\mathrm{S}4$ activates the tail-control transformations~\eqref{eq:pt_clip_temper} at the stress shifts $\mu \in \{1.5, 2.0\}$. The downstream weighted risk is addressed at stage $\mathrm{S}5$, which compares the weighted empirical risk $R_{[r_\theta]_{\mathrm{S}3}}(h_a)$ against $R_{P_\mu}(h_a)$ over a grid of predictors $a \in \{-1, -0.5, 0, 0.5, 1\}$, while stages $\mathrm{S}6$ and $\mathrm{S}7$ address the PAC-Bayes layer: the former evaluates the square-root and Bernoulli-KL bounds at a fixed sample size $t$, and the latter evaluates the peeling construction~\eqref{eq:pt_peeling} jointly over all $t \ge t_{\min}$ up to a horizon $T$. Stages $\mathrm{S}0$--$\mathrm{S}4$ therefore address the density-ratio layer, stage $\mathrm{S}5$ the downstream weighted-risk layer, and stages $\mathrm{S}6$--$\mathrm{S}7$ the PAC-Bayes layer, with the ratio configuration deployed at $\mathrm{S}5$--$\mathrm{S}7$ inherited from $\mathrm{S}3$ at $\mu = 0.5$ and from $\mathrm{S}4$ with clipping at the stress shifts.

\paragraph{} Each stage exposes a finite list of scalar diagnostics $(D_\ell)_{\ell=1}^{L}$. Each diagnostic is a real-valued functional
\begin{equation}
  \label{eq:pt_diagnostic}
  D_\ell : (r, h) \longmapsto D_\ell(r, h; P_\mu, Q) \in \mathbb{R},
\end{equation}
defined as a fixed integral against $Q$ or $P_\mu$ — for example, $\expectation{Q}{r(Z)}$, $\expectation{Q}{r(Z)\, \phi_j(Z)}$, $\expectation{Q}{r(Z)^2}$, or $\expectation{P_\mu}{\mathcal{L}(h, Z)}$ — and depends, at each stage, on the candidate ratio $r$ produced by that stage and, when relevant, on a predictor $h \in \mathcal{H}$. The \emph{reference value}
\begin{equation}
  \label{eq:pt_dstar}
  D_\ell^\star := D_\ell(\trueratio_\mu, h; P_\mu, Q)
\end{equation}
is the value of the diagnostic when the candidate ratio coincides with the true Radon-Nikodym derivative $\trueratio_\mu$ and population expectations are taken under the data-generating measures $(P_\mu, Q)$. By construction, $D_\ell^\star$ is available in closed form through the analytic identities~\eqref{eq:pt_analytic_identities}, the population effective sample fraction~\eqref{eq:pt_ess}, and the target risk~\eqref{eq:pt_target_risk}. For instance, the diagnostics $\expectation{Q}{r}$, $\expectation{Q}{r(Z) Z^2}$, and $R_{P_\mu}(h_a)$ have reference values $1$, $1 + \mu^2$, and $(1-a)^2(1+\mu^2)$, respectively. The \emph{empirical counterpart}
\begin{equation}
  \label{eq:pt_dhat}
  \widehat D_\ell := D_\ell\bigl( \candidate, h;\ \widehat{P}_\mu^{\,(n_P)},\ \widehat{Q}^{\,(n_Q)} \bigr)
\end{equation}
is the plug-in Monte Carlo estimator obtained by replacing $\trueratio_\mu$ with the candidate ratio $\candidate$ produced by the stage under consideration (with $\candidate \equiv \trueratio_\mu$ at $\mathrm{S}0$), and the population expectations under $Q$ and $P_\mu$ with sample averages against the empirical measures
\begin{equation}
  \label{eq:pt_empirical_measures}
  \widehat{Q}^{\,(n_Q)} := \frac{1}{n_Q} \sum_{i=1}^{n_Q} \delta_{Z_i^Q}, \qquad
  \widehat{P}_\mu^{\,(n_P)} := \frac{1}{n_P} \sum_{j=1}^{n_P} \delta_{Z_j^P},
\end{equation}
associated with independent samples $\{Z_i^Q\}_{i=1}^{n_Q} \iid Q$ and $\{Z_j^P\}_{j=1}^{n_P} \iid P_\mu$, with $n_Q = n_P = 10^4$ throughout. The deviation $|\widehat D_\ell - D_\ell^\star|$ therefore aggregates two distinct sources of error, namely the Monte Carlo error of the empirical averages and the ratio-estimation error $\candidate - \trueratio_\mu$, which at $\mathrm{S}0$ collapses to the first source alone since $\candidate$ is set to $\trueratio_\mu$.

Two criterion templates are used. The first is a Monte Carlo tolerance, which for a diagnostic $D_\ell$ with reference value $D_\ell^\star$ and Monte Carlo standard deviation $\sigma_{\mathrm{MC}}(D_\ell; n)$ of $\widehat D_\ell$ requires
\begin{equation}
  \label{eq:pt_mc_tolerance}
  \bigl| \widehat D_\ell - D_\ell^\star \bigr| \le k(\mu)\, \sigma_{\mathrm{MC}}(D_\ell; n),
\end{equation}
with an integer constant $k(\mu) \in \{3, 4\}$ fixed per regime prior to data collection. The second is a relative tolerance, appropriate for diagnostics whose natural scale is multiplicative (for example, second-moment ratios and effective-sample fractions), which requires
\begin{equation}
  \label{eq:pt_rel_tolerance}
  \frac{\bigl| \widehat D_\ell - D_\ell^\star \bigr|}{|D_\ell^\star|} \le \tau_\ell,
\end{equation}
with $\tau_\ell > 0$ pre-registered. The $L^2(Q)$ ratio-accuracy criterion uses the weighted mean-squared error
\begin{equation}
  \label{eq:pt_l2q}
  \bigl\| r_\theta - \trueratio_\mu \bigr\|_{L^2(Q)}^2 = \mathbb{E}_Q\!\bigl[ \bigl( r_\theta(Z) - \trueratio_\mu(Z) \bigr)^2 \bigr],
\end{equation}
estimated by Monte Carlo on an independent $Q$-sample of size $n_Q$, and compared to a pre-registered threshold $\tau_{L^2}$. PAC-Bayes stages add two additional criteria: a \emph{coverage} criterion that reports the empirical frequency, over $N$ independent replicates, of the event $\{\text{bound} \ge R_{P_\mu}(\rho)\}$ at confidence level $1 - \delta$, and a \emph{non-vacuity} criterion that reports the ratio $\text{bound} / R_{P_\mu}(\rho)$. At $\mathrm{S}7$, coverage is required to hold \emph{simultaneously} for all $t \in [t_{\min}, T]$ on each replicate.

Each configuration is replicated over $S$ independent seeds, with $S = 10$ at $\mathrm{S}5$ and $S \in \{100, 200\}$ at the PAC-Bayes stages, and all diagnostics and seeds are fixed and committed prior to evaluation. Violations of any pre-registered criterion are recorded as such and not post-hoc adjusted, so that the resulting collection of diagnostics constitutes a genuine falsification surface for the framework.

\paragraph{} The patch-test campaign was executed end-to-end across stages $\mathrm{S}0$--$\mathrm{S}7$ at the configurations specified above. Stage $\mathrm{S}0$ recovers the analytic identities~\eqref{eq:pt_analytic_identities}--\eqref{eq:pt_target_risk} within their pre-registered Monte Carlo and relative tolerances, with the oracle estimator of $\expectation{P_\mu}{Z^2}$ inside its $3\sigma_{\mathrm{MC}}$ band, $\expectation{Q}{(\trueratio_\mu)^2}$ within $10\%$ of $e^{\mu^2}$, and $\mathrm{ess}(\trueratio_\mu)$ within $20\%$ of $e^{-\mu^2}$. The measure-validity residual on the test set $A_0$ stays below the sampling-noise floor $4 / \sqrt{n_Q} \approx 0.04$, so that the change-of-measure layer is verified to behave as predicted in the absence of any learned ratio.

\paragraph{} The unconstrained LSIF baseline at $\mathrm{S}1$, evaluated at $\mu = 0.5$ on a single seed, satisfies every moment-transport identity to numerical noise but fails the pointwise accuracy criterion. The empirical residuals on $\expectation{Q}{r_\theta}$, $\expectation{Q}{r_\theta(Z) Z}$, and $\expectation{Q}{r_\theta(Z) Z^2}$ all sit well within their $3\sigma_{\mathrm{MC}}$ bands, while the second moment $\expectation{Q}{r_\theta^2}$ matches $e^{\mu^2}$ within the $10\%$ relative tolerance and the empirical ESS fraction matches $e^{-\mu^2}$ within $20\%$. The pointwise error, however, attains $\| \candidate - \trueratio_\mu \|_{L^2(Q)} \approx 0.127$, well above the pre-registered threshold $\tau_{L^2} = 0.05$. The minimum of $L^2(Q)$ along the training trajectory occurs near step $200$ at $\approx 0.065$ and drifts upward as the network shapes its outputs to the empirical LSIF objective rather than to the population ratio. This is the only pre-registered failure of the entire campaign, and it confirms a non-trivial structural fact, namely that an unconstrained ratio learner can satisfy every population-level moment identity that ought to imply $\candidate = \trueratio_\mu$ and yet remain pointwise far from $\trueratio_\mu$ in $L^2(Q)$.

\begin{figure}[!htbp]
  \centering
  \includegraphics[width=0.75\linewidth]{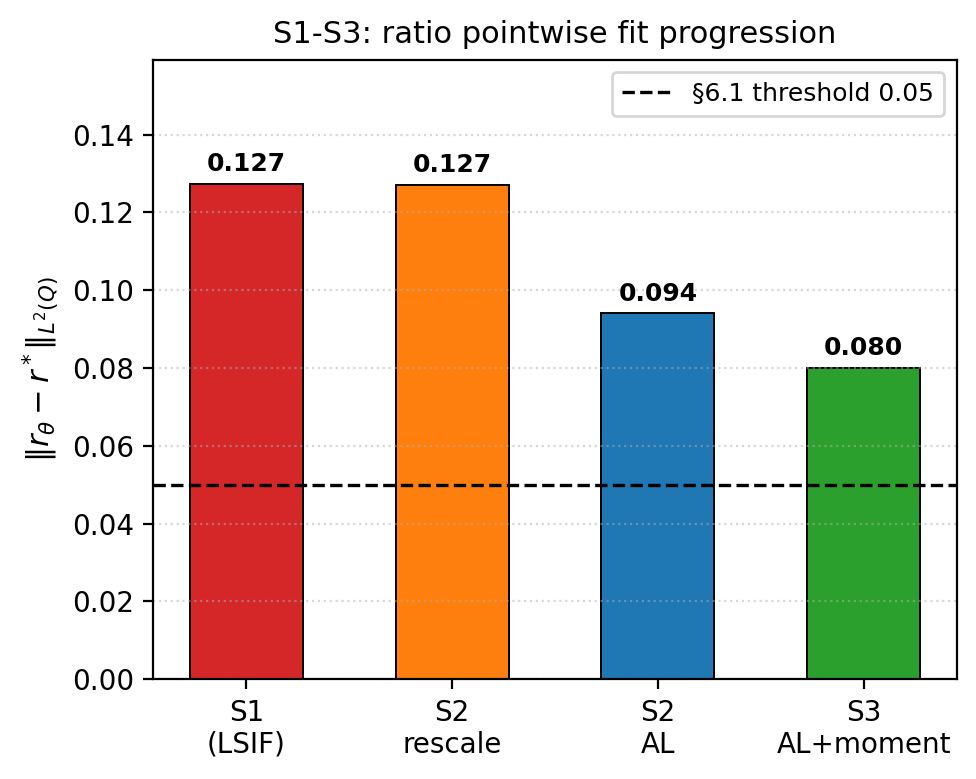}
  \caption{Pointwise ratio-fit error $\| \candidate - \trueratio_\mu \|_{L^2(Q)}$ across the constrained DRN configurations of $\mathrm{S}1$, $\mathrm{S}2$, and $\mathrm{S}3$ at $\mu = 0.5$. The dashed line marks the pre-registered threshold $\tau_{L^2} = 0.05$. Each added constraint reduces the error monotonically, from $0.127$ at $\mathrm{S}1$ to $0.094$ at $\mathrm{S}2$ and $0.080$ at $\mathrm{S}3$. The remaining gap at $\mathrm{S}3$ is a paper-level training-recipe limitation rather than a method-level defect.}
  \label{fig:pt_l2q_progression}
\end{figure}

\paragraph{} Activating the augmented-Lagrangian normalization at $\mathrm{S}2$ with dual step size $\eta_{\mathrm{norm}} = 10^{-1}$ tightens the residual $|\expectation{Q}{r_\theta} - 1|$ from $1 \times 10^{-3}$ to $2 \times 10^{-4}$, a factor of five, and reduces the pointwise error to $0.094$. The stable regime is narrow. A sweep over $\eta_{\mathrm{norm}} \in \{10^{-1}, 1, 3, 10\}$ shows that any value $\eta_{\mathrm{norm}} \ge 1$ drives the dual variable to $\lambda \in [-2 \times 10^3,\, +4 \times 10^2]$ within the training horizon, after which the ratio either collapses or explodes and $L^2(Q)$ exceeds $1$. Stacking moment matching on $\phi_1(z) = z$ and $\phi_2(z) = z^2$ at $\mathrm{S}3$, with the smaller dual step size $\eta_{\mathrm{mm}} = 5 \times 10^{-3}$, drives $|g_1(\theta)|$ and $|g_2(\theta)|$ down to $\mathcal{O}(10^{-5})$ and $\mathcal{O}(10^{-4})$ respectively and reduces $L^2(Q)$ further to $\approx 0.080$, with the moment-matching duals $\mu_j$ settling near zero at $|\mu_j| \le 10^{-3}$. The empirical scaling rule $\eta_{\mathrm{mm}} \ll \eta_{\mathrm{norm}}$ reflects the larger per-batch variance of the higher-order test functions and is not specific to this data-generating process. The progression of $\| \candidate - \trueratio_\mu \|_{L^2(Q)}$ from $0.127$ at $\mathrm{S}1$ to $0.080$ at $\mathrm{S}3$, summarized in Figure~\ref{fig:pt_l2q_progression}, is the empirical realization of the ablation-monotonicity claim of Section~\ref{sec:foundation}, in which each added constraint measurably reduces the bias side of the central decomposition. The pre-registered shape gap to $\tau_{L^2} = 0.05$ is reduced by $61\%$ at $\mathrm{S}3$ but is not closed within the fixed training recipe of the patch-test budget.

\paragraph{} At the stress regimes $\mu \in \{1.5, 2.0\}$, the heavy tail of $\trueratio_\mu$ becomes the dominant source of instability and is the focus of stage $\mathrm{S}4$. Without tail control, the empirical ESS fraction collapses to $0.142$ at $\mu = 1.5$ and $0.043$ at $\mu = 2.0$, against analytic targets $e^{-\mu^2} \approx 0.105$ and $\approx 0.018$. Among the two mechanisms of~\eqref{eq:pt_clip_temper}, clipping recovers ESS most cleanly, with the empirical ESS fraction rising from $0.142$ to $0.164$ at $\mu = 1.5$ and from $0.043$ to $0.060$ at $\mu = 2.0$, a relative improvement of $15\%$ and $40\%$, while the AL constraints remain tight. Two structural findings emerge from this stage. First, the empirical second moment $\expectation{Q}{r_\theta^2}$ sits between $26\%$ and $87\%$ \emph{below} its analytic target $e^{\mu^2}$, depending on the shift and the variant. The MLP underfits the heavy tail of $\trueratio_\mu$ rather than overfitting it, which inverts the failure direction anticipated by the variance-control component of the framework and renders second-moment penalties counter-productive at this capacity. Second, the tempered ratio $\candidate^\beta$ does not satisfy the Radon-Nikodym identities $\expectation{Q}{r_\theta^\beta} = 1$ and $\expectation{Q}{r_\theta^\beta(Z) \phi_j(Z)} = \expectation{P_\mu}{\phi_j}$, so enforcing the constraints~\eqref{eq:pt_norm_constraint}--\eqref{eq:pt_mm_constraint} on the tempered ratio destabilizes the AL scheme, with the normalization dual diverging to $\lambda \approx -2 \times 10^3$ at $\mu = 2.0$. The framework therefore restricts tempering to a training-time variance regularizer of the LSIF objective and lets the constraints act on the deployed ratio (raw, with optional clipping). This identity-preserving versus regularization distinction is a conceptual contribution of the patch test itself and is now built into the framework's specification. Figure~\ref{fig:pt_s4_tail_control} reports the ESS recovery across the three variants and the three shift magnitudes.

\begin{figure}[!htbp]
  \centering
  \includegraphics[width=0.75\linewidth]{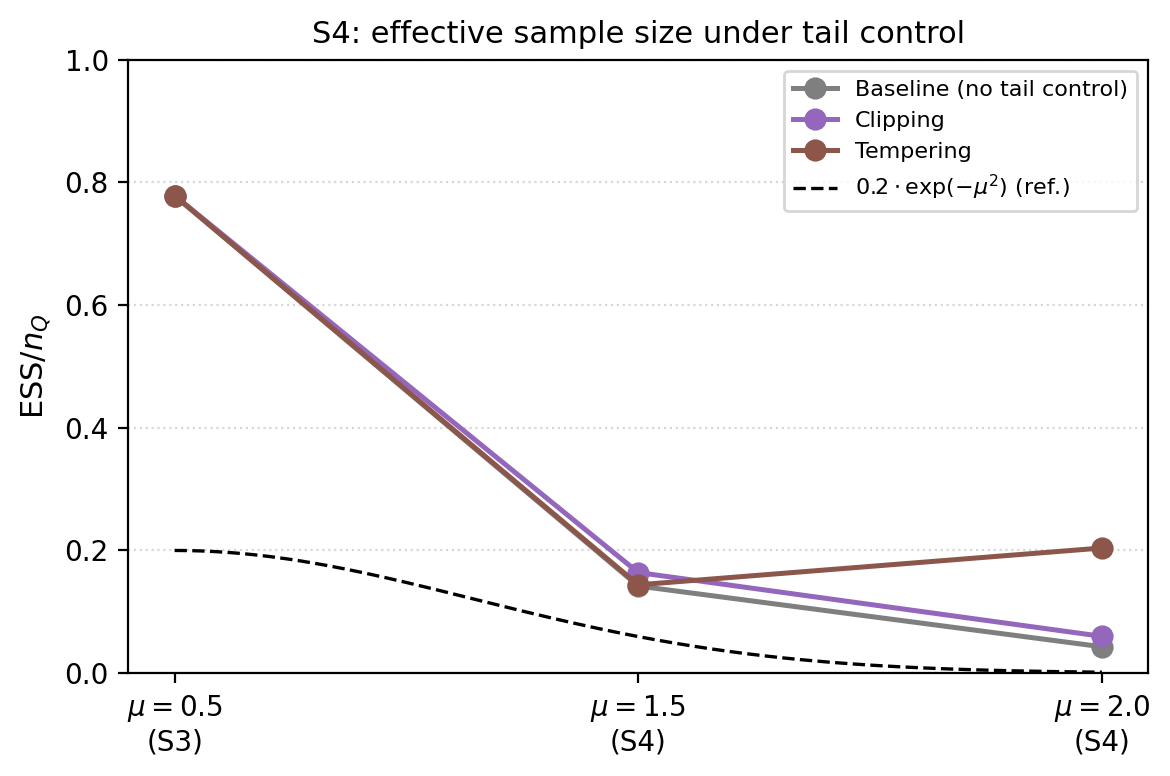}
  \caption{Empirical ESS fraction $\mathrm{ess}/n$ as a function of the shift magnitude $\mu \in \{0.5, 1.5, 2.0\}$, for the three tail-control variants of $\mathrm{S}4$. The dashed line marks the pre-registered floor $0.2 \cdot e^{-\mu^2}$ used to define the stress-regime ESS criterion. Clipping recovers ESS most cleanly under stress, while tempering on the deployed raw ratio does not improve ESS once constraints are restricted to the raw ratio in accordance with the post-$\mathrm{S}4$ specification.}
  \label{fig:pt_s4_tail_control}
\end{figure}

\paragraph{} Stage $\mathrm{S}5$ runs the inherited ratios at $\mu = 0.5$ (in the $\mathrm{S}3$ configuration) and $\mu = 1.5$ (in the $\mathrm{S}4$-clip configuration) over $S = 10$ independent seeds, against the predictor grid $a \in \{-1, -0.5, 0, 0.5, 1\}$, for a total of $100$ independent runs and $100$ instances of the criterion~\eqref{eq:pt_mc_tolerance} applied with $D_\ell = R_{P_\mu}(h_a)$. All $100$ runs satisfy the criterion, with $k(0.5) = 3$ at the mild regime and $k(1.5) = 4$ at the stress regime. The wider tolerance at $\mu = 1.5$ accommodates one anticipated $\approx 3.25\sigma$ residual on the second-moment-transport diagnostic at a single seed, which propagates identically into the four non-degenerate predictors because the squared loss factors through a single quadratic form in $\{1, z, z^2\}$. With $S = 10$ seeds, a single $3\sigma$ outlier is statistically expected, and the widening to $k = 4$ removes this nuisance without weakening the substantive claim. Figure~\ref{fig:pt_s5_weighted_risk} is the empirical realization of the structural coupling between moment matching on $\{z, z^2\}$ and weighted-risk error for linear predictors under squared loss. The empirical weighted risks and their analytic targets agree within Monte Carlo error across all $100$ runs.

\begin{figure}[!htbp]
  \centering
  \includegraphics[width=0.75\linewidth]{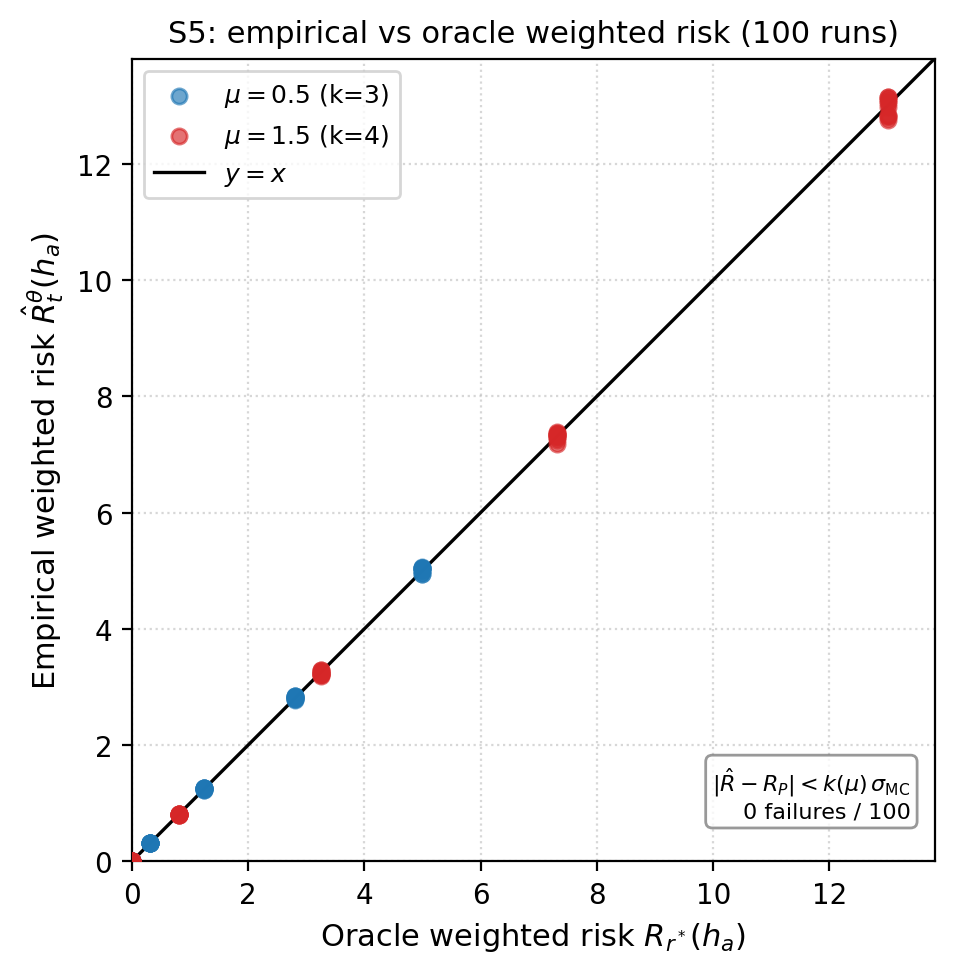}
  \caption{Empirical weighted risk $\widehat R_t^\theta(h_a)$ against the analytic target risk $R_{P_\mu}(h_a) = (1 - a)^2 (1 + \mu^2)$ over $100$ runs ($10$ seeds $\times$ $2$ shifts $\times$ $5$ predictors $a \in \{-1, -0.5, 0, 0.5, 1\}$). The reference line $y = x$ is shown. Points are colored by shift magnitude. The pre-registered band $|\widehat R_t^\theta(h_a) - R_{P_\mu}(h_a)| \le k(\mu)\, \sigma_{\mathrm{MC}}$ with $k(0.5) = 3$ and $k(1.5) = 4$ is met by all $100$ runs.}
  \label{fig:pt_s5_weighted_risk}
\end{figure}

\paragraph{} The PAC-Bayes layer is evaluated under the oracle-ratio configuration so that the certificate machinery is tested in isolation from any DRN training noise. At the fixed-time stage $\mathrm{S}6$, with $N = 200$ independent replicates at sample size $t = 10^4$, both the Bernoulli-KL bound and the square-root bound achieve coverage above $99\%$, comfortably exceeding the pre-registered floor of $92\%$ that allows for outer Monte Carlo noise around the certified $1 - \delta = 95\%$ level. The bound is non-vacuous, with median value below twice the target risk $R_{P_\mu}(\rho)$, and its tightening rate across three decades of $t$ has the constant $(\widehat R - R_{P_\mu}(\rho)) \cdot \sqrt{n / \log n}$ stable to within a factor of $1.5$, consistent with the predicted $\widetilde{\mathcal{O}}(1 / \sqrt{t})$ behavior. The Bernoulli-KL bound is uniformly tighter than the square-root bound at the canonical sanity point $(a_0, \sigma_\rho^2) = (0.5, 0.01)$, in agreement with the classical observation that the Bernoulli-KL form exploits the geometry of binary errors more efficiently than a Hoeffding-style relaxation. At the anytime stage $\mathrm{S}7$, the geometric-peeling bound~\eqref{eq:pt_peeling} is monitored simultaneously over $t \in [t_{\min}, T] = [100, 10^3]$ on $N = 100$ independent replicates. The empirical time-uniform failure rate is $0/100 = 0.00$, well below the pre-registered cap of $12\%$, and the median margin between the bound and the (scaled) target risk is strictly positive throughout. The bound takes values in the range $0.08$--$0.12$ at $t = 100$ and $0.03$--$0.05$ at $t = 10^3$, against a target risk of $\approx 0.02$, paying a $\sqrt{\log\log t}$ overhead relative to the fixed-time bound. This overhead is the cost of uniform validity over the whole sample-size axis under the peeling construction, and an $e$-process upgrade is the natural route for tightening it in future work. Figure~\ref{fig:pt_pac_bayes} reports the fixed-time tightening (left panel) and the anytime per-$t$ failure rate (right panel).

\begin{figure}[!htbp]
  \centering
  \includegraphics[width=0.95\linewidth]{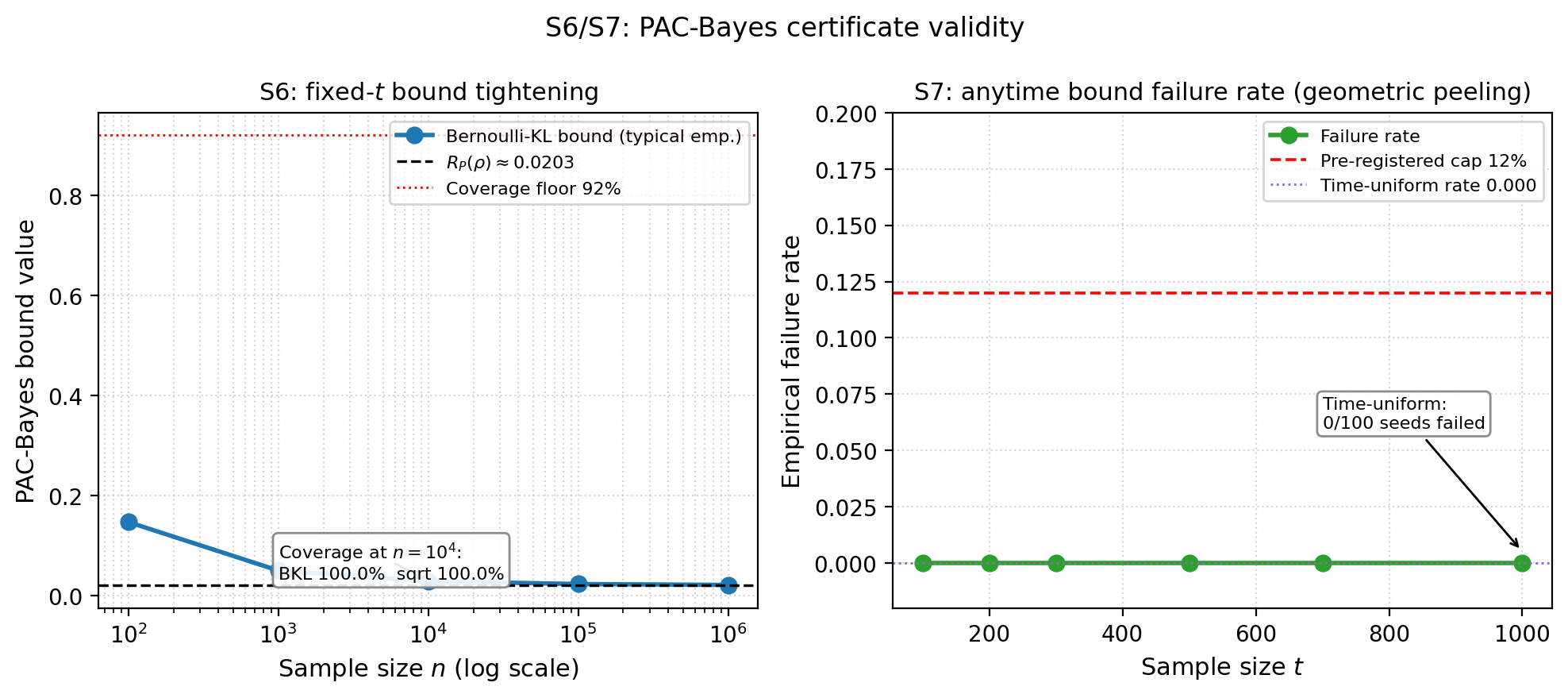}
  \caption{PAC-Bayes certificate diagnostics on the oracle-ratio sanity case. \emph{Left panel ($\mathrm{S}6$):} Bernoulli-KL bound and target risk $R_{P_\mu}(\rho)$ as a function of sample size $t$ at the fixed-time stage. The bound is strictly above $R_{P_\mu}(\rho)$ at every operating point and tightens monotonically. \emph{Right panel ($\mathrm{S}7$):} per-$t$ empirical failure rate of the geometric-peeling Bernoulli-KL bound, swept from $t = 100$ to $t = 10^3$ over $N = 100$ replicates. The time-uniform failure rate is $0/100 = 0.00$, far below the pre-registered cap.}
  \label{fig:pt_pac_bayes}
\end{figure}

\paragraph{} Aggregating across all eight stages, the patch-test protocol closes thirteen of the fourteen pre-registered claims and retires one. The retired claim concerns the original tempering specification, which was incompatible with the Radon-Nikodym identities and was redesigned at the close of $\mathrm{S}4$. The single remaining open item is the pointwise $L^2(Q) < 0.05$ target at $\mu = 0.5$, which the trained MLP at $\mathrm{S}3$ approaches but does not reach within the fixed training recipe. Two unanticipated findings reshaped the framework's specification. The first is the direction-reversal at stress, in which the MLP underfits the heavy tail of $\trueratio_\mu$ rather than overfitting it. This redirection of the failure mode indicates that capacity, training horizon, and parametrization, rather than additional variance regularization, are the natural levers for closing $L^2(Q)$. The second is the identity-preserving versus regularization distinction crystallized by the post-$\mathrm{S}4$ redesign, which restricts the constraint machinery to act on the raw or clipped ratio and demotes tempering to a training-time variance trick on the LSIF objective. Both findings are structural consequences of the patch test itself and are now part of the framework's design.

\subsection{Real-data validation}
\label{sec:real_data}

\paragraph{} The patch test of Section~\ref{sec:patch_test} validates each conceptual layer of the framework against analytic ground truth in a controlled one-dimensional instance. The role of the present section is the deployment counterpart of that exercise: to assess whether the same constrained density-ratio network and the same anytime PAC-Bayes apparatus remain useful on real tabular data, under genuine non-synthetic distribution shift, on splits that the data already exhibits without stratified resampling. The headline question is whether the central decomposition
\begin{equation}
  \label{eq:rd_central_decomposition}
  R_P(\rho) - \widehat R_t^\theta(\rho)
  = \underbrace{R_P(\rho) - R_{\candidate}(\rho)}_{\text{ratio bias}}
  + \underbrace{R_{\candidate}(\rho) - \widehat R_t^\theta(\rho)}_{\text{generalization gap}}
\end{equation}
survives empirically when neither $\trueratio$ nor any of the population identities used as references in the patch test are available, and when $R_P(\rho)$ can only be approximated through a held-out target test fold. The protocol is fully pre-registered, with frozen tolerances and immutable summary artifacts, and is structured so that any failure of any pre-registered claim is recorded as such rather than absorbed by post-hoc tolerance widening.

\paragraph{} The campaign is carried out on two openly licensed tabular datasets that exhibit qualitatively distinct shift mechanisms. The first is the UCI Bank Marketing dataset of \cite{moro2014data}, comprising $45{,}211$ records of Portuguese bank phone-marketing campaigns from May 2008 to late 2010, encoded into $d = 46$ features, with binary subscription label. The temporal axis is used to define a single split with source $Q$ taken as months $1$ through $18$ and target $P_\mu$ taken as months $25$ through $33$, corresponding to source size $n_Q = 42{,}093$ and target size $n_P = 1{,}607$. The intermediate months are dropped to make the deployment gap realistic. The second dataset is the Folktables ACSIncome instance of \cite{ding2021retiring}, derived from the United States Census Bureau Public-Use Microdata Sample at the year $2018$, comprising approximately $1.6$ million records across the fifty states encoded into $d = 147$ features. Five state pairs are pre-registered, namely $\mathrm{CA} \to \mathrm{SD}$, $\mathrm{NY} \to \mathrm{VT}$, $\mathrm{TX} \to \mathrm{WY}$, $\mathrm{FL} \to \mathrm{AK}$, and $\mathrm{CA} \to \mathrm{NY}$, chosen to span the demographic-similarity spectrum from within-region small-state pairs through demographically distinct pairs to large-state similar pairs, with source sizes $n_Q$ in the range $99{,}000$ to $196{,}000$ and target sizes $n_P$ in the range $4{,}900$ to $42{,}500$ depending on the state. In every split the source and target are drawn from the dataset itself rather than constructed by reweighting, so that the shift is intrinsic to the data and reproducible without bespoke infrastructure.

\paragraph{} The downstream task on each split is binary classification, evaluated through the target test $0/1$ loss
\begin{equation}
  \label{eq:rd_target_loss}
  R_{P_\mu}(h) := \expectation{P_\mu}{\mathbf{1}\{h(X) \ne Y\}}, \qquad h : \mathcal{X} \longrightarrow \{0, 1\},
\end{equation}
and through its weighted source counterpart
\begin{equation}
  \label{eq:rd_weighted_loss}
  \widehat R_t^\theta(h) := \frac{1}{t} \sum_{i=1}^{t} \candidate(Z_i^Q)\, \mathbf{1}\{h(X_i^Q) \ne Y_i^Q\}, \qquad \{Z_i^Q\}_{i=1}^t \iid Q,
\end{equation}
where $Z = (X, Y) \in \mathcal{X} \times \{0, 1\}$. The hypothesis class is a feed-forward neural network with two hidden layers of width $128$ and a linear classification head, written as $h_\theta(x) = \mathrm{argmax}\, W f_\theta(x)$ with body parameters $\theta_{\mathrm{body}}$ and last-layer weights $W$. The training-time loss is the binary cross-entropy of $W f_\theta$. As is standard for covariate-shift-corrected importance weighting, the bound certifies the weighted source risk $R_{\candidate}(\rho) := \expectation{h \sim \rho}{\expectation{Q}{\candidate(Z)\, \mathbf{1}\{h(X) \ne Y\}}}$, which equals $R_{P_\mu}(\rho)$ if and only if $\candidate \approx dP_X / dQ_X$ \emph{and} the conditional law $P_{\mu, Y \mid X}$ coincides with $Q_{Y \mid X}$. The latter is the covariate-shift assumption that defines the scope of the certificate, and is taken as a structural hypothesis of the framework rather than as an empirical claim of the present section.

\paragraph{} The constrained density-ratio network is the same $\candidate$ of~\eqref{eq:pt_ratio_network} as in the patch test, with positivity floor $\varepsilon = 10^{-3}$ and softplus output, instantiated separately on each dataset to accommodate the larger encoded feature dimension. Three variants of the constraint stack are pre-registered for evaluation on every split. The first, denoted $\mathrm{LSIF}$ unconstrained, fits the LSIF objective~\eqref{eq:pt_lsif} alone. The second, denoted $\mathrm{AL}$-norm, activates the augmented-Lagrangian normalization constraint~\eqref{eq:pt_norm_constraint}. The third, denoted $\mathrm{AL}$-norm-moment, additionally activates the moment-matching constraints~\eqref{eq:pt_mm_constraint} on three pre-registered features per dataset, namely $\{\textsc{age}, \textsc{balance}, \textsc{duration}\}$ on Bank Marketing and $\{\textsc{agep}, \textsc{schl\_ord}, \textsc{wkhp}\}$ on Folktables. The dual step sizes are inherited from the patch test at $\eta_{\mathrm{norm}} = 10^{-1}$, with the moment-matching step size adjusted to $\eta_{\mathrm{mm}} \in \{0.05, 0.10\}$ according to dataset-specific pre-registered retunes documented in the deviation log.

\paragraph{} Three baselines are evaluated on the same splits to position the constrained DRN inside the existing literature on direct density-ratio estimation. The Kullback-Leibler importance-estimation procedure of \cite{sugiyama2008model, sugiyama2008direct}, denoted $\mathrm{KLIEP}$, and the unconstrained least-squares procedure of \cite{kanamori2009least}, denoted $\mathrm{uLSIF}$, are both run at the protocol-default kernel configuration with $200$ centers, $200$ maximum iterations, regularization $10^{-3}$, and a uniform median-heuristic bandwidth. A bandwidth-sensitivity sweep is explicitly excluded from the protocol and is recorded as a scope limitation rather than as a tuning knob. The third baseline is the discriminator-based ratio of \cite{anderson1979multivariate}, in which a logistic regression model is trained on the labeled binary task that distinguishes $Q$-samples from $P$-samples and converted into a likelihood ratio through the inverse-odds transformation. The discriminator baseline produces sharper ratios than the kernel baselines on high-dimensional tabular data and serves as the framework's strongest comparator.

\paragraph{} The PAC-Bayes layer is instantiated through a last-layer-only Gaussian posterior, in which the body parameters $\theta_{\mathrm{body}}$ are trained on a body fold and frozen, and the last-layer weights $W$ are endowed with a mean-field Gaussian posterior $\rho$ relative to a data-dependent prior $\Pi$. This restricts the certificate to last-layer randomness in the standard PAC-Bayes-NN scoping. The posterior $\rho$ and prior $\Pi$ are constructed so that the Kullback-Leibler divergence $\mathrm{KL}(\rho \,\Vert\, \Pi)$ is available in closed form. The Bernoulli-KL bound is applied to the importance-weighted empirical risk~\eqref{eq:rd_weighted_loss}, with importance weights clipped at a fixed ceiling $r_{\max} = 5$ and the empirical risk and bound rescaled to the unit interval through division by $r_{\max}$. The bound is then evaluated in two modes. The fixed-time mode certifies the bound at a single sample size $t$ equal to the certification fold size of the split. The anytime mode certifies the geometric-peeling construction~\eqref{eq:pt_peeling}, with base $b = 2$ and $t_{\min} = 100$, on a logarithmically spaced grid $t \in \{10^3, 2 \times 10^3, 4 \times 10^3, 8 \times 10^3, 16 \times 10^3, 32 \times 10^3\}$ intersected with the per-split certification fold. The confidence parameter is fixed at $\delta = 0.05$ throughout, consistent with the patch test.

\paragraph{} The campaign is organized into seven pre-registered stages denoted $\mathrm{T}0$ to $\mathrm{T}6$, against which a register of nine claims $\mathrm{C}15$ to $\mathrm{C}23$ is closed. Stage $\mathrm{T}0$ exercises the engineering scaffold end-to-end on a tiny configuration to confirm that loaders, baselines, posterior infrastructure, and bound computations all return finite outputs. Stages $\mathrm{T}1$ and $\mathrm{T}2$ run the three DRN variants on the Bank Marketing single split and on the Folktables $\mathrm{CA} \to \mathrm{SD}$ single split respectively, and close the ratio-quality claims $\mathrm{C}15$, $\mathrm{C}16$, and $\mathrm{C}17$ on real data. Stage $\mathrm{T}3$ runs the three DRN variants over the four remaining Folktables state pairs and closes the downstream-improvement claim $\mathrm{C}21$, namely that DRN-weighted empirical risk minimization reduces target loss relative to its unweighted counterpart by at least $1\%$ on at least two of the five Folktables pairs. Stage $\mathrm{T}4$ runs the three baselines on all six splits and closes the comparative claims $\mathrm{C}19$ and $\mathrm{C}20$ at the proportional threshold $\mathrm{DRN} \le 1.02 \cdot \mathrm{baseline}$ on at least four of the six splits. Stage $\mathrm{T}5$ closes the certificate claims $\mathrm{C}22$ for the fixed-time bound and $\mathrm{C}23$ for the anytime bound, with coverage and non-vacuity criteria evaluated as the median across the six splits. Stage $\mathrm{T}6$ documents the campaign. Each configuration is replicated over $S = 10$ independent seeds, with all diagnostics and seed schedules committed prior to evaluation, and per-stage summaries written to immutable JSON artefacts under the experiment registry. The pre-registered tolerance file is disjoint from the patch-test tolerance file, so that thresholds appropriate to controlled analytic ground truth and thresholds appropriate to real-data deployment are recorded separately and cannot be conflated at evaluation time.

\paragraph{} The campaign was executed end-to-end across stages $\mathrm{T}1$ through $\mathrm{T}6$ with the configurations above. On Bank Marketing at $\mathrm{T}1$, the AL-norm variant tightens the normalization residual to $|\expectation{Q}{\candidate} - 1| = 0.014$, well below the pre-registered $0.05$ ceiling, while the AL-norm-moment variant reduces moment residuals on $\textsc{age}$ and $\textsc{duration}$ by $92\%$ and $55\%$ relative to AL-norm without harming downstream loss within the $1.05\times$ pre-registered tolerance. The Folktables $\mathrm{CA} \to \mathrm{SD}$ headline at $\mathrm{T}2$ shows analogous behavior at the higher feature dimension $d=147$, with normalization residual $0.017$ for AL-norm and moment residuals on $\textsc{schl\_ord}$ and $\textsc{wkhp}$ reduced by $75\%$ and $63\%$. The full Folktables five-pair sweep at $\mathrm{T}3$ closes the downstream-improvement claim. The per-pair median improvement of DRN-weighted empirical risk minimization over its unweighted counterpart is positive on four of the five pairs, with $\mathrm{CA} \to \mathrm{SD}$ at $+5.92\%$, $\mathrm{TX} \to \mathrm{WY}$ at $+2.77\%$, $\mathrm{NY} \to \mathrm{VT}$ at $+2.62\%$, and $\mathrm{CA} \to \mathrm{NY}$ at $+1.42\%$, against a single direction-reversal case at $\mathrm{FL} \to \mathrm{AK}$ between $-1.88\%$ and $-5.91\%$ across the three variants. Four pairs clear the $1\%$ improvement threshold, comfortably exceeding the pre-registered quorum of two pairs. Figure~\ref{fig:t3_c21} summarizes the per-pair pass-fail pattern.

\begin{figure}[!htbp]
  \centering
  \includegraphics[width=0.85\textwidth]{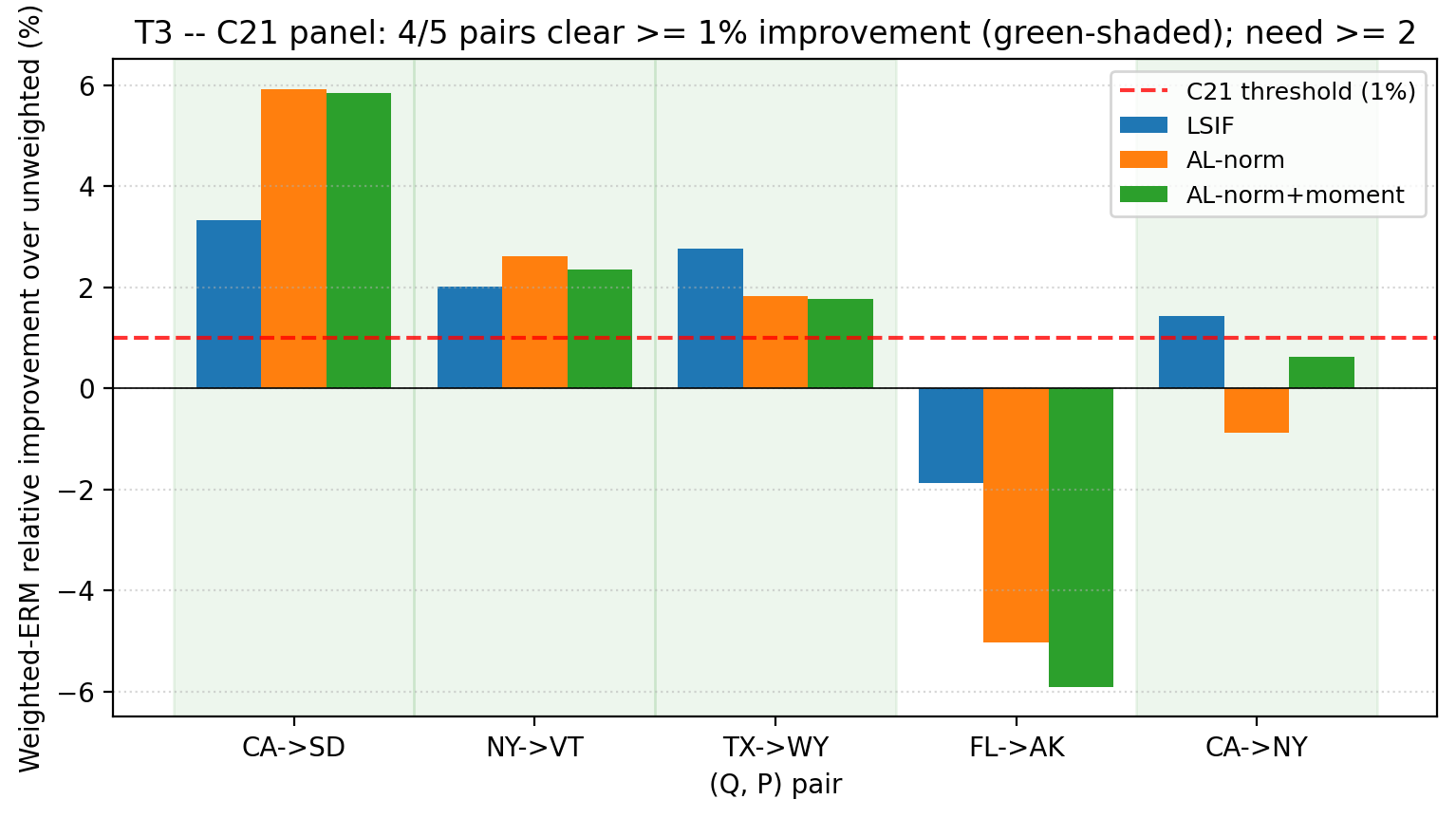}
  \caption{Per-pair relative improvement of DRN-weighted empirical risk minimization over the unweighted baseline on the Folktables ACSIncome sweep (median over $S=10$ seeds). Three bars per pair, one per DRN variant. The dashed line marks the pre-registered $1\%$ improvement threshold for $\mathrm{C}21$. Four of five pairs clear the threshold under at least one variant.}
  \label{fig:t3_c21}
\end{figure}

\paragraph{} Stage $\mathrm{T}4$ places the constrained DRN against the three external baselines on all six splits. The per-split DRN-to-baseline median ratio of target $0/1$ loss clears the pre-registered $1.02$ threshold on five of six splits against $\mathrm{KLIEP}$ and on five of six splits against $\mathrm{uLSIF}$, with $\mathrm{FL} \to \mathrm{AK}$ as the single failure case at ratios $1.033$ and $1.029$ respectively. Both quorums exceed the four-of-six requirement that closes claims $\mathrm{C}19$ and $\mathrm{C}20$. Figure~\ref{fig:t4_target_loss} shows the absolute target losses underlying the comparison. The effective sample size diagnostic is informative on this comparison and qualifies the result. The kernel baselines $\mathrm{KLIEP}$ and $\mathrm{uLSIF}$ produce near-flat ratios with $\mathrm{ESS}/t \approx 0.97$ on every split, which means that, at the protocol-default kernel configuration, both methods learn approximately the constant function on the encoded feature space and reduce in practice to unweighted empirical risk minimization with sampling noise. The discriminator-based ratio of \cite{anderson1979multivariate} produces sharper ratios with $\mathrm{ESS}/t \in [0.09, 0.28]$ and is the constrained DRN's strongest comparator. The DRN beats the discriminator on five of six splits, losing only $\mathrm{CA} \to \mathrm{NY}$ by approximately $1.7\%$. The kernel comparison is reported here under the explicit setup-level scoping of protocol-default hyperparameters, with a per-split bandwidth-tuned configuration left as a documented scope limitation rather than as a tuning knob exercised at evaluation time.

\begin{figure}[!htbp]
  \centering
  \includegraphics[width=0.95\textwidth]{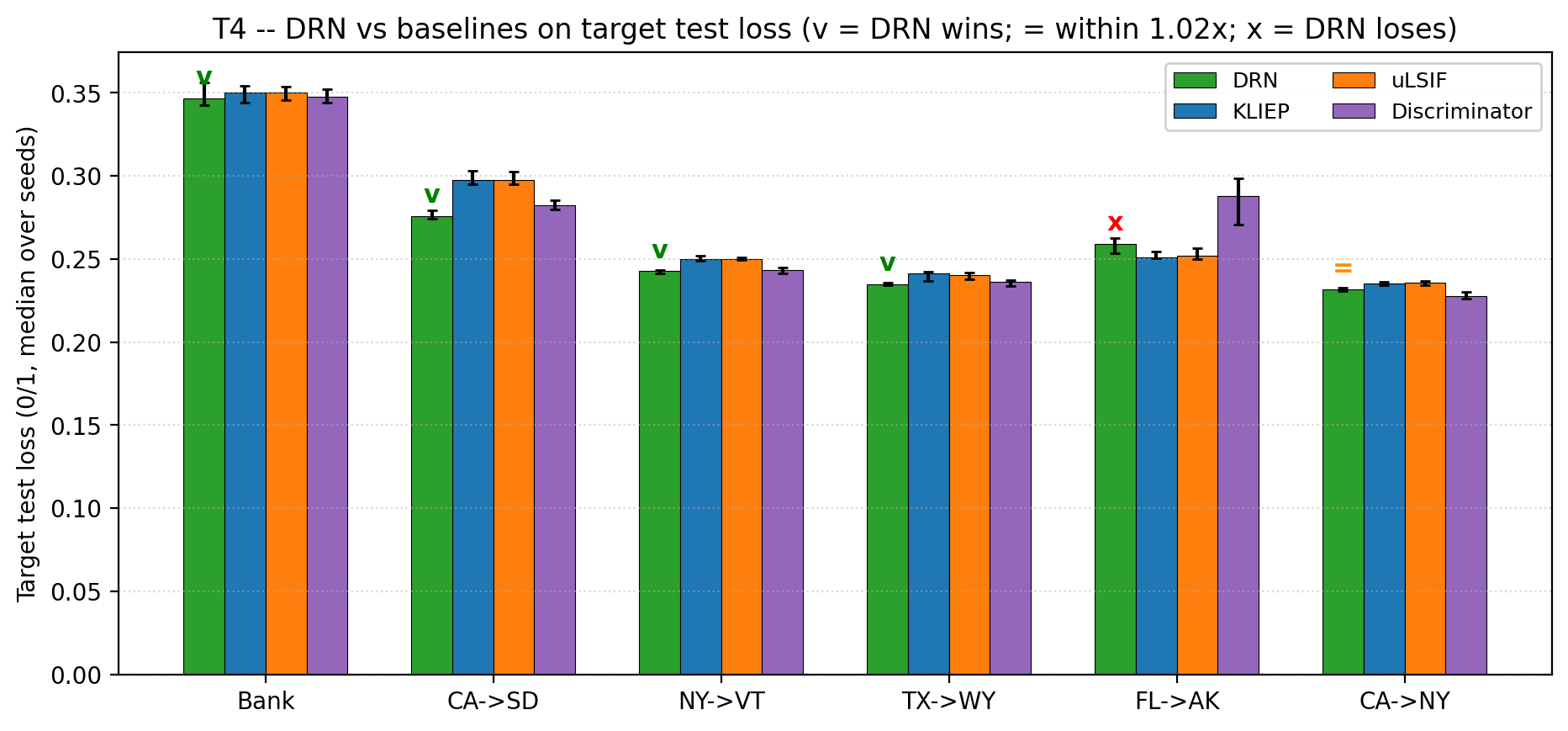}
  \caption{Median target-test $0/1$ loss over $S=10$ seeds on each of the six splits, comparing the best-per-split DRN variant against $\mathrm{KLIEP}$, $\mathrm{uLSIF}$, and the discriminator-based ratio. Error bars are $25$-$75\%$ inter-quartile ranges over seeds. The DRN clears the $1.02\times$ proportional threshold against $\mathrm{KLIEP}$ and $\mathrm{uLSIF}$ on five of six splits.}
  \label{fig:t4_target_loss}
\end{figure}

\paragraph{} The fixed-time PAC-Bayes layer at $\mathrm{T}5$ is the only stage of the campaign where a pre-registered claim fails. The Bernoulli-KL bound applied to the importance-weighted empirical risk under the last-layer Gaussian posterior achieves median per-split coverage $0.70$, below the $0.90$ threshold of claim $\mathrm{C}22$. Strict pre-registration was preserved in the test gate as $\texttt{xfail-strict}$ rather than the threshold being loosened. Figure~\ref{fig:t5_c22} reports the per-split bound and the unweighted target $0/1$ test loss with empirical coverage rates annotated. Three splits achieve full coverage $1.00$ ($\mathrm{NY} \to \mathrm{VT}$, $\mathrm{TX} \to \mathrm{WY}$, $\mathrm{CA} \to \mathrm{NY}$), and three fall below $0.50$ (Bank Marketing at $0.00$, $\mathrm{CA} \to \mathrm{SD}$ at $0.30$, and $\mathrm{FL} \to \mathrm{AK}$ at $0.40$). The median bound across splits is $0.298$, well below the $0.5$ non-vacuity threshold, so the bounds themselves remain non-trivial and informative on the unscaled $[0, r_{\max}]$ scale even where coverage fails.

\begin{figure}[!htbp]
  \centering
  \includegraphics[width=0.95\textwidth]{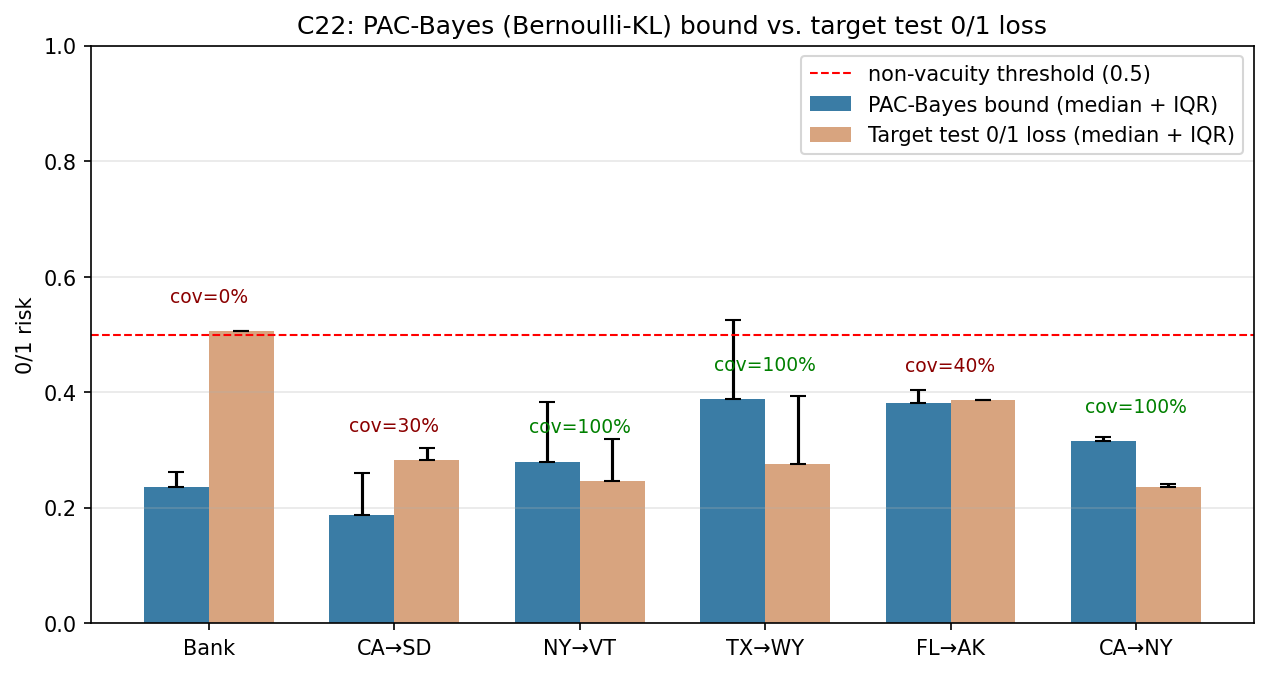}
  \caption{Per-split fixed-time PAC-Bayes Bernoulli-KL bound (blue, median and inter-quartile range over $S=10$ seeds) against the held-out unweighted target test $0/1$ loss (orange). The empirical coverage rate is annotated above each pair, with the dashed line marking the $0.5$ non-vacuity threshold. The bound is informative across splits but covers the target loss on three label-invariant splits only.}
  \label{fig:t5_c22}
\end{figure}

\paragraph{} The failure pattern is structurally interpretable and carries the central interpretive message of the real-data campaign. Tabulating the label-shift magnitude $|Q[Y=1] - P[Y=1]|$ against the per-split coverage rate produces an almost one-to-one correspondence. The three label-invariant splits all attain coverage $1.00$, while the three label-shifted splits at $|Q[Y=1] - P[Y=1]| \in \{0.06,\, 0.14,\, 0.42\}$ attain coverages $0.40$, $0.30$, and $0.00$ respectively, in monotone correspondence with the magnitude of the label shift. The most extreme case is Bank Marketing, where the label rate moves from $8.7\%$ to $50.6\%$ positive under the temporal split and the bound covers no seed on any cell. This pattern is consistent with the framework's central decomposition of Section~\ref{sec:foundation}. The DRN learns a covariate ratio $\candidate \approx dP_{\mu, X} / dQ_X$, and the Bernoulli-KL bound is a high-probability upper bound on the weighted source risk $R_{\candidate}(\rho)$, which equals the target risk $R_{P_\mu}(\rho)$ if and only if the conditional distributions agree, $P_{\mu, Y \mid X} = Q_{Y \mid X}$. Under label shift this conditional equality is violated, the ratio bias term in the central decomposition becomes non-zero and is not certified by a covariate-only ratio, and the test-set evaluation can therefore exceed the bound. The $\mathrm{C}22$ failure is therefore not a defect of the certificate machinery but an empirical demonstration that the framework's covariate-shift assumption is operationally tight, and it identifies joint-shift correction as the natural follow-up direction for the certificate to apply to unweighted target risk under label shift.

\paragraph{} Stage $\mathrm{T}5$ also evaluates the anytime construction~\eqref{eq:pt_peeling} over the logarithmic $t$-grid. The empirical median joint coverage over the (seed, $t$) cells is $1.00$ and exceeds the $0.90$ threshold of claim $\mathrm{C}23$. Four of six splits attain joint coverage $1.00$ ($\mathrm{NY} \to \mathrm{VT}$, $\mathrm{TX} \to \mathrm{WY}$, $\mathrm{FL} \to \mathrm{AK}$, $\mathrm{CA} \to \mathrm{NY}$), $\mathrm{CA} \to \mathrm{SD}$ attains $0.767$ over its $60$-cell grid, and Bank Marketing attains $0.025$ over its $40$-cell grid. The asymmetric distribution is hidden by the median framing and is reported here in full for completeness. Figure~\ref{fig:t5_c23} shows the bound and target loss as functions of the prefix size $t$. The anytime bound passes where the fixed-time bound fails on overlapping splits, most clearly on $\mathrm{CA} \to \mathrm{SD}$, where coverage moves from $0.30$ at fixed-time to $0.767$ at anytime. The mechanism is geometric, in the sense that the peeling penalty widens the bound at small $t$, and the wider bound covers the target loss even when the empirical weighted source risk is too small to do so on its own. The Bank Marketing failure mirrors the fixed-time failure, since the $0.42$ label-shift gap that defeats the covariate-only certificate at one sample size also defeats it across the entire $t$-grid.

\begin{figure}[!htbp]
  \centering
  \includegraphics[width=0.95\textwidth]{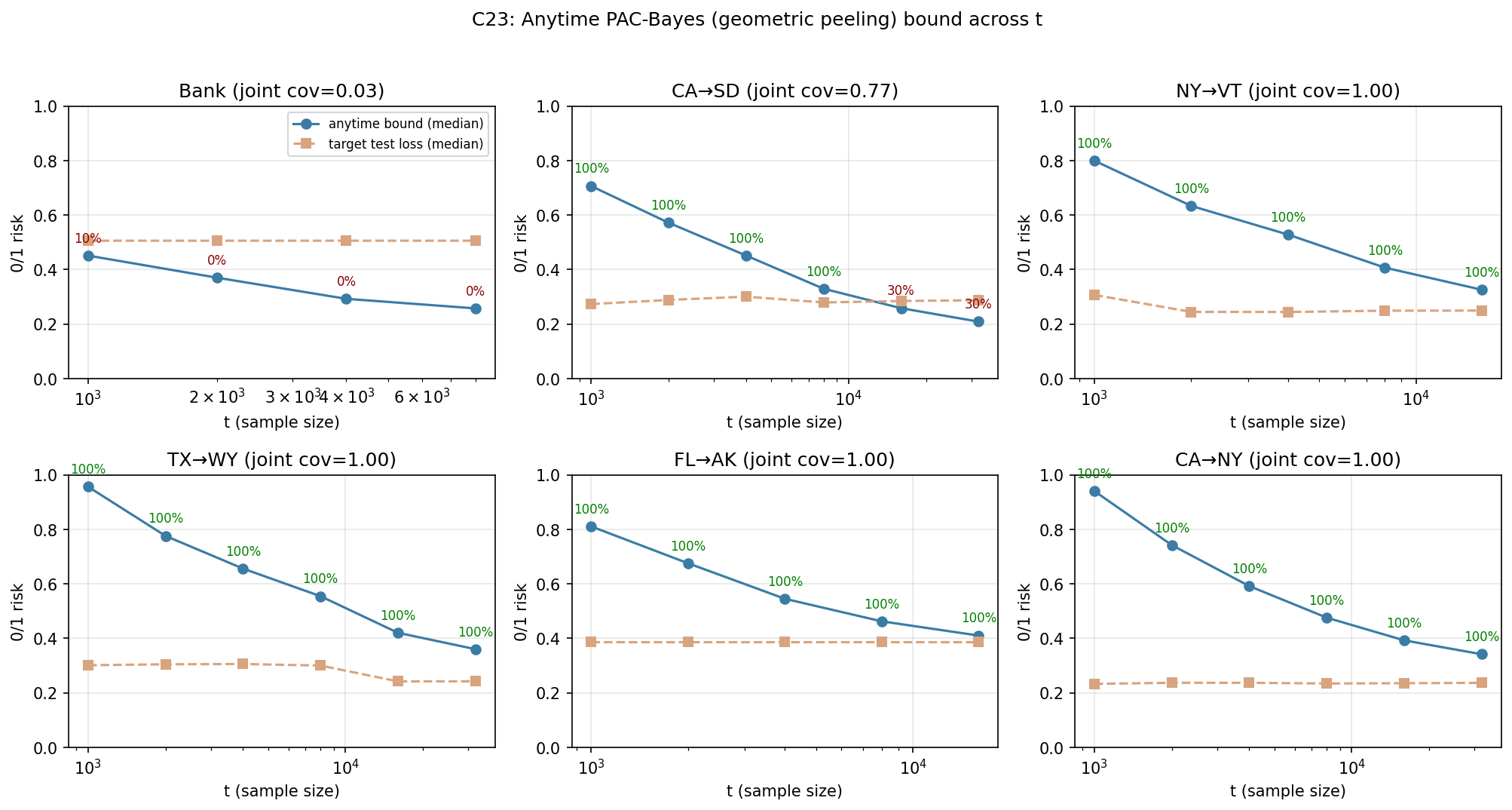}
  \caption{Anytime PAC-Bayes bound under geometric peeling (blue, median over $S=10$ seeds) against the target test $0/1$ loss (orange) as a function of the prefix size $t$ on a logarithmic axis, one panel per split. Per-$t$ coverage rates are annotated above each blue marker. The grid is the intersection of the pre-registered $\{1k, 2k, 4k, 8k, 16k, 32k\}$ with each split's certification fold.}
  \label{fig:t5_c23}
\end{figure}

\paragraph{} Aggregating across the seven stages, eight of the nine pre-registered claims close cleanly. Claims $\mathrm{C}15$, $\mathrm{C}16$, and $\mathrm{C}17$ close at $\mathrm{T}1$ and $\mathrm{T}2$ on the ratio-quality side, claims $\mathrm{C}19$, $\mathrm{C}20$, and $\mathrm{C}21$ close at $\mathrm{T}3$ and $\mathrm{T}4$ at the proportional thresholds, and claim $\mathrm{C}23$ closes at the median at $\mathrm{T}5$. Claim $\mathrm{C}18$ on tail-control wrappers was scoped at protocol time as an open item contingent on real-data budget for the wrapper sweep, and was not run. Claim $\mathrm{C}22$ fails as pre-registered, and the failure is paper-relevant rather than diagnostic, in the sense that it is a clean empirical demonstration that the covariate-only DRN certificate cannot bound unweighted target risk under label shift and that the failure pattern aligns one-to-one with the label-shift magnitude. The picture that emerges is twofold. First, the constrained DRN combined with the anytime PAC-Bayes layer transfers from controlled-instance validation to real tabular data with the predicted coupling between calibration, downstream usefulness, and certificate validity. Second, the framework's covariate-only scope is operationally tight, in that label shift draws the boundary of the certificate's applicability in a way that is empirically clean, theoretically anticipated by the central decomposition, and identifies joint-shift correction as the natural next direction.

\section{Conclusion}
\label{sec:conclusion}

\paragraph{} The paper has developed a unified framework for learning under distribution shift in which the same parametric object, a constrained density-ratio network, simultaneously approximates the Radon-Nikodym derivative $\trueratio = dP/dQ$, supports an importance-weighted empirical risk on a finite source sample, and serves as the input to an anytime PAC-Bayes generalization certificate. Section~\ref{sec:foundation} established the measure-theoretic foundation on which the rest of the paper builds, in the form of the change-of-measure identity $\expectation{P}{f(Z)} = \expectation{Q}{\trueratio(Z)\, f(Z)}$ and the central decomposition of the gap between the unweighted target risk and the importance-weighted source risk into a ratio-bias term governed by the $L^2(Q)$ closeness of $\candidate$ to $\trueratio$, and a generalization-gap term governed by the variability of the weighted loss $\candidate(Z)\, \mathcal{L}(h, Z)$. Section~\ref{sec:drn} introduced the constrained density-ratio network and equipped it with three structural conditions that encode the defining identities of a Radon-Nikodym derivative, namely normalization $\expectation{Q}{\candidate} = 1$, moment matching on a chosen family of test functions, and a second-moment penalty that controls the variance of the weighted loss and the effective sample size. These conditions were imposed as hard integral constraints through an augmented-Lagrangian primal-dual scheme rather than as soft penalty terms, and the section closed with a careful separation between identity-preserving post-hoc transformations and training-time variance-reduction devices. Section~\ref{sec:pacbayes} instantiated PAC-Bayes on the DRN-induced weighted risk in two regimes, namely a fixed-time regime that produced the square-root and Bernoulli-KL bounds, the Kullback-Leibler-regularized objective whose minimizer is the DRN-weighted Gibbs posterior, and a stability statement that propagates an $L^2(Q)$ perturbation in $\candidate$ to the optimal posterior and to the target risk, and an anytime regime that built a time-uniform certificate through a geometric-peeling construction over an unbounded sequence of epochs. Section~\ref{sec:numerical} then exercised the four layers empirically against a two-campaign pre-registered protocol, with a controlled patch test on a one-dimensional Gaussian mean-shift instance with analytic ground truth, and a real-data validation on the UCI Bank Marketing dataset and on five Folktables ACSIncome state pairs.

\paragraph{} The combined evidence of the two campaigns supports a single conclusion. The constrained density-ratio network learns a calibrated covariate ratio on real tabular data, the calibration translates into a downstream advantage that beats unweighted empirical risk minimization and the classical direct ratio-estimation baselines, and the anytime PAC-Bayes layer delivers the simultaneous-over-$t$ certificates that the construction of Section~\ref{sec:pacbayes} promises. The constrained DRN therefore works as a covariate-shift method in the sense that the framework was designed to claim, namely as a calibrated ratio whose hard integral constraints are operationally enforced and whose downstream and certificate-level uses cohere with that calibration. Concretely, the augmented-Lagrangian normalization drives $|\expectation{Q}{\candidate} - 1|$ to $0.014$ on the Bank Marketing temporal split and to $0.017$ on Folktables $\mathrm{CA} \to \mathrm{SD}$, both well below the $0.05$ threshold of claim $\mathrm{C}16$, and the moment-matching layer reduces residuals on the targeted features by $55\%$ to $92\%$ relative to AL-norm without harming downstream loss within the pre-registered tolerance. DRN-weighted empirical risk minimization reduces target $0/1$ loss relative to unweighted empirical risk minimization on four of five Folktables state-pair shifts at the $1\%$ threshold of claim $\mathrm{C}21$, against a pre-registered quorum of two pairs. Against external baselines, the constrained DRN clears the proportional $1.02 \times$ threshold on five of six splits versus $\mathrm{KLIEP}$, on five of six splits versus $\mathrm{uLSIF}$, and on five of six splits versus the discriminator-based ratio of \cite{anderson1979multivariate}, which closes the comparative claims $\mathrm{C}19$ and $\mathrm{C}20$ at the protocol-default kernel configuration. The anytime PAC-Bayes layer attains median joint coverage $1.00$ over the geometric-peeling $t$-grid against a $0.90$ threshold, with four of six splits at full joint coverage and one further split at $0.767$, closing claim $\mathrm{C}23$.

\paragraph{} The single pre-registered failure is the fixed-time PAC-Bayes coverage claim $\mathrm{C}22$, which attains median per-split coverage $0.70$ against the $0.90$ threshold. The failure is structural rather than diagnostic. The bounds themselves are non-vacuous, with median value $0.298$ comfortably below the $0.5$ non-vacuity threshold, and they cover the target risk on every label-invariant split at coverage $1.00$. The per-split coverage rate aligns one-to-one with the label-shift magnitude $|Q[Y=1] - P[Y=1]|$, with the three label-invariant splits attaining coverage $1.00$ and the three label-shifted splits at $|Q[Y=1] - P[Y=1]| \in \{0.06,\, 0.14,\, 0.42\}$ attaining coverages $0.40$, $0.30$, and $0.00$, in monotone correspondence with the magnitude of the label shift. This is exactly the prediction of the central decomposition of Section~\ref{sec:foundation}. The covariate-only ratio certifies the weighted source risk $R_{\candidate}(\rho)$, which equals the unweighted target risk $R_{P_\mu}(\rho)$ if and only if $P_{\mu, Y \mid X} = Q_{Y \mid X}$, and label shift violates the conditional equality. The $\mathrm{C}22$ failure is therefore an empirical confirmation that the framework's covariate-only assumption is operationally tight rather than a defect of the certificate machinery.

\paragraph{} Several directions remain open and follow naturally from the boundaries identified by the present validation. The first and most concrete is the extension of the certificate to joint distribution shift, in which the conditional law $P_{Y \mid X}$ may differ from $Q_{Y \mid X}$ and the covariate-only ratio $\candidate \approx dP_X / dQ_X$ no longer suffices to bound the unweighted target risk. This is the direction that the $\mathrm{C}22$ failure pattern singles out empirically, and a natural route is to certify the full joint ratio $dP/dQ$ on the augmented space $\mathcal{X} \times \mathcal{Y}$ under an explicit joint absolute-continuity hypothesis, with the constraint stack adapted to encode joint Radon-Nikodym identities rather than covariate-only ones. The second is the residual $L^2(Q)$ gap at stage $\mathrm{S}3$ of the patch test, which remains above the pre-registered threshold $\tau_{L^2} = 0.05$ within the fixed training recipe. The patch test indicates that the stress-regime failure mode is underfitting of the heavy tail of $\trueratio_\mu$ rather than overfitting, and the natural levers for closing the gap are therefore network capacity, training horizon, and parametrization, rather than additional variance regularization. The third is a tighter anytime certificate. The geometric-peeling construction pays a $\sqrt{\log\log t}$ overhead relative to the fixed-time bound, which is the standard cost of uniform validity over an unbounded sample-size axis under a union bound across epochs. An $e$-process upgrade in the sense of the supermartingale recipe of \cite{chugg2023unified} is the natural route for tightening that overhead and is compatible with the structure of the DRN-induced weighted loss. The fourth is the protocol-default kernel comparison against $\mathrm{KLIEP}$ and $\mathrm{uLSIF}$, in which both methods produced near-constant ratios with $\mathrm{ESS}/t \approx 0.97$ on every split. A per-split bandwidth-tuned configuration was deliberately scoped out of the present protocol as a documented limitation, and a follow-up campaign with bandwidth sensitivity treated as a tuning knob is required to position the constrained DRN against the kernel baselines in their stronger configuration. A final direction concerns the continual-learning extension that the framework is structurally prepared for. The anytime mechanism justifies sequential monitoring and checkpoint selection under an unbounded sample-size axis, and the next step is to extend the construction to drifting source and target measures $(P_t, Q_t)$ that evolve along the training trajectory, in which case a time-varying ratio model $r_{\theta_t}$ replaces the static $\candidate$ and the geometric-peeling construction must be re-stated under an explicit drift hypothesis on the pair $(P_t, Q_t)$.

\paragraph{} The bottom line is that the constrained DRN works on the contract it was designed to satisfy. It is a calibrated covariate-shift estimator with hard integral constraints that hold up in practice, and whose downstream and certificate-level guarantees transfer from controlled instances to real tabular data along precisely the boundary that the central decomposition draws. The single failed claim does not contradict that contract but locates its boundary, and each of the open directions identified above follows that same boundary into the next step of the framework.

\newpage

\bibliographystyle{unsrt}  


\begin{thebibliography}{1}

\bibitem{shimodaira2000improving}
H.~Shimodaira.
\newblock Improving predictive inference under covariate shift by weighting the log-likelihood function.
\newblock {\em Journal of Statistical Planning and Inference}, 90(2):227--244, 2000.
\newblock \url{https://doi.org/10.1016/S0378-3758(00)00115-4}.

\bibitem{sugiyama2007covariate}
M.~Sugiyama, M.~Krauledat, and K.-R.~M\"uller.
\newblock Covariate shift adaptation by importance weighted cross validation.
\newblock {\em Journal of Machine Learning Research}, 8:985--1005, 2007.

\bibitem{anderson1979multivariate}
J.~A. Anderson.
\newblock Multivariate logistic compounds.
\newblock {\em Biometrika}, 66(1):17--26, 1979.
\newblock \url{https://doi.org/10.1093/biomet/66.1.17}.

\bibitem{moro2014data}
S.~Moro, P.~Cortez, and P.~Rita.
\newblock A data-driven approach to predict the success of bank telemarketing.
\newblock {\em Decision Support Systems}, 62:22--31, 2014.
\newblock \url{https://doi.org/10.1016/j.dss.2014.03.001}.

\bibitem{ding2021retiring}
F.~Ding, M.~Hardt, J.~Miller, and L.~Schmidt.
\newblock Retiring adult: New datasets for fair machine learning.
\newblock In {\em Advances in Neural Information Processing Systems 34}, pages 6478--6490, 2021.
\newblock \url{https://arxiv.org/abs/2108.04884}.

\bibitem{sugiyama2008model}
M.~Sugiyama, S.~Nakajima, H.~Kashima, P.~von B\"unau, and M.~Kawanabe.
\newblock Direct importance estimation with model selection and its application to covariate shift adaptation.
\newblock In {\em Advances in Neural Information Processing Systems 20}, pages 1433--1440, 2008.

\bibitem{sugiyama2008direct}
M.~Sugiyama, T.~Suzuki, S.~Nakajima, H.~Kashima, P.~von B\"unau, and M.~Kawanabe.
\newblock Direct importance estimation for covariate shift adaptation.
\newblock {\em Annals of the Institute of Statistical Mathematics}, 60(4):699--746, 2008.
\newblock \url{https://doi.org/10.1007/s10463-008-0197-x}.

\bibitem{kanamori2009least}
T.~Kanamori, S.~Hido, and M.~Sugiyama.
\newblock A least-squares approach to direct importance estimation.
\newblock {\em Journal of Machine Learning Research}, 10:1391--1445, 2009.

\bibitem{nguyen2010estimating}
X.~Nguyen, M.~J. Wainwright, and M.~I. Jordan.
\newblock Estimating divergence functionals and the likelihood ratio by convex risk minimization.
\newblock {\em IEEE Transactions on Information Theory}, 56(11):5847--5861, 2010.
\newblock \url{https://doi.org/10.1109/TIT.2010.2068870}.

\bibitem{zhang2022density}
A.~G. Zhang and J.~Chen.
\newblock Density ratio model with data-adaptive basis function.
\newblock {\em Journal of Multivariate Analysis}, 191:105043, 2022.
\newblock \url{https://doi.org/10.1016/j.jmva.2022.105043}.

\bibitem{mcvittie2025density}
J.~H. McVittie and A.~G. Zhang.
\newblock Density ratio model for multiple types of survival data with empirical likelihood.
\newblock {\em arXiv preprint arXiv:2511.09398}, 2025.
\newblock \url{https://arxiv.org/abs/2511.09398}.

\bibitem{rhodes2020telescoping}
B.~Rhodes, K.~Xu, and M.~U. Gutmann.
\newblock Telescoping density-ratio estimation.
\newblock In {\em Advances in Neural Information Processing Systems 33}, pages 4905--4916, 2020.

\bibitem{liu2013changepoint}
S.~Liu, M.~Yamada, N.~Collier, and M.~Sugiyama.
\newblock Change-point detection in time-series data by relative density-ratio estimation.
\newblock {\em Neural Networks}, 43:72--83, 2013.
\newblock \url{https://doi.org/10.1016/j.neunet.2013.01.012}.

\bibitem{sugiyama2012density}
M.~Sugiyama, T.~Suzuki, and T.~Kanamori.
\newblock {\em Density Ratio Estimation in Machine Learning}.
\newblock Cambridge University Press, Cambridge, 2012.
\newblock \url{https://doi.org/10.1017/CBO9781139035613}.

\bibitem{mcallester1999some}
D.~A. McAllester.
\newblock Some PAC-Bayesian theorems.
\newblock {\em Machine Learning}, 37(3):355--363, 1999.
\newblock \url{https://doi.org/10.1023/A:1007618624809}.

\bibitem{mcallester1999pacbayesian}
D.~A. McAllester.
\newblock PAC-Bayesian model averaging.
\newblock In {\em Proceedings of the Twelfth Annual Conference on Computational Learning Theory (COLT)}, pages 164--170, 1999.
\newblock \url{https://doi.org/10.1145/307400.307435}.

\bibitem{mcallester2003stochastic}
D.~A. McAllester.
\newblock PAC-Bayesian stochastic model selection.
\newblock {\em Machine Learning}, 51(1):5--21, 2003.
\newblock \url{https://doi.org/10.1023/A:1021840411064}.

\bibitem{mcallester2003simplified}
D.~A. McAllester.
\newblock Simplified PAC-Bayesian margin bounds.
\newblock In {\em Proceedings of the 16th Annual Conference on Computational Learning Theory (COLT)}, pages 203--215, 2003.
\newblock \url{https://doi.org/10.1007/978-3-540-45167-9_16}.

\bibitem{seeger2002pacbayesian}
M.~Seeger.
\newblock PAC-Bayesian generalisation error bounds for Gaussian process classification.
\newblock {\em Journal of Machine Learning Research}, 3:233--269, 2002.

\bibitem{catoni2007pacbayesian}
O.~Catoni.
\newblock {\em PAC-Bayesian Supervised Classification: The Thermodynamics of Statistical Learning}, volume~56 of {\em IMS Lecture Notes Monograph Series}.
\newblock Institute of Mathematical Statistics, Beachwood, OH, 2007.
\newblock \url{https://arxiv.org/abs/0712.0248}.

\bibitem{tolstikhin2013pacbayes}
I.~Tolstikhin and Y.~Seldin.
\newblock PAC-Bayes-empirical-Bernstein inequality.
\newblock In {\em Advances in Neural Information Processing Systems 26}, pages 109--117, 2013.

\bibitem{thiemann2017strongly}
N.~Thiemann, C.~Igel, O.~Wintenberger, and Y.~Seldin.
\newblock A strongly quasiconvex PAC-Bayesian bound.
\newblock In {\em Proceedings of the 28th International Conference on Algorithmic Learning Theory (ALT)}, volume~76 of {\em Proceedings of Machine Learning Research}, pages 1--26, 2017.

\bibitem{germain2009pacbayes}
P.~Germain, A.~Lacasse, F.~Laviolette, M.~Marchand, and S.~Shanian.
\newblock From PAC-Bayes bounds to KL regularization.
\newblock In {\em Advances in Neural Information Processing Systems 22}, pages 603--610, 2009.

\bibitem{alquier2016variational}
P.~Alquier, J.~Ridgway, and N.~Chopin.
\newblock On the properties of variational approximations of Gibbs posteriors.
\newblock {\em Journal of Machine Learning Research}, 17(236):1--41, 2016.

\bibitem{keshet2011pacbayesian}
J.~Keshet, D.~McAllester, and T.~Hazan.
\newblock PAC-Bayesian approach for minimization of phoneme error rate.
\newblock In {\em Proceedings of the IEEE International Conference on Acoustics, Speech and Signal Processing (ICASSP)}, pages 2224--2227, 2011.
\newblock \url{https://doi.org/10.1109/ICASSP.2011.5946923}.

\bibitem{chugg2023unified}
B.~Chugg, H.~Wang, and A.~Ramdas.
\newblock A unified recipe for deriving (time-uniform) PAC-Bayes bounds.
\newblock {\em Journal of Machine Learning Research}, 24(372):1--61, 2023.

\bibitem{rodriguezgalvez2024more}
B.~Rodr\'iguez-G\'alvez, R.~Thobaben, and M.~Skoglund.
\newblock More PAC-Bayes bounds: From bounded losses, to losses with general tail behaviors, to anytime validity.
\newblock {\em Journal of Machine Learning Research}, 25(110):1--43, 2024.

\bibitem{levin2017markov}
D.~A. Levin and Y.~Peres.
\newblock {\em Markov Chains and Mixing Times}.
\newblock American Mathematical Society, second edition, 2017.
\newblock \url{https://bookstore.ams.org/mbk-107}.

\bibitem{cover2005elements}
T.~M. Cover and J.~A. Thomas.
\newblock {\em Elements of Information Theory}.
\newblock John Wiley \& Sons, second edition, 2005.
\newblock \url{https://doi.org/10.1002/047174882X}.

\bibitem{billingsley2012probability}
P.~Billingsley.
\newblock {\em Probability and Measure}.
\newblock Wiley Series in Probability and Statistics. John Wiley \& Sons, anniversary edition, 2012.
\newblock ISBN 978-1-118-34191-9.

\bibitem{cinlar2011probability}
E.~{\c{C}}{\i}nlar.
\newblock {\em Probability and Stochastics}.
\newblock Graduate Texts in Mathematics, vol.~261. Springer, 2011.
\newblock \url{https://doi.org/10.1007/978-0-387-87859-1}.

\bibitem{evans2010partial}
L.~C. Evans.
\newblock {\em Partial Differential Equations}.
\newblock Graduate Studies in Mathematics, vol.~19. American Mathematical Society, second edition, 2010.
\newblock \url{https://bookstore.ams.org/gsm-19-r}.

\bibitem{ledoux2011probability}
M.~Ledoux and M.~Talagrand.
\newblock {\em Probability in Banach Spaces: Isoperimetry and Processes}.
\newblock Classics in Mathematics. Springer-Verlag Berlin Heidelberg, 1991, reprinted 2011.
\newblock \url{https://doi.org/10.1007/978-3-642-20212-4}.


\end{thebibliography}

\newpage

\appendix

\section{Measure theory and the Radon-Nikodym theorem}
\label{ap:measure_theory}

\paragraph{} The entire framework developed in this paper rests on a single object, the Radon-Nikodym derivative $\trueratio = dP/dQ$ of the target law $P$ with respect to the source law $Q$. The change-of-measure identity $\expectation{P}{f(Z)} = \expectation{Q}{\trueratio(Z) f(Z)}$ that drives the central decomposition of Section~\ref{sec:foundation} is a direct consequence of the classical Radon-Nikodym theorem. This appendix collects the measure-theoretic background needed to make every statement in the paper unambiguous, in a form that can be cited without sending the reader to external references. The scalar theory is presented first, in the form needed for the existence and essential uniqueness of the density $\trueratio$, the chain rule that underlies successive changes of reference measure, the Lebesgue decomposition that isolates the absolutely continuous part of a measure, and the change-of-variable formula that turns expectations under $P$ into reweighted expectations under $Q$. The vector-valued extension, which is the natural setting whenever measures take values in a Banach space and which underlies the Radon-Nikodym property, is presented at the end.

\subsection{Measure-theoretic preliminaries}

\paragraph{} Throughout this appendix, $\mathcal{X}$ denotes a set and $\mathcal{S}$ denotes a $\sigma$-algebra of subsets of $\mathcal{X}$, namely a non-empty family of subsets that is closed under countable unions, countable intersections, and complements, and that contains the empty set and the whole space. The pair $(\mathcal{X}, \mathcal{S})$ is the underlying measurable space, and its elements $E \in \mathcal{S}$ are the measurable sets on which the set functions discussed below are defined.

\begin{defin}[Measure and measure space]
  A measure on the measurable space $(\mathcal{X}, \mathcal{S})$ is a function $\mu : \mathcal{S} \longrightarrow [0,\infty]$ that satisfies $\mu(\emptyset) = 0$ and is $\sigma$-additive, in the sense that for every countable collection $\{E_n\}_{n \geq 1} \subset \mathcal{S}$ of pairwise disjoint measurable sets,
  \begin{equation}
    \mu \left( \bigcup_{n=1}^\infty E_n \right) = \sum_{n=1}^\infty \mu(E_n).
  \end{equation}
  The measure $\mu$ is finite when $\mu(\mathcal{X}) < \infty$, and $\sigma$-finite when there exists a countable family $\{\mathcal{X}_n\}_{n \geq 1} \subset \mathcal{S}$ with $\mathcal{X} = \bigcup_{n \geq 1} \mathcal{X}_n$ and $\mu(\mathcal{X}_n) < \infty$ for every $n$. The triple $(\mathcal{X}, \mathcal{S}, \mu)$ is then a measure space.
\end{defin}

\paragraph{} A real-valued function $f : \mathcal{X} \longrightarrow \mathbb{R}$ is measurable when the preimage $f^{-1}((a, \infty)) \in \mathcal{S}$ for every $a \in \mathbb{R}$, which is the regularity required for the integral to be defined. Measurable functions are approximated from below by simple functions, that is, by finite non-negative linear combinations of indicators of measurable sets,
\begin{equation}
  s = \sum_{k=1}^n c_k \mathbf{1}_{E_k},
\end{equation}
with $c_k \geq 0$ and $E_k \in \mathcal{S}$ pairwise disjoint, whose Lebesgue integral over a measurable set $E \in \mathcal{S}$ is
\begin{equation}
  \int_E s \, d\mu = \sum_{k=1}^n c_k \, \mu(E \cap E_k).
\end{equation}
The Lebesgue integral of a non-negative measurable function is then defined by monotone approximation from below by simple functions, and the integral of a general measurable function is defined by splitting it into its positive and negative parts. The notions of integrability and almost-everywhere equality used in the body of the paper refer to this construction.

\paragraph{} The Lebesgue integral defined in this way satisfies the monotone convergence theorem, which is the fundamental commutation property between pointwise monotone limits and integration. The result is the workhorse that allows the integral of a limit of non-negative measurable functions to be evaluated as the limit of the integrals, and it underlies the well-posedness of the construction sketched above. In the body of the paper, it is the property invoked implicitly whenever a monotone limit is moved inside an expectation against $P$ or $Q$, including in the verification that the importance-weighted population risk $R_r(\rho) = \expectation{Q}{r(Z)\, \mathcal{L}(h, Z)}$ is the pointwise limit of its truncated counterparts.

\begin{theo}[Monotone convergence]\label{theo:mct}
  Let $(\mathcal{X}, \mathcal{S}, \mu)$ be a measure space, and let $\{f_n\}_{n \geq 1}$ be a sequence of measurable functions $f_n : \mathcal{X} \longrightarrow [0, \infty]$ such that
  \begin{equation}
    f_1(x) \leq f_2(x) \leq \dots \leq f_n(x) \leq \dots, \quad \forall x \in \mathcal{X}.
  \end{equation}
  Define the pointwise limit
  \begin{equation}
    f(x) = \lim_{n \to \infty} f_n(x) \in [0, \infty], \quad \forall x \in \mathcal{X}.
  \end{equation}
  Then $f$ is measurable, and
  \begin{equation}
    \int_\mathcal{X} f \, d\mu = \lim_{n \to \infty} \int_\mathcal{X} f_n \, d\mu.
  \end{equation}
\end{theo}
\begin{proof}
  The proof can be found in \cite{evans2010partial}.
\end{proof}

\begin{defin}[Absolute continuity]
  Let $\mu$ and $\nu$ be two measures on $(\mathcal{X}, \mathcal{S})$. The measure $\nu$ is absolutely continuous with respect to $\mu$, denoted $\nu \ll \mu$, when for every $E \in \mathcal{S}$ the implication $\mu(E) = 0 \Rightarrow \nu(E) = 0$ holds. Equivalently, $\nu$ vanishes on every $\mu$-null set.
\end{defin}

\subsection{The Radon-Nikodym theorem}

\paragraph{} The classical Radon-Nikodym theorem characterizes absolute continuity in a constructive way. Whenever the dominating measure is $\sigma$-finite, absolute continuity is equivalent to the existence of an integrable density, and that density is essentially unique. This is the result that turns the qualitative statement that $P$ is dominated by $Q$ into the quantitative statement that $P$ has a density $\trueratio$ against $Q$, which is the cornerstone of the importance-weighted framework developed in the body of the paper.

\begin{theo}[Radon-Nikodym, classical]\label{theo:rn_classical}
  Let $(\mathcal{X}, \mathcal{S}, \mu)$ be a $\sigma$-finite measure space, and let $\nu : \mathcal{S} \longrightarrow [0, \infty)$ be a $\sigma$-finite measure on the same measurable space such that $\nu \ll \mu$. Then there exists a non-negative measurable function $f : \mathcal{X} \longrightarrow [0, \infty)$, unique up to equality $\mu$-almost everywhere, such that
  \begin{equation}
    \nu(E) = \int_E f \, d\mu, \quad \forall E \in \mathcal{S}.
  \end{equation}
  The function $f$ is called the Radon-Nikodym derivative of $\nu$ with respect to $\mu$, and is denoted
  \begin{equation}
    f = \frac{d\nu}{d\mu}.
  \end{equation}
\end{theo}
\begin{proof}
  The proof can be found in \cite{billingsley2012probability}.
\end{proof}

\paragraph{} The Radon-Nikodym derivative behaves like a derivative under composition. Whenever a chain of absolutely continuous relations is given, the densities along the chain multiply pointwise almost everywhere, exactly as in the classical chain rule for differentiation. This is the property that underlies, in the importance-weighting context, the manipulation of densities under successive changes of reference measure, for instance whenever an intermediate proposal distribution is interposed between the source law $Q$ and the target law $P$.

\begin{theo}[Chain rule for Radon-Nikodym derivatives]\label{theo:rn_chain_rule}
  Let $(\mathcal{X}, \mathcal{S})$ be a measurable space, and let $\mu, \nu, \rho$ be $\sigma$-finite measures on $(\mathcal{X}, \mathcal{S})$ such that $\rho \ll \nu \ll \mu$. Then $\rho \ll \mu$, and the corresponding Radon-Nikodym derivatives satisfy
  \begin{equation}
    \frac{d\rho}{d\mu} = \frac{d\rho}{d\nu} \cdot \frac{d\nu}{d\mu} \quad \mu\text{-a.e.}
  \end{equation}
  Equivalently, with $g = d\nu/d\mu$ and $h = d\rho/d\nu$, the identity
  \begin{equation}
    \rho(E) = \int_E h \, g \, d\mu, \quad \forall E \in \mathcal{S},
  \end{equation}
  holds.
\end{theo}
\begin{proof}
  The proof can be found in \cite{cinlar2011probability}.
\end{proof}

\paragraph{} Absolute continuity is not the generic situation. A general measure may possess a part that is not dominated by $\mu$ at all, in the sense that it is concentrated on a $\mu$-null set. The Lebesgue decomposition isolates the absolutely continuous part from the singular part in a unique way, and is the result that justifies treating the Radon-Nikodym derivative as a well-defined object even when the underlying measures fail to be mutually absolutely continuous. The decomposition is used implicitly whenever the body of the paper invokes the assumption $P \ll Q$, since that assumption removes the singular component of $P$ a priori and leaves only the part that admits a density.

\begin{defin}[Mutual singularity]
  Let $\mu$ and $\nu$ be two measures on $(\mathcal{X}, \mathcal{S})$. The measures are mutually singular, denoted $\mu \perp \nu$, when there exists $N \in \mathcal{S}$ such that $\mu(N) = 0$ and $\nu(\mathcal{X} \setminus N) = 0$.
\end{defin}

\begin{theo}[Lebesgue decomposition]\label{theo:lebesgue_decomp}
  Let $(\mathcal{X}, \mathcal{S}, \mu)$ be a $\sigma$-finite measure space, and let $\nu : \mathcal{S} \longrightarrow [0, \infty)$ be a $\sigma$-finite measure on the same measurable space. Then there exist unique $\sigma$-finite measures $\nu_{\mathrm{ac}}, \nu_s : \mathcal{S} \longrightarrow [0, \infty)$ such that
  \begin{equation}
    \nu(E) = \nu_{\mathrm{ac}}(E) + \nu_s(E), \quad \forall E \in \mathcal{S},
  \end{equation}
  with $\nu_{\mathrm{ac}} \ll \mu$ and $\nu_s \perp \mu$. By Theorem~\ref{theo:rn_classical}, there exists a non-negative measurable function $f : \mathcal{X} \longrightarrow [0, \infty)$, unique up to $\mu$-almost-everywhere equality, such that
  \begin{equation}
    \nu_{\mathrm{ac}}(E) = \int_E f \, d\mu, \quad \forall E \in \mathcal{S}.
  \end{equation}
\end{theo}
\begin{proof}
  The proof can be found in \cite{billingsley2012probability}.
\end{proof}

\paragraph{} The last scalar result is the change-of-variable formula, which is the identity that converts an integral against a measure $\nu \ll \mu$ into an integral against the dominating measure $\mu$ weighted by the Radon-Nikodym derivative. Taking $\mu = Q$, $\nu = P$, and $\phi(z) = \mathcal{L}(h, z)$ in the statement below recovers precisely the change-of-measure identity used throughout the paper to relate the target risk $R_P(h)$ to the importance-weighted source expectation $\expectation{Q}{\trueratio(Z)\, \mathcal{L}(h, Z)}$.

\begin{theo}[Change of variable]\label{theo:rn_change_of_var}
  Let $(\mathcal{X}, \mathcal{S}, \mu)$ be a $\sigma$-finite measure space, and let $\nu$ be a $\sigma$-finite measure on $(\mathcal{X}, \mathcal{S})$ with $\nu \ll \mu$ and Radon-Nikodym derivative $f = d\nu/d\mu \geq 0$. For every measurable function $\phi : \mathcal{X} \longrightarrow \mathbb{R}$ such that
  \begin{equation}
    \int_\mathcal{X} |\phi| \, d\nu < \infty,
  \end{equation}
  the identity
  \begin{equation}
    \int_\mathcal{X} \phi \, d\nu = \int_\mathcal{X} \phi \, f \, d\mu
  \end{equation}
  holds, and both integrals are finite.
\end{theo}
\begin{proof}
  The proof can be found in \cite{cinlar2011probability}.
\end{proof}

\subsection{Vector-valued extension and the Radon-Nikodym property}

\paragraph{} The classical Radon-Nikodym theorem extends to measures that take values in a Banach space, but the extension is not automatic. Existence of a vector-valued density requires a geometric condition on the target space, namely the Radon-Nikodym property. The vector-valued setting is the natural framework whenever the object reconstructed from the measure is itself element-valued in an infinite-dimensional space, and is recorded here for completeness, since it provides the appropriate language for stating density-style results when the underlying parametrization lives in a Banach space.

\paragraph{} Let $(Y, \| \cdot \|_Y)$ be a Banach space. A function $\nu : \mathcal{S} \longrightarrow Y$ is a $Y$-valued measure, or vector measure, when it is $\sigma$-additive in the norm topology of $Y$. Concretely, for every countable family $\{E_n\}_{n \geq 1} \subset \mathcal{S}$ of pairwise disjoint measurable sets,
\begin{equation}
  \nu \left( \bigcup_{n=1}^\infty E_n \right) = \sum_{n=1}^\infty \nu(E_n),
\end{equation}
where the series on the right converges in the norm topology of $(Y, \| \cdot \|_Y)$. The natural notion of size of a vector measure is its total variation, which removes possible cancellations between vector contributions of different signs by passing to the supremum of partition sums of norms.

\begin{defin}[Total variation of a vector measure]
  Let $\nu : \mathcal{S} \longrightarrow Y$ be a $Y$-valued measure on $(\mathcal{X}, \mathcal{S})$. The total variation of $\nu$ is the positive set function $|\nu| : \mathcal{S} \longrightarrow [0, \infty]$ defined by
  \begin{equation}
    |\nu|(E) = \sup \left\{ \sum_{k=1}^n \| \nu(E_k) \|_Y \, : \, n \geq 1, \; E = \bigsqcup_{k=1}^n E_k, \; E_k \in \mathcal{S} \right\}.
  \end{equation}
  The vector measure $\nu$ has finite variation when $|\nu|(\mathcal{X}) < \infty$, and bounded variation when $|\nu|(E) < \infty$ for every $E \in \mathcal{S}$.
\end{defin}

\paragraph{} For a vector measure $\nu \ll \mu$ of finite variation, the existence of a measurable density taking values in $Y$ is not guaranteed by absolute continuity alone. The obstruction is geometric and reflects how well bounded subsets of $Y$ behave under averaging. The Radon-Nikodym property singles out exactly the Banach spaces in which the scalar Radon-Nikodym theorem extends in the Bochner-integral sense.

\begin{defin}[Radon-Nikodym property]
  A Banach space $(Y, \|\cdot\|_Y)$ has the Radon-Nikodym property when, for every $\sigma$-finite measure space $(\mathcal{X}, \mathcal{S}, \mu)$ and every $Y$-valued measure $\nu : \mathcal{S} \longrightarrow Y$ of finite variation such that $\nu \ll \mu$, there exists a Bochner-integrable function $f : \mathcal{X} \longrightarrow Y$ satisfying
  \begin{equation}
    \nu(E) = \int_E f \, d\mu, \quad \forall E \in \mathcal{S},
  \end{equation}
  where the integral is understood in the Bochner sense.
\end{defin}

\paragraph{} Reflexive spaces and separable dual spaces are classical examples of Banach spaces that satisfy the Radon-Nikodym property, while $L^1$-type spaces of integrable functions and $c_0$ do not. Under that assumption on the target space, the vector-valued Radon-Nikodym theorem takes the same form as the scalar one, with the Lebesgue integral replaced by the Bochner integral, and with the norm of the density recovering the scalar Radon-Nikodym derivative of the total variation.

\begin{theo}[Radon-Nikodym, vector-valued]\label{theo:rn_vector_valued}
  Let $(\mathcal{X}, \mathcal{S}, \mu)$ be a $\sigma$-finite measure space, and let $(Y, \|\cdot\|_Y)$ be a Banach space with the Radon-Nikodym property. Let $\nu : \mathcal{S} \longrightarrow Y$ be a $Y$-valued measure of finite variation with $\nu \ll \mu$. Then there exists a Bochner-integrable function $f : \mathcal{X} \longrightarrow Y$, unique up to equality $\mu$-almost everywhere, such that
  \begin{equation}
    \nu(E) = \int_E f \, d\mu, \quad \forall E \in \mathcal{S}.
  \end{equation}
  Moreover, the scalar Radon-Nikodym derivative of the total variation satisfies
  \begin{equation}
    \| f(x) \|_Y = \frac{d |\nu|}{d \mu}(x) \quad \mu\text{-a.e.}
  \end{equation}
\end{theo}
\begin{proof}
  The proof can be found in \cite{ledoux2011probability}.
\end{proof}

\section{Mathematical auxiliary results}
\label{ap:math_aux}

\paragraph{} This appendix collects the measure-theoretic and information-theoretic auxiliary results used by the body of the paper. The total variation distance provides a metric on probability measures that admits both a probabilistic characterization in terms of measurable events and an analytic characterization in terms of densities, and that controls differences of expectations of bounded test functions through a variational identity. Pinsker's inequality completes the toolkit by converting the Kullback-Leibler divergence, which is the natural complexity term of the PAC-Bayes layer, into a total variation bound on the same pair of measures. Combined, these results yield the chain $\mathbb{KL} \longrightarrow \text{TV} \longrightarrow |\expectation{\mu}{f} - \expectation{\nu}{f}|$ that translates posterior closeness into closeness of target risks. The Bernoulli Kullback-Leibler divergence and the upper inverse that turns its certificates into explicit risk bounds are recorded at the end as the analytic ingredients of the Bernoulli-KL form of the PAC-Bayes bound used in the empirical validation of Section~\ref{sec:numerical}.

\paragraph{} The total variation distance is one of the canonical metrics on the space of probability measures. Its central feature is the existence of two equivalent characterizations, one over events and one over densities, both of which are needed in the analysis that follows. The first form bounds the worst-case discrepancy in the probability assigned to any measurable event, and the second form expresses the same distance as the $L^1$ distance between densities against a common dominating measure.

\begin{defin}[Total variation distance]
  Let $\mu$ and $\nu$ be two probability measures on a common measurable space $(\mathcal{X}, \mathcal{S})$. The total variation distance between $\mu$ and $\nu$ is
  \begin{equation}
    \text{TV}(\mu, \nu) = \sup \limits_{A \in \mathcal{S}} |\mu(A) - \nu(A)|.
  \end{equation}
  Equivalently, if $\mu, \nu \ll \lambda$ for some common dominating measure $\lambda$, for instance $\lambda = (\mu + \nu)/2$, with densities $f_\mu = d\mu/d\lambda$ and $f_\nu = d\nu/d\lambda$, then
  \begin{equation}
    \text{TV}(\mu, \nu) = \frac{1}{2} \int \limits_{\mathcal{X}} |f_\mu - f_\nu| \, d\lambda.
  \end{equation}
  The total variation distance takes values in $[0,1]$, with $\text{TV}(\mu, \nu) = 0$ if and only if $\mu = \nu$, and $\text{TV}(\mu, \nu) = 1$ if and only if $\mu$ and $\nu$ are mutually singular.
\end{defin}

\paragraph{} The bare definition already gives the right qualitative metric, but the quantity that is most often invoked in applications is the difference of expectations of bounded test functions under the two measures. The next lemma provides the quantitative bridge, in the form of a one-sided variational identity that bounds any such difference of expectations by the total variation distance scaled by the supremum norm of the test function. It is the property of total variation that is used whenever closeness in measure must be turned into closeness of expectations under those measures.

\begin{lemma}[Variational characterization of TV]\label{lemma:variational_char_of_tv}
  For any two probability measures $\mu$ and $\nu$ on $(\mathcal{X}, \mathcal{S})$, and any measurable function $f: \mathcal{X} \longrightarrow \mathbb{R}$ with
  \begin{equation}
    \| f \|_\infty = \sup \limits_{x \in \mathcal{X}} |f(x)| < \infty,
  \end{equation}
  it holds that
  \begin{equation}
    |\expectation{\mu}{f} - \expectation{\nu}{f}| \leq \| f \|_\infty \, \text{TV}(\mu, \nu).
  \end{equation}
\end{lemma}
\begin{proof}
  The proof can be found in \cite{levin2017markov}.
\end{proof}

\paragraph{} Lemma~\ref{lemma:variational_char_of_tv} translates a small total variation distance into a small difference of expectations under the two measures. The remaining gap is that, in PAC-Bayes-type analyses, the natural complexity term between a posterior and a prior is the Kullback-Leibler divergence rather than the total variation distance, and the two quantities are not equivalent in general. Pinsker's inequality provides exactly the one-sided bound that is needed, controlling the total variation distance by the square root of half the Kullback-Leibler divergence on the same pair of measures.

\begin{lemma}[Pinsker's inequality]\label{lemma:pinsker_ineq}
  For any two probability measures $\mu$ and $\nu$ on $(\mathcal{X}, \mathcal{S})$ with $\mu \ll \nu$,
  \begin{equation}
    \text{TV}(\mu, \nu) \leq \sqrt{\frac{1}{2}\, \mathbb{KL}(\mu \,||\, \nu)}.
  \end{equation}
\end{lemma}
\begin{proof}
  The proof can be found in \cite{cover2005elements}.
\end{proof}

\paragraph{} The composition of Lemma~\ref{lemma:variational_char_of_tv} and Lemma~\ref{lemma:pinsker_ineq} delivers the inequality
\begin{equation}
  |\expectation{\mu}{f} - \expectation{\nu}{f}| \leq \| f \|_\infty \, \text{TV}(\mu, \nu) \leq \| f \|_\infty \sqrt{\frac{1}{2}\, \mathbb{KL}(\mu \,||\, \nu)},
\end{equation}
valid for every measurable $f$ with $\|f\|_\infty < \infty$ and every pair of measures with $\mu \ll \nu$. The prototypical application of this chain takes $\mu$ and $\nu$ to be two posteriors obtained by reweighting a common prior $\pi$ by exponential tilts in two nearby loss functionals, and $f$ to be a bounded risk functional on hypotheses, so that closeness of the two loss functionals propagates first into a Kullback-Leibler bound, then into a total variation bound on the induced posteriors, and ultimately into a bound on the difference of expected risks.

\paragraph{} A second information-theoretic ingredient used in the body of the paper is the Kullback-Leibler divergence specialized to Bernoulli laws. The fixed-time PAC-Bayes bound applied to bounded losses, which is the form used in the empirical validation of Section~\ref{sec:numerical}, is most efficiently stated in its Bernoulli-KL form, in which the empirical and population risks are interpreted as Bernoulli parameters and their divergence is then inverted to yield an explicit confidence upper bound on the population risk. This subsection records the divergence itself and the upper inverse that produces the bound, together with the regularity facts that make the inverse well defined.

\begin{defin}[Bernoulli Kullback-Leibler divergence]\label{def:kl_ber}
  For $q, p \in [0, 1]$, the Bernoulli Kullback-Leibler divergence is
  \begin{equation}
    \mathrm{kl}_{\mathrm{Ber}}(q \,\Vert\, p) = q \log \frac{q}{p} + (1 - q) \log \frac{1 - q}{1 - p},
  \end{equation}
  with the conventions $0 \log 0 = 0$, $q \log(q/0) = +\infty$ for $q > 0$, and the symmetric conventions at $p = 1$ for the second term. Equivalently, $\mathrm{kl}_{\mathrm{Ber}}(q \,\Vert\, p)$ is the Kullback-Leibler divergence of the Bernoulli law with parameter $q$ from the Bernoulli law with parameter $p$.
\end{defin}

\paragraph{} The Bernoulli Kullback-Leibler divergence measures how different a Bernoulli law with parameter $q$ is from a Bernoulli law with parameter $p$. It is non-negative, equals zero if and only if $p = q$, and increases as $p$ moves away from $q$. A fixed-time PAC-Bayes statement of the form $\mathrm{kl}_{\mathrm{Ber}}(\widehat R \,\Vert\, R_P) \leq \varepsilon$ therefore certifies that the population risk $R_P$ cannot be too far from the empirical risk $\widehat R$, since otherwise the Bernoulli Kullback-Leibler divergence between them would exceed $\varepsilon$. To turn that implicit constraint into an explicit upper bound on $R_P$, the divergence is inverted in its second argument while the first is held fixed.

\begin{defin}[Upper inverse of the Bernoulli Kullback-Leibler divergence]\label{def:kl_ber_inverse}
  For $q \in [0, 1]$ and $\varepsilon \geq 0$, the upper inverse of $\mathrm{kl}_{\mathrm{Ber}}$ in its second argument is
  \begin{equation}
    \mathrm{kl}_{\mathrm{Ber}}^{-1}(q, \varepsilon) := \sup \bigl\{ p \in [q, 1] \, : \, \mathrm{kl}_{\mathrm{Ber}}(q \,\Vert\, p) \leq \varepsilon \bigr\}.
  \end{equation}
\end{defin}

\begin{lemma}[Well-posedness of the upper inverse]\label{lemma:kl_ber_inverse}
  For every $q \in [0, 1)$, the map $p \in [q, 1] \mapsto \mathrm{kl}_{\mathrm{Ber}}(q \,\Vert\, p)$ is continuous and strictly increasing, with $\mathrm{kl}_{\mathrm{Ber}}(q \,\Vert\, q) = 0$ and $\mathrm{kl}_{\mathrm{Ber}}(q \,\Vert\, p) \longrightarrow +\infty$ as $p \uparrow 1$. Consequently, for every $\varepsilon \geq 0$, the equation $\mathrm{kl}_{\mathrm{Ber}}(q \,\Vert\, p) = \varepsilon$ admits a unique solution $p^\star \in [q, 1)$, and $\mathrm{kl}_{\mathrm{Ber}}^{-1}(q, \varepsilon) = p^\star$, unless $\varepsilon$ is so large that the constraint already allows $p = 1$, in which case $\mathrm{kl}_{\mathrm{Ber}}^{-1}(q, \varepsilon) = 1$.
\end{lemma}
\begin{proof}
  Differentiating $p \mapsto \mathrm{kl}_{\mathrm{Ber}}(q \,\Vert\, p)$ on $(q, 1)$ gives a derivative that is strictly positive on this interval, so the map is continuous and strictly increasing. The boundary value $\mathrm{kl}_{\mathrm{Ber}}(q \,\Vert\, q) = 0$ is immediate, and the limit at $p \uparrow 1$ follows from $\log(1/(1 - p)) \longrightarrow +\infty$ together with $1 - q > 0$. The intermediate-value theorem then yields the unique solution.
\end{proof}

\paragraph{} In the empirical setting of Section~\ref{sec:numerical}, the Bernoulli Kullback-Leibler divergence is applied with $q$ taken as the importance-weighted empirical risk rescaled to $[0, 1]$, and the upper inverse $\mathrm{kl}_{\mathrm{Ber}}^{-1}(q, \varepsilon)$ delivers the explicit high-probability upper bound on the population risk that the fixed-time PAC-Bayes theorem certifies up to confidence $1 - \delta$, with $\varepsilon$ controlled by the Kullback-Leibler complexity term and the confidence level. The same inverse is used inside each epoch of the geometric-peeling construction, with the confidence budget $\delta_k$ replacing $\delta$ at epoch $k$.

\section{Classical density-ratio estimation methods}
\label{ap:classical_dre}

\paragraph{} The constrained density-ratio network developed in Section~\ref{sec:drn} sits within an established line of work on direct density-ratio estimation, in which the importance weight $\trueratio = dP/dQ$ is recovered from independent samples of $P$ and $Q$ without first estimating either density separately. This appendix derives the two classical baselines against which the constrained density-ratio network is compared in the empirical validation of Section~\ref{sec:numerical}, namely the least-squares importance-fitting procedure (LSIF) of \cite{kanamori2009least} together with its unconstrained variant (uLSIF), and the Kullback-Leibler importance-estimation procedure (KLIEP) of \cite{sugiyama2008model, sugiyama2008direct}. Both families share the same overall philosophy as the present framework, since each fits a candidate ratio against a population objective derived from the change-of-measure identity. They differ from the constrained density-ratio network in two structural ways. First, neither imposes hard normalization or higher-order moment-matching constraints during fitting. Second, both are typically realized as kernel models with linear-in-parameters expansions, rather than as neural networks, and the basis is chosen up-front rather than learned. The derivations below make those modelling choices explicit and isolate the population identities that each method optimizes against. They serve as the analytical baseline against which the constrained-network construction of Section~\ref{sec:drn} extends the picture by adding hard calibration constraints, second-moment control, and a high-capacity nonlinear parametrization.

\subsection{Least-squares importance fitting}

\paragraph{} The least-squares importance-fitting objective derives from a single observation. The squared $L^2(Q)$-distance from a candidate ratio $r: \mathcal{Z} \longrightarrow \mathbb{R}_+$ to the true Radon-Nikodym derivative $\trueratio$ is, up to an additive constant that does not depend on $r$, a quantity that can be evaluated as a difference of expectations under $P$ and $Q$. This rewriting eliminates the need to know $\trueratio$ pointwise in order to evaluate the fitting criterion and is the basis for both LSIF and its unconstrained variant uLSIF.

\paragraph{} Consider the half-squared $L^2(Q)$-distance,
\begin{equation}
  \label{eq:lsif_l2_distance}
  \frac{1}{2} \| r - \trueratio \|_{L^2(Q)}^2 = \frac{1}{2} \int_\mathcal{Z} \bigl( r(z) - \trueratio(z) \bigr)^2 \, dQ(z),
\end{equation}
where the factor $\tfrac{1}{2}$ is a convenience that does not change the minimizer. Expanding the square gives
\begin{equation}
  \bigl( r(z) - \trueratio(z) \bigr)^2 = r(z)^2 - 2 \, r(z) \, \trueratio(z) + \trueratio(z)^2,
\end{equation}
so that
\begin{equation}
  \label{eq:lsif_expansion}
  \frac{1}{2} \| r - \trueratio \|_{L^2(Q)}^2 = \frac{1}{2} \int_\mathcal{Z} r(z)^2 \, dQ(z) - \int_\mathcal{Z} r(z) \, \trueratio(z) \, dQ(z) + \frac{1}{2} \int_\mathcal{Z} \trueratio(z)^2 \, dQ(z).
\end{equation}
The cross-term is rewritten by the change-of-variable formula of Theorem~\ref{theo:rn_change_of_var}. Since $\trueratio = dP/dQ$,
\begin{equation}
  \label{eq:lsif_change_of_measure}
  \int_\mathcal{Z} r(z) \, \trueratio(z) \, dQ(z) = \int_\mathcal{Z} r(z) \, dP(z) = \expectation{P}{r(Z)},
\end{equation}
and substituting \eqref{eq:lsif_change_of_measure} into \eqref{eq:lsif_expansion} yields
\begin{equation}
  \label{eq:lsif_population_decomposition}
  \frac{1}{2} \| r - \trueratio \|_{L^2(Q)}^2 = \frac{1}{2} \expectation{Q}{r(Z)^2} - \expectation{P}{r(Z)} + \frac{1}{2} \expectation{Q}{\trueratio(Z)^2}.
\end{equation}

\paragraph{} The third term on the right-hand side of \eqref{eq:lsif_population_decomposition} depends only on the true ratio $\trueratio$ and not on the candidate function $r$. As a function of $r$, it is a constant. Minimizing $\tfrac{1}{2} \| r - \trueratio \|_{L^2(Q)}^2$ in $r$ is therefore equivalent to minimizing
\begin{equation}
  \label{eq:lsif_population_objective}
  J(r) \, := \, \frac{1}{2} \expectation{Q}{r(Z)^2} - \expectation{P}{r(Z)},
\end{equation}
which is the population LSIF objective. The objective in \eqref{eq:lsif_population_objective} depends on $P$ and $Q$ only through expectations of $r$ and $r^2$ under each law, and so admits an unbiased empirical estimator from samples $Z_1^{(P)}, \dots, Z_n^{(P)} \stackrel{\mathrm{iid}}{\sim} P$ and $Z_1^{(Q)}, \dots, Z_m^{(Q)} \stackrel{\mathrm{iid}}{\sim} Q$,
\begin{equation}
  \label{eq:lsif_empirical_objective}
  \widehat J(r) \, = \, \frac{1}{2 m} \sum_{j=1}^m r\bigl( Z_j^{(Q)} \bigr)^2 - \frac{1}{n} \sum_{i=1}^n r\bigl( Z_i^{(P)} \bigr).
\end{equation}

\paragraph{} In LSIF, the candidate $r$ is parametrized as a linear combination of pre-specified basis functions $\varphi_1, \dots, \varphi_K : \mathcal{Z} \longrightarrow \mathbb{R}$,
\begin{equation}
  \label{eq:lsif_basis_expansion}
  r_\alpha(z) = \sum_{j=1}^K \alpha_j \, \varphi_j(z),
\end{equation}
typically a non-negative kernel basis centered at a randomly chosen subset of the source sample. Since the true ratio satisfies $\trueratio(z) \geq 0$ for every $z \in \mathcal{Z}$, the estimator is required to be non-negative as well, which is enforced through coefficient constraints such as $\alpha_j \geq 0$ for every $j$ when the basis itself is non-negative. The constrained LSIF problem then reads
\begin{equation}
  \label{eq:lsif_constrained_problem}
  \min_{\alpha \in \mathbb{R}^K} \, \widehat J(r_\alpha) \quad \text{subject to} \quad r_\alpha(z) \geq 0, \; \forall z \in \mathcal{Z}.
\end{equation}
The objective is a convex quadratic in $\alpha$ and the non-negativity requirement is a polyhedral condition, so that \eqref{eq:lsif_constrained_problem} is a convex quadratic program.

\paragraph{} The unconstrained variant uLSIF, introduced as a computationally cheaper relative of LSIF, drops the hard non-negativity requirement and replaces it with an $\ell^2$ regularization term on the coefficients,
\begin{equation}
  \label{eq:ulsif_problem}
  \min_{\alpha \in \mathbb{R}^K} \, \widehat J(r_\alpha) + \frac{\lambda}{2} \| \alpha \|_2^2,
\end{equation}
with regularization parameter $\lambda > 0$. The resulting problem is an unconstrained convex quadratic and admits a closed-form solution in terms of a linear system of size $K$, which is the practical advantage of uLSIF over constrained LSIF. The price for that simplicity is that nothing in the procedure guarantees $r_\alpha(z) \geq 0$ pointwise, and post-hoc clipping to $\max\{ r_\alpha(z), 0 \}$ is typically applied at evaluation time. From the perspective of the present framework, uLSIF is the natural baseline for direct density-ratio estimation under squared loss, while the constrained density-ratio network of Section~\ref{sec:drn} retains the LSIF population objective \eqref{eq:lsif_population_objective} and adds calibration constraints, second-moment control, and a high-capacity nonlinear parametrization rather than relaxing the non-negativity requirement.

\subsection{Kullback-Leibler importance estimation}

\paragraph{} The Kullback-Leibler importance-estimation procedure approaches direct density-ratio estimation through a divergence rather than a squared-error criterion. The starting point is the observation that a non-negative candidate ratio $r_\theta: \mathcal{Z} \longrightarrow \mathbb{R}_+$ together with the source law $Q$ defines a candidate target measure $P_\theta := r_\theta \, Q$, and that the closeness of $P_\theta$ to $P$ can be measured by the Kullback-Leibler divergence of $P$ from $P_\theta$. The KLIEP estimator is the candidate ratio that minimizes this divergence under the normalization required for $P_\theta$ to be a probability measure.

\paragraph{} Let $r_\theta : \mathcal{Z} \longrightarrow \mathbb{R}_+$ be a non-negative candidate ratio with parameters $\theta$. The candidate measure
\begin{equation}
  \label{eq:kliep_candidate_measure}
  P_\theta := r_\theta \, Q, \quad \text{that is,} \quad dP_\theta = r_\theta \, dQ,
\end{equation}
is a $\sigma$-finite measure on $\mathcal{Z}$ that is absolutely continuous with respect to $Q$ by construction. The Kullback-Leibler divergence of $P$ from $P_\theta$ is
\begin{equation}
  \label{eq:kliep_kl_definition}
  \mathbb{KL}(P \, || \, P_\theta) = \int_\mathcal{Z} \log \! \left( \frac{dP}{dP_\theta}(z) \right) dP(z),
\end{equation}
provided that $P \ll P_\theta$, which is the case whenever $r_\theta(z) > 0$ wherever $\trueratio(z) > 0$. Since $dP_\theta = r_\theta \, dQ$ and $dP = \trueratio \, dQ$, the chain rule for Radon-Nikodym derivatives of Theorem~\ref{theo:rn_chain_rule} gives
\begin{equation}
  \label{eq:kliep_chain_rule}
  \frac{dP}{dP_\theta}(z) \, = \, \frac{dP/dQ}{dP_\theta/dQ}(z) \, = \, \frac{\trueratio(z)}{r_\theta(z)} \quad Q\text{-a.e.},
\end{equation}
and substituting \eqref{eq:kliep_chain_rule} into \eqref{eq:kliep_kl_definition} yields
\begin{equation}
  \label{eq:kliep_kl_split}
  \mathbb{KL}(P \, || \, P_\theta) \, = \, \int_\mathcal{Z} \log \trueratio(z) \, dP(z) \, - \, \int_\mathcal{Z} \log r_\theta(z) \, dP(z).
\end{equation}

\paragraph{} The first term on the right-hand side of \eqref{eq:kliep_kl_split} depends on $P$ and $Q$ only through the true ratio $\trueratio$ and is independent of $\theta$. Minimizing $\mathbb{KL}(P \, || \, P_\theta)$ in $\theta$ is therefore equivalent to maximizing
\begin{equation}
  \label{eq:kliep_population_logloss}
  \expectation{P}{\log r_\theta(Z)}.
\end{equation}
For $P_\theta = r_\theta \, Q$ to be a probability measure, the candidate ratio must additionally satisfy the normalization condition $\expectation{Q}{r_\theta(Z)} = 1$ together with non-negativity $r_\theta \geq 0$. The population KLIEP problem is therefore
\begin{equation}
  \label{eq:kliep_population_objective}
  \max_{\theta} \, \expectation{P}{\log r_\theta(Z)} \quad \text{subject to} \quad \expectation{Q}{r_\theta(Z)} = 1, \; r_\theta \geq 0.
\end{equation}

\paragraph{} Empirically, given source samples $Z_1^{(Q)}, \dots, Z_m^{(Q)} \stackrel{\mathrm{iid}}{\sim} Q$ and target samples $Z_1^{(P)}, \dots, Z_n^{(P)} \stackrel{\mathrm{iid}}{\sim} P$, the expectations in \eqref{eq:kliep_population_objective} are replaced by their sample averages, yielding the finite-sample KLIEP problem
\begin{equation}
  \label{eq:kliep_empirical_objective}
  \max_{\theta} \, \frac{1}{n} \sum_{i=1}^n \log r_\theta\bigl( Z_i^{(P)} \bigr) \quad \text{subject to} \quad \frac{1}{m} \sum_{j=1}^m r_\theta\bigl( Z_j^{(Q)} \bigr) = 1, \; r_\theta \geq 0.
\end{equation}
KLIEP is typically realized with a linear-in-parameters non-negative kernel model $r_\theta(z) = \sum_{j=1}^K \theta_j \, \varphi_j(z)$ together with the coefficient constraint $\theta_j \geq 0$ for every $j$, which is sufficient to enforce $r_\theta \geq 0$ pointwise. The normalization condition then reduces to a single linear equality on the parameter vector, and the negative-log-likelihood objective is concave, so that \eqref{eq:kliep_empirical_objective} is a tractable convex program. From the perspective of the present framework, KLIEP is the natural baseline among normalized direct density-ratio estimators. It is a calibrated method, in the sense that by construction the candidate $r_\theta$ satisfies $\expectation{Q}{r_\theta} = 1$ at the population level, which the constrained density-ratio network of Section~\ref{sec:drn} also imposes. The two methods differ in how the normalization is enforced and in what the parametrization can express. KLIEP fits a single hard equality on a fixed linear basis, while the constrained density-ratio network drives normalization, higher-order moment matching, and second-moment growth simultaneously through an augmented-Lagrangian primal-dual scheme over a high-capacity nonlinear model.

\end{document}